\documentclass[11pt]{article} 
\usepackage{hyperref}
\usepackage{url}
\usepackage{smile}
\usepackage{graphicx} 
\usepackage{algorithm}
\usepackage{algorithmic}
\usepackage{epstopdf}
\usepackage[margin=0.8in]{geometry}
\usepackage[export]{adjustbox}
\usepackage{mathtools, cuted}
\usepackage{bbm}
\usepackage{wrapfig}
\usepackage{subcaption}
\usepackage{caption}
\usepackage{enumitem}
\usepackage{tabularx}
\usepackage{smile}
\usepackage{tcolorbox}
\usepackage{enumitem}
\usepackage{mathtools, cuted}
\usepackage{caption}
\usepackage{subcaption}

\numberwithin{equation}{section}

\usepackage[nocompress]{cite}

\hypersetup{
    colorlinks=true,
    linkcolor=blue,
    anchorcolor=blue, 
    citecolor=blue,
}

\allowdisplaybreaks

\title{
 First-order Policy Optimization for Robust Markov Decision Process
 \thanks{
This research was partially supported by NSF DMS-2134037 and AFOSR FA9550-22-1-0447.
 } 
    }
\author{
    Yan Li   \thanks{H. Milton Stewart School of Industrial and Systems Engineering, Georgia Institute of Technology, Atlanta, GA, 30332. (E-mail: \url{yli939@gatech.edu}).}
        \and 
    Guanghui Lan \thanks{H. Milton Stewart School of Industrial and Systems Engineering, Georgia Institute of Technology, Atlanta, GA, 30332. (E-mail: \url{george.lan@isye.gatech.edu}).}
    \and
    Tuo Zhao \thanks{H. Milton Stewart School of Industrial and Systems Engineering, Georgia Institute of Technology, Atlanta, GA, 30332. (E-mail: \url{tourzhao@gatech.edu}).}
}
\date{\vspace{-6ex}}

\begin{document}

\maketitle

\begin{abstract}
We consider the problem of solving robust Markov decision process (MDP), which involves a set of discounted, finite state, finite action space MDPs with uncertain transition kernels. The goal of planning is to find a robust policy that optimizes the worst-case values against the  transition uncertainties, and thus encompasses the standard MDP planning as a special case. For $(\mathbf{s},\mathbf{a})$-rectangular uncertainty sets, we establish several structural observations on the robust objective, which facilitates the development of a policy-based first-order method, namely the robust policy mirror descent (RPMD). An $\cO (\log(1/\epsilon))$ iteration complexity for finding an $\epsilon$-optimal policy  is established with linearly increasing stepsizes. We further develop a stochastic variant of the robust policy mirror descent method, named SRPMD, when the first-order information is only available through online interactions with the nominal environment. We show that the optimality gap converges linearly up to the noise level, and consequently establish an $\tilde{\cO}(1/\epsilon^2)$ sample complexity by developing a temporal difference learning method for policy evaluation. Both iteration and sample complexities are also discussed for RPMD with a constant stepsize. To the best of our knowledge, all the aforementioned results appear to be new for policy-based first-order methods applied to the robust MDP problem.
\end{abstract}


\section{Introduction}

We consider the problem of solving the robust Markov decision process (MDP) where the transition kernel is uncertain, and one seeks to learn a policy that behaves robustly against such uncertainties.
Specifically, 
a robust MDP  $\cM_\cU \coloneqq \cbr{\cM_u = (\cS, \cA, c, \PP_u, \gamma): u \in \cU}$ consists of a set of MDPs, where $\cS$ and $\cA$ denote the state and action space, respectively;  $\PP_u: \cS \times \cA \to [0,1]$ denotes the transition kernel, indexed by $u \in \cU$; $c: \cS \times \cA \to \RR$ denotes the cost function, which we assume with loss of generality that $0 < c(s,a) \leq 1$ for all $(s,a)$; $\gamma$ denotes the discount factor.
The set of MDPs differ from each other only in their respective transition kernels. 
The standard value function $V^{\pi}_u: \cS \to \RR$ of a policy $\pi$ with respect to MDP $\cM_u$,  is defined as 
\begin{align*}
V^{\pi}_u(s) = \EE \sbr{\tsum_{t=0}^\infty \gamma^t c(s_t, a_t) \big| s_0 = s, a_t \sim \pi(\cdot|s_t), s_{t+1} \sim \PP_u(\cdot| s_t, a_t) }, ~~ \forall s \in \cS.
\end{align*}

Our end goal is to learn a policy $\pi^*$ that is the solution of the following problem 
\begin{align}\label{formulation_plain}
\textstyle 
\min \cbr{V^{\pi}_r(s) \coloneqq \max_{u \in  \cU} V^{\pi}_u(s): ~ \pi \in \Pi} ,
\end{align} 
where $\Pi$ denotes the set of all stationary and randomized policies.
That is, \eqref{formulation_plain} aims to learn a policy that minimizes the worst-case value simultaneously for every state.
 Clearly, when $\cU$ is a singleton, the problem of finding robust policy \eqref{formulation_plain} reduces to solving a standard MDP planning problem. 
 
 Before any technical discussions, we first construct a simple example motivating our study of finding a robust policy in the sense of \eqref{formulation_plain}, when facing transition uncertainty. 
 Specifically, we will construct a pair of MDP $\cM$ and $\cM_u$, with the same $(\cS, \cA ,c, \gamma)$,  and the transition kernels of these two MDPs are close to each other, yet the optimal policy for $\cM$ achieves highly suboptimal value in $\cM_u$.
 In contrast, we show that there exists a policy that achieves close-to-optimal performances in both $\cM$ and $\cM_u$.
 
 \begin{figure}[t!]
 \centering
\makebox[\linewidth][c]{%
     \centering
     \begin{subfigure}[b]{0.35\textwidth}
         \centering
         \includegraphics[width=\textwidth]{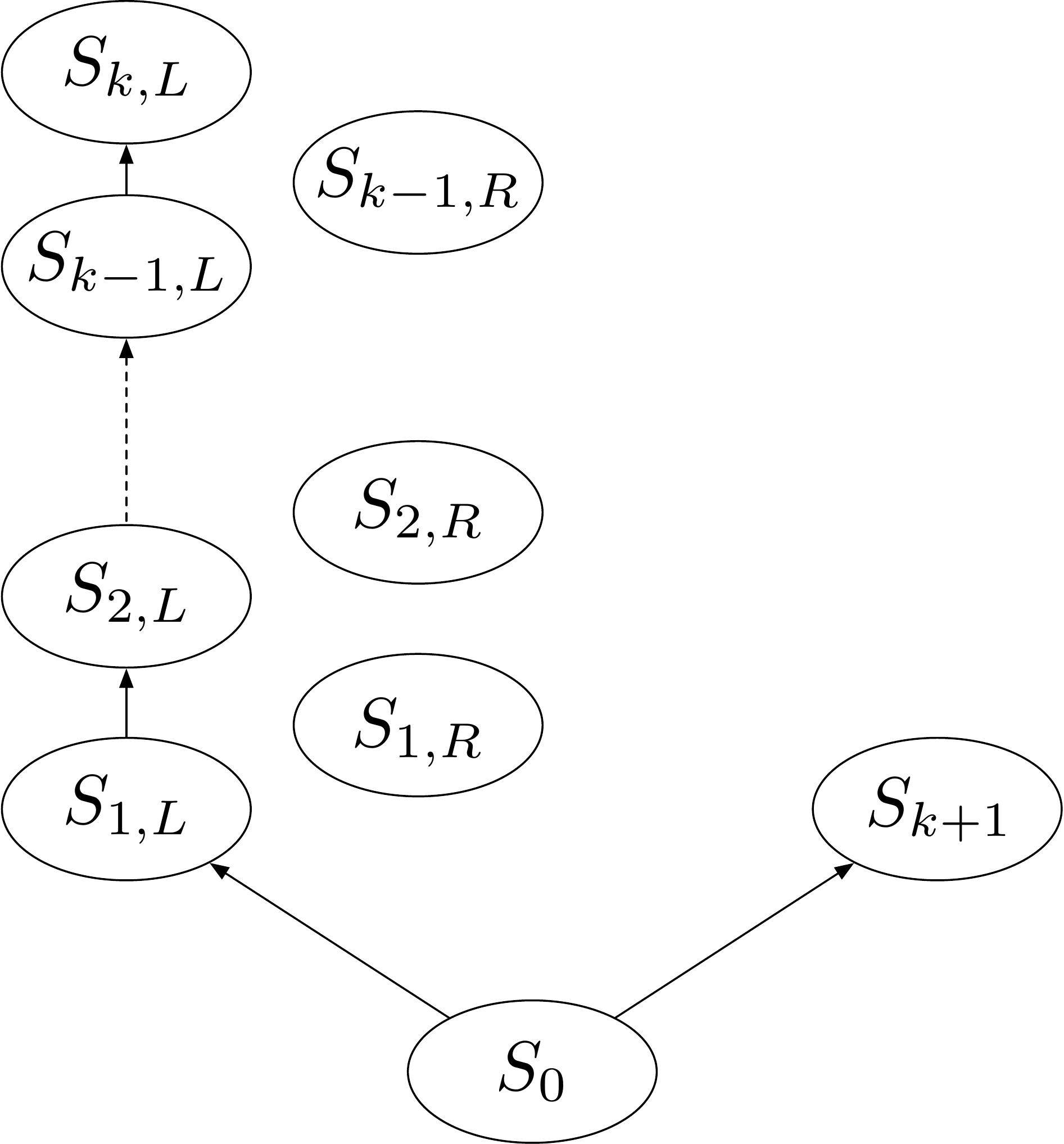}
             \vspace{-0.1in} 
             \caption{Nominal $\cM$ with transition kernel $\PP$.}
             \label{subfig:nominal}
     \end{subfigure}
     \hspace{0.1in}
     \begin{subfigure}[b]{0.35\textwidth}
         \centering
         \includegraphics[width=\textwidth]{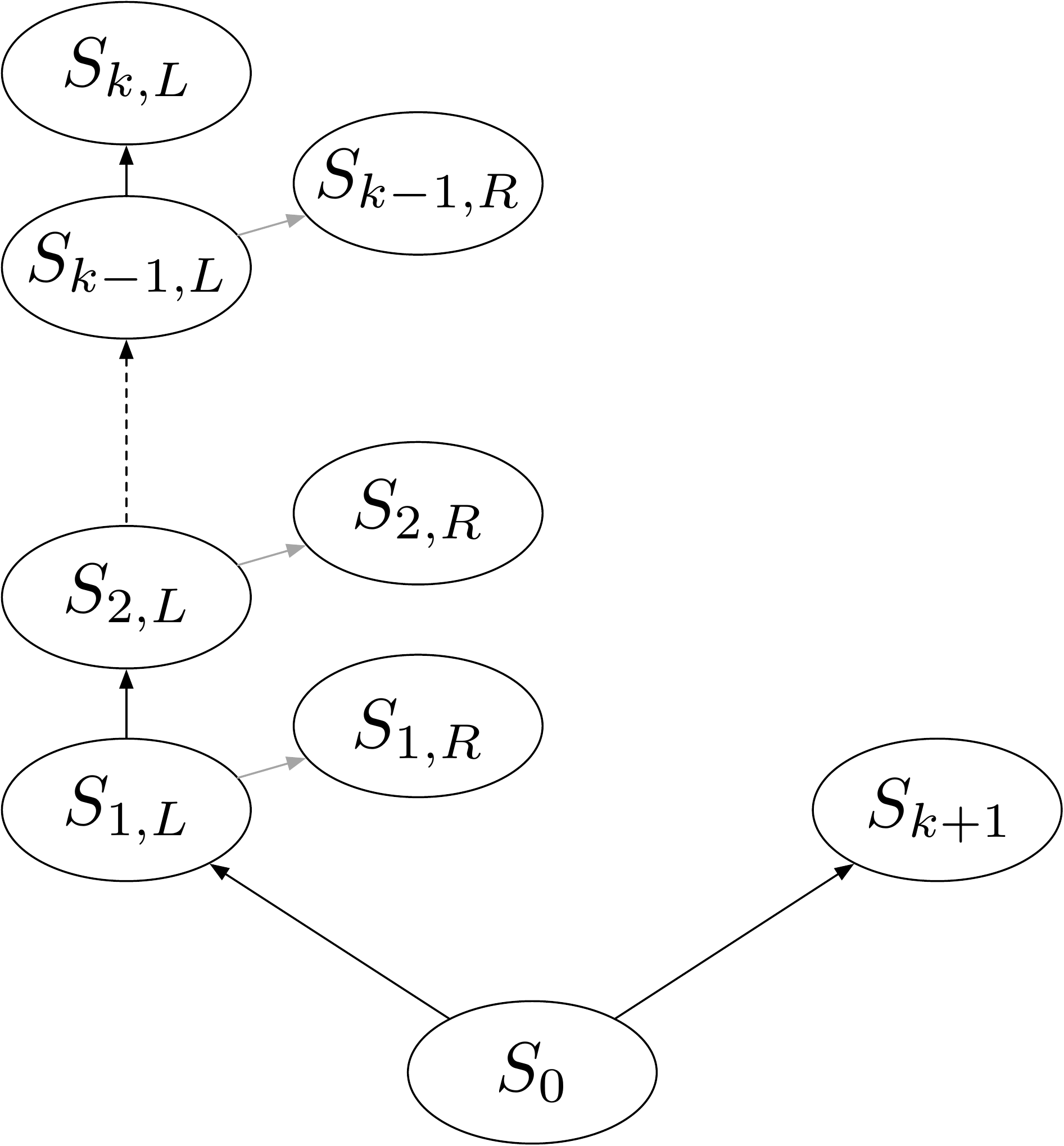}
             \vspace{-0.1in}
            \caption{$\cM_u$ with transition kernel $\PP_u \approx \PP$. }
            \label{subfig:perturb}
     \end{subfigure}  
     }
     \vspace{-0.25in}
  \caption{ 
A nominal MDP, and its approximate clone with small changes to the transition kernel. 
  }
  \label{fig:motivation}
\end{figure}

 \begin{example}[Tradeoff between planning efficiency and robustness]\label{example_motivate}
Consider the nominal MDP $\cM$ illustrated in Figure \ref{subfig:nominal}. 
Starting at $S_0$, the agent has two possible actions $\cbr{L, R}$,  consisting of going either left or right. 
Going right incurs a cost of $- 1$, going left incurs a cost of $0$.
We assume the cost occurs immediately after the action is made. 
Whenever the agent arrives at $S_{i, R}$ for $1 \leq i \leq k-1$ and $S_{k+1}$, the agent stays at the same place going forward.

Additionally, if the agent goes left from $S_0$, then in the ensuing rounds the agent has only one available action, which is to transit to the next state following the arrows.
No cost is incurred until the agent transits from $S_{k-1, L}$ to $S_{k,L}$, which incurs a cost of $-(1 + \epsilon) \gamma^{-k + 1}$  for some small positive number $\epsilon \ll 1$ (e.g., $\epsilon = 0.01$). 

Now consider another MDP $\cM_u$ that is exact the same as $\cM$, except in the transition kernel, illustrated in Figure \ref{subfig:perturb}.
In particular, for any $1 \leq m \leq k-1$, the transition changes to $\PP_u(S_{m+1, L} | S_{m, L}) = p, ~ \PP_u(S_{m, R} | S_{m, L}) = 1-p$, for some $p \in (0,1)$ close to $1$ (e.g., $p=0.99$).
Thus $\cM_u$ can be viewed as an approximate copy of $\cM$.

It should be clear that for MDP $\cM$, going left at $S_0$ incurs a value of $-(1+ \epsilon)$, which is close, but still strictly small than the value $-1$ of going right.
Thus the optimal policy $\pi^*$ for $\cM$ is to always go left.

However, when deploying $\pi^*$ in the slightly changed environment $\cM_u$, one can clearly see that going left at $S_0$ now incurs a value of $-(1+ \epsilon) p^{k-1}$, while going right still has the value of $-1$.
For $k$ large enough, we obverse that the value of going left approaches $0$, which is significantly worse than the value of going right.

In conclusion, we see that going left at $S_0$ serves as the optimal policy in $\cM$, but its performance degrades significantly despite being deployed at a similar environment $\cM_u$. 
In fact, such a policy is even worse than the policy of randomly going left or right with equal probability.
In contrast, going right at $S_0$ is an $\epsilon$-optimal policy in $\cM$, and also the optimal policy in $\cM_u$, thus being a much more desirable choice in terms of robustness.
\end{example}

The tensions between accuracy and robustness  have been discussed extensively in the supervised learning literature \cite{tsipras2018robustness, xu2009robustness, xu2008robust, Li2020Implicit}. Example \ref{example_motivate} demonstrates that similar tensions between cost-minimization (planning) and robustness also exists in the control of uncertain Markov decision process.
It should be noted that the key ingredient for the lack of robustness in the MDP demonstrated in Figure \ref{fig:motivation} is  the {reward (cost) sparseness}.
This feature has been widely observed for practical MDP applications \cite{pathak2017curiosity, nair2018overcoming}, and we believe the same mechanism can be one of the few important factors that lead to brittle robustness observed in existing reinforcement learning applications. 

To proceed, we will focus on the case of $(\mathbf{s}, \mathbf{a})$-rectangular uncertainty set, defined below.

\begin{definition}[$(\mathbf{s}, \mathbf{a})$-Rectangular Uncertainty]\label{def_rectangular}
We assume the  transition kernel $\PP_u$ for the MDP $\cM_u$ takes the form of 
\begin{align}\label{rectangular_set}
\PP_u (\cdot | s,a) = \PP_{\NN}(\cdot | s,a) + u(\cdot|s,a), ~ \forall (s,a) \in \cS \times \cA, 
\end{align}
where $\PP_{\NN}(\cdot | s, a)$ denotes the nominal transition kernel, and 
 $u \in \RR^{\abs{\cS} \times \abs{\cS} \times \abs{\cA}}$ denotes the perturbation to the nominal transition kernel. 
   The uncertainty is said to be rectangular if $\cU$ satisfies 
   \begin{align*}
   \cU = \Pi_{s \in \cS, a \in \cA} \cU_{s,a}, ~~ \text{where}~~ \cU_{s,a} = \cbr{u(\cdot|s,a): u \in \cU}.
   \end{align*}
We assume $\cU$ is compact.
 In addition, we let 
 \begin{align}\label{eq_ambiguity_set}
 \cP_{s,a} = \PP_{\NN}(\cdot|s,a) + \cU_{s,a}
 \end{align}
 denote the set of possible transition probabilities at $(s,a) \in \cS \times \cA$, which we refer to as the ambiguity set.
 \end{definition}

\begin{remark}
 While one can also consider uncertain cost function $c_u$ when modeling, our motivation to focus on modeling uncertain transitions  is due to the observation that cost function is mostly an endogenous user choice, and thus seems less suitable to be modeled as an uncertainty.
 \end{remark}
 
 From Definition \ref{def_rectangular}, it should be also clear that  Example \ref{example_motivate} can be readily modeled into a robust MDP with a rectangular uncertainty set.
 We will also define the nominal environment of the robust MDP $\cM_\cU$ as follows, a useful notion for our discussions in the stochastic settings.
 
\begin{definition}[Nominal Environment]
The nominal environment $\cM_{\mathrm{N}}$ for a robust MDP problem $\cM_\cU$, is the MDP with transition kernel $\PP_{\NN}$.
\end{definition}

Given a policy $\pi$, we define its robust value function $V^{\pi}_r: \cS \to \RR$ as  $V^{\pi}_r(s) = \max_{u \in \cU} V^{\pi}_u(s)$ for all $s\in \cS$.
Consequently, solving \eqref{formulation_plain} is equivalently to minimizing the robust value function:
\begin{align}\label{formulation_robust}
\textstyle
\pi \in  \Argmin_{\pi \in \Pi} V^{\pi}_r(s).
\end{align}
The existence of an optimal policy $\pi^*$ solving \eqref{formulation_robust} is well known in the literature \cite{nilim2005robust, iyengar2005robust},
and we denote the set of optimal policies as $\Pi^* \subseteq \Pi$.
Hence, we can succinctly reformulate \eqref{formulation_robust} into a single objective optimization problem 
\begin{align}\label{formulation_single_obj}
\textstyle
\min_{\pi \in \Pi} \cbr{f_\rho(\pi) \coloneqq \EE_{s \sim \rho} V^{\pi}_r(s) },
\end{align}
where $\rho$ is a nonnegative measure defined over the state space $\cS$.

For any $u \in \cU$, we also define the state-action value function  of policy $\pi$ with respect to $\cM_u$ as 
\begin{align*}
Q^{\pi}_u(s, a) = \EE \sbr{\tsum_{t=0}^\infty \gamma^t c(s_t, a_t) \big| s_0 = s, a_0 =a, , a_t \sim \pi(\cdot|s_t), s_{t} \sim \PP_u(\cdot| s_{t-1}, a_{t-1}), t \geq 1 }, ~\forall (s,a) \in \cS \times \cA.
\end{align*} 
Accordingly, the robust state-action value function  $Q^{\pi}_r: \cS \times \cA \to \RR$ as 
$Q^{\pi}_r(s,a) = \max_{u \in \cU} Q^{\pi}_{u} (s,a)$ for all $(s,a) \in \cS \times \cA$.
Note that from the definition of $Q^{\pi}_u$ and $V^{\pi}_u$,  we have 
\begin{align}\label{standard_value_q_relation}
V^{\pi}_{u} (s) = \inner{Q^{\pi}_u(s, \cdot)}{ \pi(\cdot | s) } \coloneqq \inner{Q_u^{\pi}}{\pi}_s.
\end{align}

Given a policy $\pi$, 
and an uncertainty $u \in \cU$,
 the discounted state visitation measure jointly induced by $(\pi, u)$ is defined as  
\begin{align}
d^{\pi,u}_s(s') & = (1-\gamma) \tsum_{t=0}^{\infty} \gamma^t \PP_{u}^{\pi}(s_t = s' | s_0 = s), \label{def_discounted_visit_measure} 
\end{align}
where $\PP_{u}^{\pi}(s_t = s' | s_0 = s) \coloneqq \PP \rbr{ s_t = s'| s_0 = s, a_t = \pi(\cdot|s_t), s_{t+1} \sim \PP_{u}(\cdot| s_t, a_t)} $ denotes the probability of reaching state $s'$ at timestep $t$, given initial state $s$, and following policy $\pi$ within MDP $\cM_u$.
Given any distribution $\rho$ over $\cS$, we define
distribution $d_{\rho}^{\pi, u}$ over $\cS$ as 
 $d_{\rho}^{\pi, u}(s') = \EE_{s\sim \rho} d_s^{\pi, u}(s')$.
For a finite set $\cX$, we will denote $\Delta_\cX$ as the $(\abs{\cX}-1)$-dimensional simplex.

\vspace{0.1in}
{\bf Related Literature.}
Solving robust MDP \eqref{formulation_plain} with rectangular uncertainty sets has been extensively studied in the dynamic programming literature.
Among value-based methods, value iteration (VI) is known to achieve linear convergence to the optimal robust values \cite{nilim2005robust, iyengar2005robust}.
When the environment is unknown, sample-based value based methods  \cite{ma2022distributionally, liu2022distributionally, zhou2021finite, roy2017reinforcement, wang2021online, panaganti2022sample}, including robust Q-learning, have also been developed to directly learn the optimal value function.
Policy-based methods, including the (modified) policy iteration (PI), have been studied in \cite{iyengar2005robust, ruszczynski2010risk, kaufman2013robust, wiesemann2013robust, ho2021partial}.
Approximate dynamic programming (ADP) techniques \cite{powell2007approximate} for both type of methods have also been developed, 
which allow approximate computation of policy update and evaluation  in PI \cite{badrinath2021robust, tamar2014scaling}, or approximate Bellman update of VI \cite{tamar2014scaling, ma2022distributionally, zhou2021finite}.
The application of ADP to policy-based methods also enables function approximation to be used to handle the curse of dimensionality \cite{tamar2014scaling, badrinath2021robust} while incorporating approximate policy evaluation (e.g., robust TD learning \cite{roy2017reinforcement, kose2020risk}).
It should be noted there also exists another complementary line of research, on studying robust MDP with uncertainty set beyond $(\mathbf{s}, \mathbf{a})$-rectangularity \cite{wiesemann2013robust, goyal2022robust, kumar2022efficient, derman2021twice}.

In addition to the prior developments in the context of dynamic programming, there has been a rising interests in  developing first-order methods for solving the special case of \eqref{formulation_single_obj},   where there is no uncertainty in the transition (e.g., $\cU$ being a singleton).
By using first-order information of objective \eqref{formulation_single_obj} to update the policy, 
these policy-based methods are thus termed policy gradient methods (PGM), with their convergence behavior extensively studied in the literature. 
Sublinear convergence of the optimality gap for constant stepsize PGMs  have been established in \cite{agarwal2021theory, lan2021policy}, 
and linearly converging variants have been proposed in \cite{lan2021policy, khodadadian2021linear, xiao2022convergence, cen2021fast},
with local superlinear convergence studied in \cite{li2022homotopic, khodadadian2021linear}.
\cite{li2022homotopic} recently further characterizes the policy convergence of a  PGM variant.
Moreover, stochastic PGMs, which utilize sample to estimate the first-order information, have also been proposed in \cite{lan2021policy, xu2020improving, shani2020adaptive}, and both sample and iteration complexity have been studied therein. 
Complementary to the policy-based first-order methods, \cite{goyal2022first} propose an accelerated first-order value-based method, and establish improved dependence on discount factor compared to value iteration.

In contrast to the aforementioned developments of PGMs for solving non-robust MDP,
solving robust MDP \eqref{formulation_single_obj} with first-order methods has been largely under-explored.
Specifically, \cite{grand2020scalable} propose a first-order value-based method derived based on value iteration,
while \cite{wang2022policy} seems to be the only PGM variant to date that directly aims to solve \eqref{formulation_single_obj}, which focuses on a subclass of polyhedral uncertainty.
Given the abundant empirical observations on the unsatisfactory  performance of PGM-trained RL agents when the deployment environment differs from the training environment \cite{shen2020deep, yang2022rorl}, there seems to be a practical need to develop first-order policy-based methods that can learn a policy with robustness  guarantees. 
 
\vspace{0.1in}
Our contributions mainly exist in the the following aspects.
First, we establish some new structural results for the robust Markov decision process.
In particular, we discuss the differentiability of robust values, a robust version of performance difference lemma (c.f., \cite{kakade2002approximately}), and a novel variational inequality perspective on solving robust MDP. 
These results serve as a crucial component that facilitates our ensuing computational development, and might be of independent interests for other algorithmic studies (e.g., natural robust policy gradient, see Section \ref{sec_rpmd}).

Second, we develop a first-order policy-based method, named robust policy mirror descent (RPMD),  for solving the robust MDP problem \eqref{formulation_single_obj}.
Despite the non-convex and non-smooth structure of the objective (see \cite{wang2022policy}), RPMD finds an $\epsilon$-optimal policy in $\cO(\log (1/\epsilon))$ iterations.
The established convergence results hold for any Bregman divergence, as long as the policy space has a bounded distance to the initial policy measured in the same divergence.
To the best of our knowledge, no existing PGM can attain the obtained iteration complexity when solving \eqref{formulation_single_obj}.
In addition, we also establish the sublinear convergence of constant-stepsize RPMD, for both Euclidean Bregman divergence, 
and a more general class of Bregman divergence applied to any relatively strongly convex uncertainty set.



Finally, we develop stochastic variants of the RPMD method, named SRPMD,  when the  first-order information is only available through online interactions with the nominal environment.
For general Bregman divergences, we show that SRPMD with linearly increasing stepsizes converges linearly up to the noise level, and consequently determine an $\cO(1/\epsilon^2)$  sample complexity.
For Euclidean Bregman divergence, we show an $\cO(1/\epsilon^3)$ sample complexity with a properly chosen constant stepsize. 
To the best of our knowledge, all the developed sample complexity results of RPMD appear to be new for PGM methods applied to the robust MDP problem.


The rest of this manuscript is organized as follows. 
Section \ref{sec_structural_props} makes some structural observations on the robust Markov decision process that will prove useful in the ensuing algorithmic developments. 
Section \ref{sec_rpmd} introduces the deterministic RPMD method and establish its convergence properties.
Section \ref{sec_srpmd} develops the stochastic variants of RPMD when only stochastic first-order information is available. 
Section \ref{sec_sample_complexity} then establishes the sample complexity for the proposed stochastic RPMD methods.
Finally, concluding remarks are made in Section \ref{sec_conclusion}.


\section{Structural Properties of Robust MDP}\label{sec_structural_props}

In this section, we develop some important observations on the structural properties of robust MDP, which will prove to be useful in our ensuing algorithmic developments.

\subsection{Structure of Robust Value Functions}
We first characterize the robust value function of any stochastic policy, following similar arguments for deterministic policies in \cite{nilim2005robust}.

\begin{proposition}\label{prop_bellman_v}
For robust MDP $\cM_\cU$ with a compact rectangular uncertainty set $\cU$, defined in Definition \ref{def_rectangular},
the robust value function satisfies the following nonlinear Bellman equation
\begin{align}\label{bellman_robust_value}
V^{\pi}_r(s) = 
\sum_{a \in \cA} c(s,a) \pi(a|s) + 
\gamma \sum_{a \in \cA} \pi(a|s) \max_{u \in \cU} \sum_{s' \in \cS}   \PP_u(s'|s,a) V^{\pi}_r(s') , ~~\forall s \in \cS.
\end{align}
In addition, a worst-case transition kernel $\PP_{u_\pi}$ for the policy $\pi$ is given by 
\begin{align}\label{def_worst_case_transition}
u_{\pi} (\cdot|s,a) \in \Argmax_{u(\cdot|s,a) \in \cU_{s,a}} \sum_{s' \in \cS} \PP_u(s'|s,a) V^{\pi}_r(s'), ~~ \forall (s,a) \in \cS \times \cA, 
\end{align}
or equivalently, 
\begin{align*}
V^{\pi}_r(s) = 
\sum_{a \in \cA} c(s,a) \pi(a|s) + 
\gamma \sum_{a \in \cA} \pi(a|s) \sum_{s' \in \cS} \PP_{u_{\pi}}(s'|s,a) V^{\pi}_r(s') ,~~ \forall s\in \cS.
\end{align*}
\end{proposition}

It is worth mentioning that the worst-case environment $u_{\pi}$ defined in \eqref{def_worst_case_transition} can be non-unique. 
In this case, one can choose any of them by an arbitrary deterministic rule. 
Note that the last relation in Proposition \ref{prop_bellman_v} also shows that $V^{\pi}_r$ is the solution of the standard Bellman equation for standard value function with uncertainty $u_\pi$, denoted by $V^{\pi}_{u_{\pi}}$. Hence from the uniqueness of the solution for the Bellman equation, we obtain 
\begin{align}\label{robust_and_standard_value_relation}
V^{\pi}_r = V^{\pi}_{u_{\pi}}.
\end{align}

Following similar lines as in Proposition \ref{prop_bellman_v}, we can establish the following properties of $Q^{\pi}_r$.

\begin{proposition}\label{prop_robust_bellman_q}
The robust state-action value function $Q^{\pi}_r$ satisfies 
\begin{align}\label{eq_robust_bellman_q}
Q^{\pi}_r(s,a) =  c(s,a) + \max_{u \in \cU} \gamma \sum_{s' \in \cS}  \PP_u(s'|s,a) \sum_{a' \in \cA} \pi(a'|s') Q^{\pi}_r(s',a'), ~~\forall (s,a) \in \cS \times \cA.
\end{align}
Moreover, $Q^{\pi}_r$ and $V^{\pi}_r$ satisfies the following relation
\begin{align}\label{def_robust_q}
Q^{\pi}_r(s,a) =
c(s,a) + \gamma \max_{u \in \cU} \sum_{s' \in \cS} \PP_u(s'|s,a) V^{\pi}_r(s'),  ~~\forall (s,a) \in \cS \times \cA,
\end{align}
Finally, we also have
\begin{align}
V^{\pi}_r(s) &= \inner{Q^{\pi}_r(s, \cdot)}{ \pi(\cdot | s) } \coloneqq \inner{Q^{\pi}_r}{\pi}_s, ~~\forall s \in \cS ,\label{robust_value_q_relation} \\
Q^{\pi}_r &= Q^{\pi}_{u_{\pi}} \label{robust_and_standard_q_relation} ,
\end{align}
where $u_{\pi}$ is defined as in \eqref{def_worst_case_transition}.
\end{proposition}

The proof of Proposition \ref{prop_bellman_v} and \ref{prop_robust_bellman_q} are deferred to Appendix \ref{section_appendix}.
We next proceed to discuss the differentiability of robust values.

\subsection{Differentiability of Robust Values}

A seemingly natural concept to minimize $f_\rho$ defined in \eqref{formulation_single_obj} is to iteratively update the policy by following its negative gradient direction.
At the same time, 
it should be noted that even for a fixed uncertainty (i.e., non-robust MDP), 
 for any state $s \in \cS$,  the value function $V^{\pi}_u(s)$ is only well defined over the set of randomized policies $\Pi$, for which any $\pi \in \Pi$ must satisfy $\mathbf{1}^\top \pi(\cdot |s) = 1$ for all $s\in \cS$.
Hence $\Pi = \mathrm{dom}(f_\rho)$ belongs to a lower-dimensional subspace in $ \RR^{ \abs{\cS}  \times \abs{\cA}}$ with dimension $(\abs{\cA} - 1)\abs{\cS}$.
Consequently, 
the implicitly assumed  gradient  $\nabla f_\rho$, given by 
$\lim_{\delta \to \mathbf{0}, \delta \in \RR^{\abs{\cS} \times \abs{\cA}}} \abs{ f_\rho(\pi + \delta) - f_\rho(\pi) - \inner{\nabla f_\rho(\pi)}{\delta}} /  \norm{\delta} \to 0$,
is not well defined.
Given this observation, we then adopt the following definition of policy gradient for objective \eqref{formulation_single_obj} when considering direct policy parameterization.

\begin{definition}[Policy Gradient with Direct Parameterization]\label{def_policy_grad}
For any function of policy $f: \Pi \to \RR$, the policy gradient of $f$ with respect to $\pi$, denoted by $\nabla f(\pi)$, is the vector satisfying the following,
\begin{align}\label{policy_grad_def}
\lim_{\delta \to \mathbf{0}, \pi + \delta \in \Pi} \abs{ f(\pi + \delta) - f(\pi) - \inner{\nabla f(\pi)}{\delta}} /  \norm{\delta}_2 \to 0.
\end{align}
\end{definition}

Given Definition \ref{def_policy_grad}, it should be clear that for any policy $\pi \in \Pi$, if $\nabla f(\pi)$ exists, then it is unique. 
Definition \ref{def_policy_grad} slightly generalizes the notion of Fr\'echet derivative of objective \eqref{formulation_single_obj} as $\Pi$ is a closed set  in its affine span.
We then proceed to derive the policy gradient of $V^{\pi}_u(s)$ for a given uncertainty $u \in \cU$.

\begin{lemma}[Policy Gradient for Fixed Uncertainty with Direct Parameterization]\label{lemma_pg_standard}
Given $u \in \cU$ and a state $s \in \cS$, then the policy gradient of $V^{\pi}_u(s)$ with respect to $\pi$ is given by 
\begin{align*}
\nabla V^{\pi}_u (s) [s',a] = \tfrac{1}{1-\gamma} d_s^{\pi, u} (s') Q^{\pi}_u (s', a), ~~ \forall (s, a) \in \cS \times \cA,
\end{align*}
where $\nabla V^{\pi}_u (s) [s',a]$ denotes the entry of $\nabla V^{\pi}_u (s)$ corresponding to the $(s', a)$ state-action pair.
\end{lemma}

\begin{proof}
For MDP $\cM_u$, and any pair of policies $(\pi, \pi')$ with $\pi' = \pi + \delta$, we have 
\begin{align*}
V^{\pi'}_u(s) - V^{\pi}_u(s)
&\overset{(a)}{=} \tfrac{1}{1-\gamma} \EE_{s' \sim d_s^{\pi', u}} \inner{Q^{\pi}_u(s', \cdot)}{\pi'(\cdot|s') - \pi(\cdot|s')} \\
&= \tfrac{1}{1-\gamma} \sum_{s' \in \cS}  \sum_{a \in \cA}
d_s^{\pi', u}(s') Q^{\pi}_u(s', a) \rbr{\pi'(a|s') - \pi(a|s') } \\
& = \tfrac{1}{1-\gamma} \sum_{s' \in \cS}  \sum_{a \in \cA}
d_s^{\pi, u}(s') Q^{\pi}_u(s', a) \rbr{\pi'(a|s') - \pi(a|s') }
 \\
&  ~~~~~~ + 
 \tfrac{1}{1-\gamma} \sum_{s' \in \cS}  \sum_{a \in \cA}
(d_s^{\pi', u}(s') - d_s^{\pi, u}(s')) Q^{\pi}_u(s', a) \rbr{\pi'(a|s') - \pi(a|s') } \\
& = \underbrace{\tfrac{1}{1-\gamma} \sum_{s' \in \cS}  \sum_{a \in \cA}
d_s^{\pi, u}(s') Q^{\pi}_u(s', a) \delta(a|s') }_{(A)} \\
&  ~~~~~~ + 
\underbrace{ \tfrac{1}{1-\gamma} \sum_{s' \in \cS}  \sum_{a \in \cA}
(d_s^{\pi', u}(s') - d_s^{\pi, u}(s')) Q^{\pi}_u(s', a) \rbr{\pi'(a|s') - \pi(a|s') } }_{(B)},
\end{align*}
where equality $(a)$ follows directly from the performance difference lemma for standard MDPs \cite{kakade2002approximately,lan2021policy}.
It is clear that term $(A) = \inner{g}{\delta}$, where the entry of $g$ associated with $(s',a)$ state-action pair, denoted by $g(a|s')$, is given by 
$g(a|s') = d_s^{\pi, u}(s') Q^{\pi}_u(s' , a)$.
It remains to show that term $(B) = \smallO(\norm{\pi - \pi'}_2)$ where we identify $\pi$ as a matrix in $\RR^{\abs{\cS} \times \abs{\cA}}$.
To this end, is suffices to show  $\abs{d_s^{\pi', u}(s') - d_s^{\pi, u}(s')} = \cO (\norm{\pi - \pi'}_2)$ for any $s' \in \cS$.

Let us define $\PP^{\pi}_u : \cS \times \cS \to [0,1]$ by $\PP^{\pi}_u (s', s) = \sum_{a \in \cA} \PP_u(s'|s,a) \pi(a|s)$, then for any $\pi \in \Pi$,  we obtain
\begin{align}
d_{s}^{\pi, u}
= (1-\gamma) \sum_{t=0}^{\infty} \gamma^t (\PP^{\pi}_u)^t e_{s} = (1 -\gamma) \rbr{ 1- \gamma \PP^{\pi}_u}^{-1} e_s,
\end{align}
where in the last inequality we use the fact that $\rho( \gamma \PP^{\pi}_u) \leq \gamma < 1$.
Hence we have 
\begin{align*}
d^{\pi', u}_{s} - d^{\pi, u}_s 
& = (1-\gamma) \rbr{
 \rbr{ I - \gamma \PP^{\pi'}_{u}}^{-1} -  \rbr{  I - \gamma \PP^{\pi}_{u}}^{-1}
} e_s \\
& = 
(1-\gamma) \gamma
\rbr{ I - \gamma \PP^{\pi'}_{u} }^{-1}
\rbr{\PP^{\pi}_{u}  - \PP^{\pi'}_{u} }
\rbr{ I - \gamma \PP^{\pi}_{u} }^{-1}
 e_s ,
\end{align*}
where the last equality uses the matrix identity
$A^{-1} - B^{-1} = A^{-1} (B-A) B^{-1}$ for any invertible matrix $A, B$.
Note that 
\begin{align}\label{bound_matrix_ell1_norm}
\norm{\rbr{ I - \gamma \PP^{\pi}_{u} }^{-1}}_1
= \rbr{\min \cbr{ \norm{\rbr{ I - \gamma \PP^{\pi}_{u} } x}_1:  \norm{x}_1 = 1} }^{-1}
 \leq (1-\gamma)^{-1}
 \end{align}
  for any $\pi \in \Pi$ and $u \in \cU$,
we can then further obtain
\begin{align}\label{state_visitation_continuity_of_pi_fix_u}
\norm{d^{\pi', u}_{s} - d^{\pi, u}_s }_1 
\leq \tfrac{\gamma}{1-\gamma} \norm{ \PP^{\pi'}_u - \PP^{\pi}_u}_1 \overset{(a)}{=} \cO (\norm{\pi - \pi'}_\infty) \overset{(b)}{=} \cO (\norm{\pi - \pi'}_2), 
\end{align}
where the equality $(a)$ simply follows from the definition of $\PP^{\pi}_u$, and $(b)$ follows from the equivalence of norm.
The proof is then completed.
\end{proof}

 Lemma \ref{lemma_pg_standard} serves as an important stepping stone for establishing the existence of Fr\'echet subdifferential for policy objective $f_\rho$ in \eqref{formulation_single_obj}, when viewing it as an extended real-valued function with domain~$\Pi$.
 
 \begin{lemma}[Fr\'echet Subgradient of Robust MDP]\label{lemma_frechet_subgrad}
 Define $\overline{f}_\rho : \RR^{\abs{\cS} \abs{\cA}} \to \overline{R} = \RR \cup \cbr{+\infty}$ as the extended-value version of function $f_\rho$, that is, 
 $\overline{f}_\rho(\pi) = f_\rho(\pi)$ for any $\pi \in \Pi$ and $\overline{f}_\rho(\pi) = \infty$ otherwise. 
 For any $\pi \in \Pi$, let $\nabla f_\rho(\pi) \in \RR^{\abs{\cS}  \abs{\cA}}$ with the $(s,a)$-entry specified as 
 \begin{align}\label{frechet_subgrad_of_objective}
 \nabla f_\rho(\pi)[s,a] 
 = \tfrac{1}{1-\gamma} d_\rho^{\pi, u_\pi}(s) Q^\pi_{u_\pi}(s,a), ~~\forall (s,a) \in \cS \times \cA.
 \end{align}
Then $ \nabla f_\rho(\pi)$ is a Fr\'echet subgradient of $\overline{f}_\rho$ at $\pi$.
 \end{lemma}
 
 \begin{proof}
 Fixing $\pi$, it suffices to consider any $\pi'$ also in $\Pi$. Let $\delta = \pi' - \pi$, we have
 \begin{align*}
 V^{\pi'}_{u_{\pi'}} (s) - V^{\pi}_{u_\pi} (s) & =
 V^{\pi'}_{u_{\pi'}} (s) - V^{\pi'}_{u_{\pi}} (s)
 + V^{\pi'}_{u_{\pi}} (s) - V^{\pi}_{u_{\pi}} (s) \\
 & \overset{(a)}{\geq}  V^{\pi'}_{u_{\pi}} (s) - V^{\pi}_{u_{\pi}} (s) \\
 & \overset{(b)}{=} \tfrac{1}{1-\gamma} \sum_{s' \in \cS}  \sum_{a \in \cA}
d_s^{\pi, u_\pi}(s') Q^{\pi}_{u_\pi}(s', a) \delta(a|s') 
+
\smallO(\norm{\pi - \pi'}_2),
 \end{align*}
 where $(a)$ follows from the definition of $u_{\pi'}$, which guarantees $V^{\pi'}_{u_{\pi'}} (s) - V^{\pi'}_{u_{\pi}} (s) \geq 0$ for any  $s \in \cS$,
 and $(b)$ follows directly from the proof of Lemma \ref{lemma_pg_standard}.
 Thus we obtain from the definition of $f_\rho$ that 
 \begin{align*}
 f_\rho(\pi') - f_\rho(\pi) 
&  \geq 
 \tfrac{1}{1-\gamma} \sum_{s \in \cS} \rho(s) \sum_{s' \in \cS}  \sum_{a \in \cA}
d_s^{\pi, u_\pi}(s') Q^{\pi}_{u_\pi}(s', a) \delta(a|s') 
+
\smallO(\norm{\pi - \pi'}_2) \\
& =  \tfrac{1}{1-\gamma}  \sum_{s \in \cS}  \sum_{a \in \cA}
d_\rho^{\pi, u_\pi}(s) Q^{\pi}_{u_\pi}(s, a) \delta(a|s) 
+
\smallO(\norm{\pi - \pi'}_2),
 \end{align*}
 Dividing both sides of the previous relation by $\norm{\delta}$, and taking infimum over $\pi' \neq \pi$, and further taking $\norm{\delta} \downarrow 0$, we obtain 
 \begin{align*}
  \liminf_{\pi' \to \pi, \pi' \neq \pi} \rbr{ f_\rho(\pi') - f_\rho(\pi)  -  \tfrac{1}{1-\gamma}  \sum_{s \in \cS}  \sum_{a \in \cA}
d_\rho^{\pi, u_\pi}(s) Q^{\pi}_{u_\pi}(s, a) \delta(a|s) } \big/\norm{\pi' - \pi}_2 \geq 0.    
 \end{align*}
 Hence we conclude from the prior relation, and the definition of Fr\'echet subdifferential \cite{kruger2003frechet} that $\nabla f_\rho(\pi) \in \RR^{\abs{\cS}  \abs{\cA}}$ with the $(s,a)$-entry specified as in \eqref{frechet_subgrad_of_objective}
 is a Fr\'echet subgradient of $\overline{f}_\rho$ at  $\pi \in \Pi$.
 \end{proof}

With Lemma \ref{lemma_frechet_subgrad} in place, we proceed to discuss the differentiability of robust values, and determine the analytic form of the gradients. 
Our next lemma shows that the objective \eqref{formulation_single_obj} is indeed almost everywhere differentiable (in the sense of Definition \ref{def_policy_grad}), when taking the measure to be $(\abs{\cA} - 1){\abs{\cS}}$-dimensional Hausdorff measure. 
We remark that Hausdorff measure is a natural choice for our discussion of differentiability over $\Pi$, as it adapts to the low-dimensional nature of $\Pi$. 
On the other hand, choosing Lebesgue measure yields a trivial almost-everywhere-differentiable claim, as $\Pi$ itself takes a Lebesgue measure of zero in $\RR^{\abs{\cS} \abs{\cA}}$.  

\begin{lemma}[Almost-everywhere Differentiability of Robust MDP]\label{lemma_almost_differentiability}
The policy optimization objective \eqref{formulation_single_obj} is  everywhere differentiable in its domain $\Pi$ except in a zero-measure set (in $(\abs{\cA} - 1){\abs{\cS}}$-dimensional Hausdorff measure),
and the differentiability is defined as in Definition~\ref{def_policy_grad}.
\end{lemma}

\begin{proof}
We defer the proof to Appendix \ref{section_appendix} given its technical nature.
\end{proof}

Combining Lemma \ref{lemma_frechet_subgrad} and \ref{lemma_almost_differentiability}, we are now ready to show that for any $\pi \in \mathrm{ReInt}(\Pi)$, whenever $f_\rho$ is differentiable at $\pi$, then the policy gradient defined in the sense of Definition \ref{def_policy_grad} is  given exactly by \eqref{frechet_subgrad_of_objective} in Lemma \ref{lemma_frechet_subgrad}.

\begin{lemma}[Policy Gradient for Robust MDP with Direct Parameterization]\label{lemma_grad_rmdp_equiv_subgrad}
At any policy $\pi \in \mathrm{ReInt}(\Pi)$, except for a zero measure set (in $(\abs{\cA} -1)\abs{\cS}$-dimensional Hausdorff measure),  the policy optimization objective $f_\rho$ defined in \eqref{formulation_single_obj} is differentiable, and its policy gradient $\nabla f_\rho$ defined in the sense of Definition \ref{def_policy_grad}, is given by \eqref{frechet_subgrad_of_objective}.
\end{lemma}

\begin{proof}
The proof is deferred to Appendix \ref{section_appendix}.
\end{proof}

As the last result in this subsection, we show that if for any $\pi$, the corresponding worst-case uncertainty $u_\pi$ is unique, then the robust policy optimization objective \eqref{formulation_single_obj} is indeed differentiable everywhere inside $\mathrm{ReInt}(\Pi)$.
A sufficient condition can also be found in Lemma \ref{lemma_lipschitz_transition_wrt_policy}, which certifies the uniqueness of the worst-case uncertainty.

\begin{lemma}\label{lemma_everywhere_differentiability_with_unique_assump}
If for any $\pi \in \mathrm{ReInt}(\Pi)$, the corresponding worst-case environment defined in \eqref{def_worst_case_transition} is unique.
Then the policy optimization objective \eqref{formulation_single_obj} is differentiable everywhere in $\mathrm{ReInt}(\Pi)$ in the sense of Definition \ref{def_policy_grad}, and the policy gradient in given as \eqref{frechet_subgrad_of_objective}.
\end{lemma}

\begin{proof}
The proof is deferred to Appendix \ref{section_appendix}.
\end{proof}

Given our prior discussions on the  differentiability of objective \eqref{formulation_single_obj},  one can perhaps directly update the policy using the policy gradient specified in \eqref{frechet_subgrad_of_objective}. 
Moreover, the results developed in this subsection seems to be readily extensible for computing policy gradient when function approximation is adopted.

On the other hand, it should be noted that the task of calculating or estimating the gradient  for the robust MDP seems difficult when one can only access information through interactions with the nominal environment,  except for some special subclasses of uncertainty sets. 
In particular, although one can sample directly from the worst-case environment $\cM_{u_\pi}$ to perform estimation, in practice we believe this is against the principle of pursuing robustness.
Instead, a much more desirable alternative is to train the policy in a fixed (nominal) environment, without actually deploying it in its worst-case environment.
More importantly, due to the nonconvex and potentially nonsmooth landscape, one typically could only get a sublinear convergence to a stationary point in the best-iterate sense, which would further require additional landscape analysis for showing approximate optimality of the learned policy  \cite{wang2022policy}.

In the next subsection, we introduce an alternative viewpoint of solving the robust MDP based on variational inequalities.
As we will demonstrate in Section \ref{sec_rpmd}, solving the variational inequality allows us to bypass the aforementioned difficulties,
and develop a policy mirror descent method with fast global convergence guarantees.

\subsection{A Variational Inequality Perspective}

To proceed, we introduce a characterization that upper bounds the difference of robust values for two policies, and serves as a keystone in the ensuing algorithmic developments.

\begin{lemma}\label{lemma_perf_diff}
For any pair of policies $\pi, \pi'$, we have 
\begin{align}\label{ineq_perf_diff}
V^{\pi'}_r(s) - V^{\pi}_r(s)
\leq \tfrac{1}{1-\gamma} \EE_{s' \sim d^{\pi', u_{\pi'}}_s} \inner{Q^{\pi}_r}{\pi' - \pi}_{s'}.
\end{align}
\end{lemma}

\begin{proof}
Let $\xi^{\pi'}_u(s) \coloneqq  \cbr{(s_t, a_t)}_{t \geq 0}$ denote the trajectory generated from by $\pi'$ within MDP $\cM_{u_{\pi'}}$, with initial state set as $s$. 
That is, $a_t \sim \pi'(\cdot|s_t)$, $s_{t+1} \sim \PP_{u_\pi'} (\cdot | s_t , a_t)$ for all $t \geq 0$ and $s_0 = s$.
We know that 
\begin{align}
  V^{\pi'}_r ( s) - V^{\pi}_r(s)  
 & \overset{(a)}{=}  V^{\pi'}_{u_{\pi'}} ( s) - V^{\pi}_r(s) \nonumber \\
& =   \EE_{\xi^{\pi'}_u(s)}  \sbr{\sum_{t=0}^\infty \gamma^t c(s_t, a_t) } - V^{\pi}_r(s) \nonumber \\
& =  \EE_{\xi^{\pi'}_u(s)}  \sbr{\sum_{t=0}^\infty \gamma^t \sbr{ c(s_t, a_t) +  V^{\pi}_r(s_t) - V^{\pi}_r(s_t)} } - V^{\pi}_r(s)  \nonumber \\
& \overset{(b)}{=}  \EE_{\xi^{\pi'}_u(s)}  \sbr{\sum_{t=0}^\infty \gamma^t \sbr{ c(s_t, a_t) +  \gamma V^{\pi}_r(s_{t+1}) - V^{\pi}_r(s_t)} }
+  \EE_{\xi^{\pi'}_u(s)} \sbr{V^{\pi}_r(s_0) }
 - V^{\pi}_r(s) \nonumber \\
&  \overset{(c)}{=}  \EE_{\xi^{\pi'}_u(s)}  \sbr{\sum_{t=0}^\infty \gamma^t \sbr{ c(s_t, a_t) +  \gamma V^{\pi}_r(s_{t+1}) - V^{\pi}_r(s_t)} } \nonumber \\
&  \overset{(d)}{\leq}  \EE_{\xi^{\pi'}_u(s)}  \sbr{\sum_{t=0}^\infty \gamma^t \sbr{ Q^{\pi}_r(s_t, a_t) - V^{\pi}_r(s_t)} } \label{perf_diff_ineq_d} \\
&  \overset{(e)}{=}  \tfrac{1}{1-\gamma} \sum_{s' \in \cS} d^{\pi', u_{\pi'}}_s(s')\sum_{a' \in \cA}   \pi'(a' | s') Q^{\pi}_r(s', a')
- \tfrac{1}{1-\gamma} \sum_{s' \in \cS} d^{\pi',  u_{\pi'}}_s(s') \sum_{a' \in \cA} \pi(a' | s') Q^{\pi}_r(s', a') \nonumber \\
& = \tfrac{1}{1-\gamma} \EE_{s' \sim d^{\pi',  u_{\pi'}}_s} \inner{Q^{\pi}_r}{\pi' - \pi}_{s'},\nonumber
\end{align}
where $(a)$ follows from \eqref{robust_and_standard_value_relation}; $(b)$ follows from moving $V^{\pi}_r(s_0)$ outside the summation; 
$(c)$ follows from definition that $s_0 = s$; $(e)$ follows from the definition of $d^{\pi,  u_{\pi}}_s$ in \eqref{def_discounted_visit_measure}, and the relation between $V^{\pi}_r$ and $Q^{\pi}_r$ in \eqref{robust_value_q_relation}.
It remains to show that $(d)$ holds. 

For any $t \geq 0$, we have 
\begin{align*}
\EE_{\xi^{\pi'}_u(s)} \sbr{c(s_t, a_t) + \gamma V^{\pi}_r(s_{t+1})} 
& \overset{(a')}{=} \EE_{\xi^{\pi'}_u(s)} \sbr{c(s_t, a_t) + \gamma  \sum_{s' \in \cS} \PP_{u_{\pi'}}(s'| s_t, a_t) V^{\pi}_r(s') } \\
& \leq \EE_{\xi^{\pi'}_u(s)} \sbr{c(s_t, a_t) + \gamma \max_{u(\cdot|s_t, a_t)}  \sum_{s' \in \cS}  \PP_{u}(s'| s_t, a_t) V^{\pi}_r(s') } \\
& \overset{(b')} =  \EE_{\xi^{\pi'}_u(s)} \sbr{Q^{\pi}_r (s_t, a_t) },
\end{align*}
where $(a')$ follows from  the definition of $\xi^{\pi'}_u(s)$, and $(b')$ follows from the property of $Q^{\pi}_r$ \eqref{def_robust_q}.
Thus relation $(d)$ follows immediately from the prior observation and the linearity of expectation.
\end{proof}

By applying Lemma \ref{lemma_perf_diff}, we know that for any policy $\pi \in \Pi$ and an optimal robust policy $\pi^*$, we have
\begin{align*}
0 \leq V^\pi_r(s) - V^{\pi^*}_r(s) \leq 
\tfrac{1}{1-\gamma} \EE_{s' \sim d_{s}^{\pi, u_\pi}} \inner{Q^{\pi^*}_r}{\pi - \pi^*}_s
= \inner{\cF_s(\pi, \pi^*)}{\pi - \pi^*},
\end{align*}
where $\cF_s: \Pi \times \Pi \to \RR^{\abs{\cS} \abs{\cA}}$ is defined by 
\begin{align*}
\cF_s(\pi, \pi')[s',a] = \tfrac{1}{1-\gamma} d_s^{\pi, u_\pi}(s')  Q^{\pi'}_r(s',a)  ~~ \forall (s', a) \in \cS \times \cA.
\end{align*}
Given the optimality of $\pi^*$, we then know that 
$
\inner{\cF(\pi, \pi^*)}{\pi - \pi^*}  \geq 0
$
for all $\pi \in \Pi$.
Now for any $\pi \in \Pi$, let $\alpha \in (0,1)$, we define $\pi_\alpha = \alpha \pi + (1-\alpha) \pi^*$.
Substituting $\pi_\alpha$ into the prior relation, we obtain  
$
\inner{\cF_s(\pi_\alpha, \pi^*)}{\pi_\alpha - \pi^*} 
= \alpha \inner{\cF_s(\pi_\alpha, \pi^*)}{\pi - \pi^*} \geq 0.
$
Consequently,  
\begin{align}\label{minty_vi}
\inner{\cF_s(\pi_\alpha, \pi^*)}{\pi - \pi^*} \geq 0
\end{align} for any $\alpha \in (0,1)$.
It is straightforward  to verify that the visitation measure $d_s^{\pi, u}$ is a continuous function of $(\pi ,u)$ (see proof of Lemma \ref{lemma_pg_standard}).
On the other hand, note that given \eqref{def_worst_case_transition}, the worst-case uncertainty $u_\pi(\cdot|s,a) \in \partial \sigma_{\cU_{s,a}}(V^{\pi}_r)$, where $\sigma_X(\cdot)$ denotes the support function of set $X$, and $\partial f$ denotes the subdifferential of function $f$. 
For simplicity, let us assume that $\partial \sigma_{\cU_{s,a}}(V^{\pi}_r)$ is always a singleton for any $\pi$ (See Lemma \ref{lemma_lipschitz_transition_wrt_policy} for example).

Now let $\alpha \downarrow 0$, we have $\pi_\alpha \to \pi^*$.
Since $V^{\pi}_r$ is Lipschitz continuous in $\pi$ (see Lemma \ref{lemma_lipschitz_value_wrt_policy} for an elementary proof), then $V^{\pi_\alpha}_r \to V^{\pi^*}_r$.
Thus given that $\sigma_{\cU_{s,a}}$ is a closed and proper convex function, and the fact that the subdifferential map of a closed convex function is closed (Theorem 24.4, \cite{rockafellar1970convex}), we know that any 
limit point of $u_{\pi_\alpha}(\cdot|s,a)$ is also a subgradient of $\sigma_{\cU_{s,a}}(V^{\pi^*}_r)$,  and hence the limit point is indeed unique and the worst-case uncertainty for $\pi^*$.
Denoting this limit point as $u_{\pi^*}$, then by taking $\alpha \downarrow 0$ in  \eqref{minty_vi}, we obtain
\begin{align}
\lim_{\alpha \downarrow 0} \inner{\cF_s(\pi_\alpha, \pi^*)}{\pi - \pi^*} \geq 0 
& \Rightarrow
 \inner{\cG_s(\pi^*)}{\pi - \pi^*} \geq 0,   \label{vi_for_robust_statewise} \\
 \cG_s(\pi^*) [s', a] & = \tfrac{1}{1-\gamma} d_s^{\pi^*, u_{\pi^*}}(s')  Q^{\pi^*}_r(s',a) , ~~ \forall (s', a) \in \cS \times \cA. \label{vi_operator_g}
\end{align}

In view of the  discussions in the last two paragraphs, \eqref{vi_for_robust_statewise}  suggests  us to find the optimal robust policy  via solving the following variational inequality (VI), 
\begin{align}\label{vi_for_robust}
\inner{ \EE_{s \sim \rho} \sbr{ \cG_s(\pi^*)}}{\pi - \pi^*} \geq 0, ~~ \forall \pi \in \Pi.
\end{align} 

Interested readers might find that VI \eqref{vi_for_robust} has a close analogy for solving non-robust MDPs, constructed in \cite{lan2021policy}.
Specifically, for a non-robust, discounted finite MDP, the  optimal policy satisfies the following VI,
\begin{align}
\langle \EE_{s \sim \nu^*} \sbr{ \cG^N_s(\pi^*)}, & ~ \pi - \pi^*\rangle  \geq 0,  \label{vi_for_non_robust} \\ 
 \cG^N_s(\pi^*) [s', a]  =& \tfrac{1}{1-\gamma} d_s^{\pi^*}(s')  Q^{\pi^*}(s',a) , ~~ \forall (s', a) \in \cS \times \cA, \nonumber
\end{align} 
where $\nu^*$ denotes the stationary state distribution induced by the optimal policy $\pi^*$.
One can clearly see that the only difference between the non-robust VI \eqref{vi_for_non_robust} and the robust VI \eqref{vi_for_robust} is the additional role of worst-case environment for the latter.
To solve the non-robust version of VI constructed therein, \cite{lan2021policy} exploits the key fact that the constructed VI satisfies the so-called generalized monotonicity:
\begin{align}\label{monotonicity_non_robust}
\langle \EE_{s \sim \nu^*} \sbr{ \cG^N_s(\pi)}, & \pi - \pi^*\rangle   = \EE_{s \sim \nu^*} \sbr{V^{\pi}(s) - V^{\pi^*}(s)} = f_{\nu^*}(\pi) - f_{\nu^*}(\pi^*)  > 0,
\end{align}
for any policy $\pi \notin \Pi^*$.

However,  we next show that unlike the performance difference lemma for standard MDPs \cite{kakade2002approximately,lan2021policy}, for which 
\eqref{ineq_perf_diff} holds with equality,
 the inequality in Lemma \ref{lemma_perf_diff} seems unavoidable in the robust setting. 
Consequently, the VI for the robust MDP \eqref{vi_for_robust} no longer satisfies the generalized monotonicity as defined in \eqref{monotonicity_non_robust}.

\begin{figure}[b!]
\centering
  \includegraphics[width=0.55\textwidth]{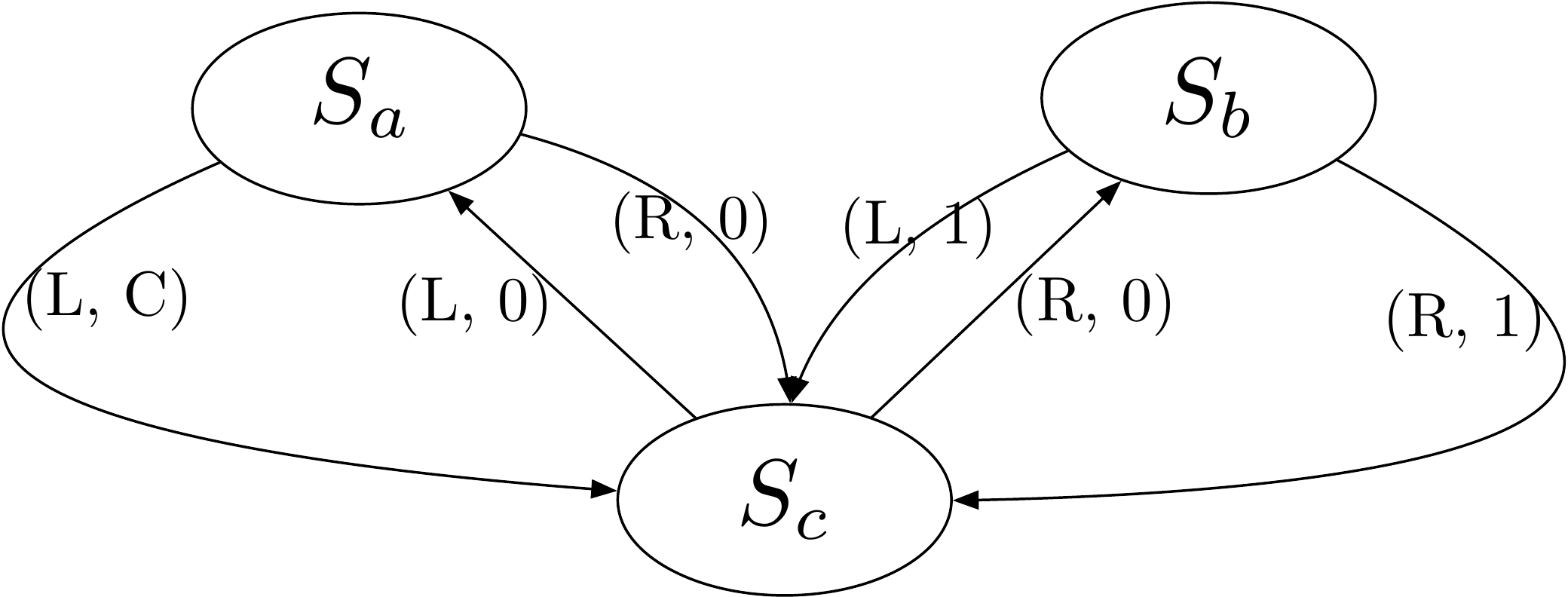}
  \caption{ 
Example of a robust MDP where Lemma \ref{lemma_perf_diff} holds with strict inequality. 
  }
  \label{fig:example}
\end{figure}

\begin{proposition}\label{prop_weak_signal}
There exists a robust MDP instance $\cM_\cU$ such that
\begin{align}\label{ineq_zero_signal}
V^{\pi^*}_r(s) - V^{\pi}_r(s)
< 0, ~~  \tfrac{1}{1-\gamma} \EE_{s' \sim d^{\pi^*,  u_{\pi^*}}_s} \inner{Q^{\pi}_r}{\pi^* - \pi}_{s'} = 0 
\end{align}
 holds  for at least one state $s \in \cS$.
 In particular, $\pi^*$ corresponds to a solution of \eqref{formulation_robust}, and $\pi$ is a strictly suboptimal policy.
 Moreover, denoting  $\nu^* \coloneqq \nu^{\pi^*, u_{\pi^*}}$ as the stationary state distribution of the optimal policy $\pi^*$ within its worst-case MDP $\cM_{u_{\pi^*}}$, we have 
 \begin{align}\label{non_monotone_robust_vi}
 0 = \langle \EE_{s \sim \nu^*} \sbr{ \cG_s(\pi)}, & \pi - \pi^*\rangle <  f_{\nu^*}(\pi) - f_{\nu^*}(\pi^*)  .
 \end{align}
\end{proposition}

\begin{proof}
Consider the MDP $\cM$ that contains three states $\cS = \cbr{S_a, S_b, S_c}$, and each state is associated with two actions $\cA = \cbr{L, R}$.
The transition of $\cM$, denoted by $\PP$, is fully deterministic, and is illustrated in Figure \ref{fig:example}. Each edge starts from the current state, and ends at the next state, with its arc $(a, c)$ consisting of the action $a$,  and the cost $c$ associated with the action $a$. We assume such a cost occurs immediately after the action is made, and is independent of the next state.
In addition, we assume $C > 1$. 

Let us consider the scenario where the uncertain environment can arbitrarily manipulate the transition probabilities at state $S_c$ except there is no returning to itself, and has no manipulation strength at other states.
That is,  for any $a \in \cA$, we have $\cP_{S_c, a} = \Delta_{\cbr{S_a, S_b}}$.
On the other hand, for any $s \neq S_c$, and for any $a \in \cA$, we have $\cU_{s, a} = \mathbf{0}$.

We consider a pair of policies $\pi, \pi^*$, defined by 
\begin{align*}
\pi(L|S_a) = 1, ~ \pi^*(R|S_a) = 1, ~ \pi(\cdot | S_b) = \pi(\cdot|S_c) = \mathrm{Unif}(\cA), ~ \pi^*(\cdot | S_b) = \pi^*(\cdot|S_c) = \mathrm{Unif}(\cA).
\end{align*}

Since $C > 1$, and the fact that  $\pi(L|S_a) = 1$,  it should be clear that for policy $\pi$, the worst case transition $\PP_{u_{\pi}}$ should satisfy 
\begin{align*}
\PP_{u_{\pi}}(S_a | S_c, L) = 1, \PP_{u_{\pi}}(S_a | S_c, R) = 1.
\end{align*}
Consequently, the robust value of $\pi$ is given by 
\begin{align*}
V^{\pi}_r(S_c) = \tfrac{\gamma C}{1 - \gamma^2}, ~ V^{\pi}_r(S_b) = 1 + \tfrac{\gamma^2 C}{1- \gamma^2}, ~ V^{\pi}_r(S_a) = \tfrac{C}{1-\gamma^2}.
\end{align*}
On the other hand, for policy $\pi^*$, since $\pi^*(R|S_a) = 1$ and $1 > 0$, the worst case transition $\PP_{u_{\pi^*}}$ should satisfy
\begin{align*}
\PP_{u_{\pi^*}}(S_b | S_c, L) = 1, \PP_{u_{\pi}^*}(S_b | S_c, R) = 1.
\end{align*}
It should also be clear that $\pi^*$ is an optimal robust policy.
From the previous observations, simple calculation yields 
\begin{align}
\sum_{s' \in \cS} \PP_{u_{\pi^*}}(s'| S_c, L) V^{\pi}_r(s') = V^{\pi}_r(S_b) <  V^{\pi}_r(S_a)  = \max_{u \in \cU} \sum_{s' \in \cS} \PP_{u}(s'| S_c, L) V^{\pi}_r(s')
 ,  \label{example_ineq1}\\
\sum_{s' \in \cS} \PP_{u_{\pi^*}}(s'| S_c, R) V^{\pi}_r(s') = V^{\pi}_r(S_b) <V^{\pi}_r(S_a) =  \max_{u \in \cU} \sum_{s' \in \cS} \PP_{u}(s'| S_c, R) V^{\pi}_r(s')
 . \label{example_ineq2}
\end{align}
Now let $\xi^{\pi^*}_u(S_c)$ denote the trajectory generated by $\pi^*$ within $\cM_{\pi^*}$, starting from state $S_c$. 
Then we have 
\begin{align*}
\EE_{\xi^{\pi^*}_u(S_c)} \big[c(s_0, a_0) + \gamma V^{\pi}_r(s_{1})\big]
& = \EE_{\xi^{\pi^*}_u(S_c)} \big[c(s_0, a_0) + \gamma  \sum_{s' \in \cS} \PP_{u_{\pi^*}}(s'| s_0, a_0) V^{\pi}_r(s') \big] \\
& = \big[c(S_c, L) + \gamma  \sum_{s' \in \cS} \PP_{u_{\pi^*}}(s'| S_c, L) V^{\pi}_r(s') \big] \pi^*(L | S_c) \\
& ~~~~~~ +  \big[c(S_c, R) + \gamma  \sum_{s' \in \cS} \PP_{u_{\pi^*}}(s'| S_c, R) V^{\pi}_r(s') \big] \pi^*(R | S_c)
 \\
 & <  \big[c(S_c, L) + \gamma \max_{u \in \cU}   \sum_{s' \in \cS} \PP_{u}(s'| S_c, L) V^{\pi}_r(s') \big] \pi^*(L | S_c) \\
& ~~~~~~ +  \big[c(S_c, R) + \gamma  \max_{u \in \cU}  \sum_{s' \in \cS} \PP_{u}(s'| S_c, R) V^{\pi}_r(s') \big] \pi^*(R | S_c)
 \\
& = \EE_{\xi^{\pi^*}_u(S_c)} \big[c(s_0, a_0) + \gamma \max_{u \in \cU}  \sum_{s' \in \cS}  \PP_{u}(s'| s_0, a_0) V^{\pi}_r(s') \big] \\
&  =  \EE_{\xi^{\pi^*}_u(S_c)} \sbr{Q^{\pi}_r (s_0, a_0) },
\end{align*}
where the strict inequality follows from observation \eqref{example_ineq1} and \eqref{example_ineq2}.
Thus by repeating the same argument in the proof of Lemma \ref{lemma_perf_diff}, we know that inequality $(d)$ in \eqref{perf_diff_ineq_d} holds with strict inequality for $s = S_c$. 
Since $(d)$ is the only inequality in the proof of Lemma \ref{lemma_perf_diff}, we conclude that 
\begin{align}\label{strict_at_Sc}
V^{\pi^*}_r(S_c) - V^{\pi}_r(S_c)
< \tfrac{1}{1-\gamma} \EE_{s' \sim d^{\pi^*, u_{\pi^*}}_{S_c}} \inner{Q^{\pi}_r}{\pi^* - \pi}_{s'}.
\end{align}

We proceed to show that $ \EE_{s' \sim d^{\pi^*, u_{\pi^*}}_{S_c}} \inner{Q^{\pi}_r}{\pi^* - \pi}_{s'} = 0$ and hence establishing \eqref{ineq_zero_signal}.
To see this, note that 
$Q^{\pi}_r(S_c, \cdot) = \gamma V^{\pi}_r(S_a)$, and hence 
$ \inner{Q^{\pi}_r}{\pi^* - \pi}_{S_c} = 0$. 
In addition, we also have $Q^{\pi}_r(S_b, \cdot) = 1 + \gamma V^{\pi}_r(S_c)$, and hence 
$ \inner{Q^{\pi}_r}{\pi^* - \pi}_{S_b} = 0$.
Finally, note that $d_{S_c}^{\pi^*, u_{\pi^*}} (S_a) = 0$, and hence $\EE_{s' \sim d^{\pi^*, u_{\pi^*}}_{S_c}} \inner{Q^{\pi}_r}{\pi^* - \pi}_{s'} = 0$.

To show \eqref{non_monotone_robust_vi}, it suffices to note that $\nu^* = \mathrm{Unif}(\cbr{S_b, S_c})$ is the stationary state distribution of $\pi^*$ within MDP $\cM_{u_{\pi^*}}$. 
With the same arguments as in the last paragraph we can show that $\EE_{s' \sim d^{\pi^*, u_{\pi^*}}_{S_b}} \inner{Q^{\pi}_r}{\pi^* - \pi}_{s'} = 0$.
Hence we obtain  
\begin{align*}
 0 =  \EE_{s \sim \nu^*}  \EE_{s' \sim d^{\pi^*, u_{\pi^*}}_{s}} \inner{Q^{\pi}_r}{\pi^* - \pi}_{s'}  < (1-\gamma) \rbr{f_{\nu^*}(\pi) - f_{\nu^*}(\pi^*)},
\end{align*}
where the strict inequality follows from \eqref{strict_at_Sc}.
The claim then follows from the definition of  $\cG_s$ in \eqref{vi_operator_g}.
\end{proof}

The construction of strict inequality in Proposition \ref{prop_weak_signal} suggests that  \eqref{ineq_perf_diff} in Lemma \ref{lemma_perf_diff}  fails to  characterize even the simplest change of robust values when switching the policy. 
 In particular, as  illustrated in \eqref{ineq_zero_signal}, improvement of value when switching from $\pi$ to the optimal $\pi^*$ seems not being captured by the aggregated inner product defined in the second term of \eqref{ineq_zero_signal} at all. 
A closer look into the constructed example shows that the state $S_a$ is the only state where the inner product is nonzero (positive), but when changing to the optimal policy, the state $S_a$ is never visited by the optimal policy via $d_{\tilde{s}}^{\pi^*, u_{\pi^*}}$, for any $\tilde{s} \neq S_a$. It seems to suggest that the fundamental difficulty causing the insufficiency of Lemma \ref{lemma_perf_diff} is due to the fact that the very state (i.e., state $S_a$) that leads to the improvement of policy is {\it never visited by the improved policy $\pi^*$ within its worst-case environment $\cM_{u^{\pi^*}}$}.
 
The following lemma aims to address the previously observed difficulty.

\begin{lemma}\label{lemma_q_signal}
For any policy $\pi$, let $u_{\pi}$ denote its worst-case uncertainty defined in \eqref{def_worst_case_transition}, then we have 
\begin{align}\label{ineq_q_signal}
\EE_{s' \sim d_s^{\pi^*, u_{\pi}}} \sbr{
 \inner{Q^{\pi}_r }{\pi - \pi^*}_{s'} 
 } 
 \geq 
  (1-\gamma) \rbr{V_r^{\pi}(s) - V^{\pi^*}_r(s)  }.
\end{align}
\end{lemma}
\begin{proof}
We have 
\begin{align*}
\EE_{s' \sim d_s^{\pi^*, u_{\pi}}} \sbr{
 \inner{Q^{\pi}_r }{\pi - \pi^*}_{s'} 
 } 
 & = 
 - \EE_{s' \sim d_s^{\pi^*, u_{\pi}}} \sbr{
 \inner{Q^{\pi}_r }{\pi^* - \pi}_{s'} 
 } \\
& \overset{(a)}{=} -
\EE_{s' \sim d_s^{\pi^*, u_{\pi}}} \sbr{
 \inner{Q^{\pi}_{u_{\pi}} }{\pi^*- \pi}_{s'} 
 } \\
 & \overset{(b)}{=} 
 (1-\gamma) \rbr{
 V^{\pi}_{u_{\pi}}(s) - V^{\pi^*}_{u_\pi}(s)
 } \\
 & \overset{(c)}{=} 
 (1-\gamma) \rbr{
 V^{\pi}_{r}(s) - V^{\pi^*}_{u_\pi}(s)
 }  \\
 & \overset{(d)}{\geq} 
  (1-\gamma) \rbr{
 V^{\pi}_{r}(s) - V^{\pi^*}_{r}(s)
 } ,
\end{align*}
where $(a)$ follows from \eqref{robust_and_standard_q_relation},
$(c)$ follows from the \eqref{robust_and_standard_value_relation},
$(d)$ follows from the definition of $V^{\pi^*}_r = \max_{u \in \cU} V^{\pi^*}_u$.
We now proceed to establish $(b)$. 

The proof of $(b)$ follows similar lines as in the proof of Lemma \ref{lemma_perf_diff}.
 Let $\xi^{\pi^*}_{u_{\pi}}(s) \coloneqq  \cbr{(s_t, a_t)}_{t \geq 0}$ denote the trajectories generated by policy $\pi^*$, within  MDP $\cM_{u_{\pi}}$, with initial state set as $s$. 
That is, $a_t \sim \pi^*(\cdot|s_t)$, $s_{t+1} \sim \PP_{u_{\pi}} (\cdot | s_t , a_t)$ for all $t \geq 0$ and $s_0 = s$.
We then have 
\begin{align*}
  V^{\pi^*}_{u_{\pi}} ( s) - V^{\pi}_{u_{\pi}}(s)  
& =   \EE_{\xi^{\pi^*}_{u_{\pi}}(s)}  \sbr{\sum_{t=0}^\infty \gamma^t c(s_t, a_t) } - V^{\pi}_{u_{\pi}}(s) \\
& =  \EE_{\xi^{\pi^*}_{u_{\pi}}(s)}  \sbr{\sum_{t=0}^\infty \gamma^t \sbr{ c(s_t, a_t) +  V^{\pi}_{u_{\pi}}(s_t) - V^{\pi}_{u_{\pi}}(s_t)} } - V^{\pi}_{u_{\pi}}(s) \\
& =  \EE_{\xi^{\pi^*}_{u_{\pi}}(s)}  \sbr{\sum_{t=0}^\infty \gamma^t \sbr{ c(s_t, a_t) +  \gamma V^{\pi}_{u_{\pi}}(s_{t+1}) - V^{\pi}_{u_{\pi}}(s_t)} }
+  \EE_{\xi^{\pi^*}_{u_{\pi}}(s)} \sbr{V^{\pi}_{u_{\pi}}(s_0) }
 - V^{\pi}_{u_{\pi}}(s) \\
& = \EE_{\xi^{\pi^*}_{u_{\pi}}(s)}  \sbr{\sum_{t=0}^\infty \gamma^t \sbr{ c(s_t, a_t) +  \gamma V^{\pi}_{u_{\pi}}(s_{t+1}) - V^{\pi}_{u_{\pi}}(s_t)} } \\
&  =  \EE_{\xi^{\pi^*}_{u_{\pi}}(s)}  \sbr{\sum_{t=0}^\infty \gamma^t \sbr{ Q^{\pi}_{u_{\pi}}(s_t, a_t) - V^{\pi}_{u_{\pi}}(s_t)} } \\
& \overset{(e)}{=} \tfrac{1}{1-\gamma} \sum_{s' \in \cS} d^{\pi^*, u_{\pi}}_s(s')\sum_{a' \in \cA}   \pi^*(a' | s') Q^{\pi}_{u_{\pi}}(s', a')
- \tfrac{1}{1-\gamma} \sum_{s' \in \cS} d^{\pi^*,  u_{\pi}}_s(s') \sum_{a' \in \cA} \pi(a' | s') Q^{\pi}_{u_{\pi}}(s', a') \\
& = \tfrac{1}{1-\gamma} \EE_{s' \sim d^{\pi^*,  u_{\pi}}_s} \inner{Q^{\pi}_{u_{\pi}}}{\pi^* - \pi}_{s'} ,
\end{align*}
where $(e)$ follows from \eqref{standard_value_q_relation}.
\end{proof}

By comparing Proposition \ref{prop_weak_signal} and Lemma \ref{lemma_q_signal}, it should be clear that the key to inequality \eqref{ineq_q_signal} is the combination of policy  and the environment when choosing the state-visitation measure.
Instead of choosing the worst-case environment $u_{\pi^*}$ for $\pi^*$ as in Proposition \ref{prop_weak_signal} (Lemma \ref{lemma_perf_diff}), we choose the worst-case environment $u_{\pi}$ of the current policy $\pi$.
In particular, returning back to the constructed example in Proposition \ref{prop_weak_signal}, we see that the state (i.e., state $S_a$) that leads to the improvement of policy can be now visited by the improved policy $\pi^*$, if the environment is still fixed as the worst-case environment $\cM_{u_\pi}$ for the original policy $\pi$.


Before we end our discussion in this section, it is  worth mentioning some similarities that robust state-action value function $Q^{\pi}_r$ shares  with the (sub)gradient in the convex optimization literature, in the sense that with proper aggregation over states: 
(1)  it provides an upper bound on the local value changes, as shown in Lemma \ref{lemma_perf_diff}, much similar to using gradient-based local quadratic approximation for smooth objectives \cite{lan2020first} (with the quadratic term being identically $0$); 
(2)  its inner product with the direction to the optimal policy is lower bounded by the optimality gap, as shown in Lemma \ref{lemma_q_signal}, similar to the (sub)-gradient for convex objectives \cite{lan2020first, nesterov2003introductory}.
These observations thus suggest using $Q^{\pi}_r$ as the first-order information to update the policy.
Nevertheless, it should be noted that the objective \eqref{formulation_single_obj} is neither convex nor smooth \cite{agarwal2021theory, wang2022policy}.
In the following section, we develop the robust policy mirror descent method that formalizes this intuition, and establish its computational efficiency in finding an optimal policy.

\section{Robust Policy Mirror Descent}\label{sec_rpmd}

\begin{algorithm}[b!]
    \caption{The robust policy mirror descent (RPMD) method}
    \label{alg_rpmd}
    \begin{algorithmic}
    \STATE{\textbf{Input:} Initial policy $\pi_0$ and stepsizes $\{\eta_k\}_{k\geq 0}$.}
    \FOR{$k=0, 1, \ldots$}
	\STATE{Update policy:
			\vspace{-0.1in}
	\begin{align*}
	\pi_{k+1} (\cdot| s) =
	\argmin_{p(\cdot|s) \in \Delta_{\cA}} \eta_k  \inner{Q_r^{\pi_k}(s, \cdot)}{ p(\cdot| s ) }  + D^{p}_{\pi_k}(s),  ~ \forall s \in \cS.
	\end{align*}}
	\vspace{-0.15in}
    \ENDFOR
    \end{algorithmic}
\end{algorithm}

In this section, we introduce the deterministic robust policy mirror descent method (RPMD) for solving \eqref{formulation_single_obj}. 
 RPMD assumes access to an oracle that outputs the robust state-action value function $Q^{\pi}_r$ of a  given policy $\pi$.
At each iteration, the RPMD method (Algorithm \ref{alg_rpmd}) updates the policy according to 
\begin{align}\label{rpmd_update}
\pi_{k+1} (\cdot| s) =
	\argmin_{p(\cdot|s) \in \Delta_{\cA}} \eta_k  \inner{Q_r^{\pi_k}(s, \cdot)}{ p(\cdot| s ) }  + D^{p}_{\pi_k}(s),  ~ \forall s \in \cS.
\end{align}
Here $D^{\pi}_{\pi'}(s)$ denotes the Bregman divergence between policy $\pi(\cdot|s)$ and $\pi'(\cdot|s)$, defined as 
\begin{align}\label{eq_bregman_def}
D^{\pi}_{\pi'}(s) =
w(\pi(\cdot|s)) - w(\pi'(\cdot|s)) - \inner{\partial w(\pi'(\cdot|s))}{\pi(\cdot|s) - \pi'(\cdot|s)},
\end{align}
where $w: \RR^{\abs{\cA}} \to \RR$ is a strictly convex function, also known as the distance-generating function, and $\partial w(p)$ denotes a subgradient of $w$ at $p \in \RR^{\abs{\cA}}$.
Common distance-generating functions include $w(p) = \norm{p}_2^2$, which induces $D^{\pi}_{\pi'}(s) = \norm{\pi(\cdot|s) - \pi'(\cdot|s)}_2^2$;  
and $w(p) = \sum_{a \in \cA} p_a \log p_a \coloneqq - \mathrm{Ent}(p)$, which induces 
\begin{align*}
D^{\pi}_{\pi'}(s) = \sum_{a \in \cA} \pi(a|s) \log\rbr{\pi(a|s)/ \pi'(a|s)} \coloneqq \mathrm{KL}(\pi(\cdot|s) \Vert \pi'(\cdot|s)).
\end{align*} 
We will also denote  $D_{w} = \max_{\pi \in \Pi}\max_{s\in \cS} D^{\pi}_{\pi_0} (s)  $, which is finite for many practical divergences when $\pi_0$ corresponds to the uniform policy, including the previously introduced KL-divergence and the squared $\ell_2$ distance.

It is also worth mentioning that RPMD with KL-divergence yields an equivalent policy update for the natural policy gradient method \cite{kakade2001natural} applied to the softmax parameterization for the robust MDP objective \eqref{formulation_single_obj}, by directly applying Lemma \ref{lemma_everywhere_differentiability_with_unique_assump}.
The same equivalence has also been observed for solving the non-robust MDP problem in the literature.

We proceed to establish some general convergence properties of RPMD.
To begin with,  the following lemma characterizes each policy update of RPMD.

\begin{lemma}\label{lemma_three_point}
For any $p\in \Delta_{\cA}$ and any $s\in\cS$,  we have 
\begin{align}\label{eq_three_point}
\eta_k \inner{Q^{\pi_k}_r(s, \cdot)}{\pi_{k+1}(\cdot|s) - p } + D^{\pi_{k+1}}_{\pi_k}(s) 
\leq D^{p}_{\pi_k}(s) -  D^{p}_{\pi_{k+1}}(s).
\end{align}
\end{lemma}

\begin{proof}
From the optimality condition of the RPMD update \eqref{rpmd_update}, we have for any $p \in \Delta_{\cA}$,
\begin{align}\label{rpmd_update_opt_condition_raw}
\eta_k \inner{Q^{\pi_k}_r(s, \cdot)}{p - \pi_{k+1}(\cdot|s)} 
+ \inner{\nabla D^{\pi_{k+1}}_{\pi_k}(s)}{ p -\pi_{k+1}(\cdot|s)} \geq 0,
\end{align}
In addition, given the definition of Bregman divergence,  we have the following identity
\begin{align*}
 \inner{\nabla D^{\pi_{k+1}}_{\pi_k}(s)}{ p -\pi_{k+1}(\cdot|s)}  = 
 D^{p}_{\pi_k} (s) - D^{\pi_{k+1}}_{\pi_k}(s) - D^{p}_{\pi_{k+1}}(s).
\end{align*}
Combining the previous observation with \eqref{rpmd_update_opt_condition_raw}, we immediately obtain the result.
\end{proof}

The next lemma then establishes the basic convergence properties of RPMD.
\begin{lemma}\label{lemma_rpmd_convergence_prop}
At each iteration of RPMD, we have 
\begin{align}\label{rpmd_basic_recursion}
f_\rho(\pi_{k+1}) - f_\rho(\pi^*) 
\leq  \rbr{ 1 - \tfrac{1-\gamma}{M}} \rbr{f_\rho(\pi_k) - f_\rho(\pi^*)} 
+ \tfrac{1}{M \eta_k} \EE_{s \sim d_\rho^{\pi^*, u_k}} D^{\pi^*}_{\pi_k}(s)
- \tfrac{1}{M \eta_k}  \EE_{s \sim d_\rho^{\pi^*, u_k}} D^{\pi^*}_{\pi_{k+1}}(s),
\end{align}
where $ M  \coloneqq \sup_{u \in \cU} \norm{d_\rho^{\pi^*, u} / \rho}_\infty $ is finite whenever $\mathrm{supp}(\rho) = \cS$.
\end{lemma}

\begin{proof}
By plugging in $p = \pi_k$ in \eqref{eq_three_point}, we obtain 
\begin{align}\label{rpmd_monotone}
\eta_k \inner{Q^{\pi_k}_r(s, \cdot)}{\pi_{k+1}(\cdot|s)  - \pi_{k}(\cdot|s)}  
\leq 
-D^{\pi_k}_{\pi_{k+1}}(s) - D^{\pi_{k+1}}_{\pi_k}(s) \leq 0, ~\forall s \in \cS.
\end{align}
On the other hand, plugging in $p = \pi^*$ in \eqref{eq_three_point}, we obtain 
\begin{align}\label{eq_decomp_rpmd}
\underbrace{\eta_k \inner{Q^{\pi_k}_r(s, \cdot)}{\pi_{k+1}(\cdot|s) - \pi_k(\cdot|s) }}_{(A)} + 
\underbrace{\eta_k \inner{Q^{\pi_k}_r(s, \cdot)}{\pi_{k}(\cdot|s) - \pi^*(\cdot|s) }}_{(B)} + 
 D^{\pi_{k+1}}_{\pi_k}(s) 
\leq D^{\pi^*}_{\pi_k}(s) -  D^{\pi^*}_{\pi_{k+1}}(s).
\end{align}
 We let $u_k = u_{\pi_k}$ denote the worst-case uncertainty of policy $\pi_k$ for any $k \geq 0$.

To handle term $(A)$, note that
\begin{align}
V^{\pi_{k+1}}_r(s) - V^{\pi_k}_r(s)
& \overset{(a)}{\leq} \tfrac{1}{1-\gamma} \EE_{s' \sim d^{\pi_{k+1}, u_{k+1}}_{s}} \nonumber
\inner{Q^{\pi_k}_r}{\pi_{k+1} - \pi_k}_{s'} \nonumber\\
& \overset{(b)}{\leq} \tfrac{d^{\pi_{k+1}, u_{k+1}}_{s}(s)}{1-\gamma}  \inner{Q^{\pi_k}_r(s, \cdot)}{\pi_{k+1}(\cdot|s) - \pi_k(\cdot|s)} \nonumber\\
& \overset{(c)}{\leq} \inner{Q^{\pi_k}_r(s, \cdot)}{\pi_{k+1}(\cdot|s) - \pi_k(\cdot|s)} \leq 0, \label{rpmd_gradient_monotone}
\end{align}
where $(a)$ is due to Lemma \ref{lemma_perf_diff}; $(b)$ is due to \eqref{rpmd_monotone};
$(c)$ is due to \eqref{rpmd_monotone}, and the observation  that $d^{\pi_{k+1}, u_{k+1}}_{s}(s) \geq (1-\gamma)$ for all $s\in \cS$.

To handle term $(B)$,  we make use of Lemma \ref{lemma_q_signal}.
Specifically,  we obtain from \eqref{ineq_q_signal} that
\begin{align}\label{rpmd_vi}
\EE_{s' \sim d_s^{\pi^*, u_{k}}} \sbr{
 \inner{Q^{\pi_k}_r }{\pi_k - \pi^*}_{s'} 
 } 
 \geq 
  (1-\gamma) \rbr{V_r^{\pi}(s) - V^{\pi^*}_r(s)  } \geq 0.
\end{align}
Hence by aggregating \eqref{eq_decomp_rpmd} across different states with weights set as $ d_s^{\pi^*, u_k}$, and making use of  \eqref{rpmd_gradient_monotone} and \eqref{rpmd_vi}, 
we obtain 
\begin{align*}
&  \EE_{s' \sim d_s^{\pi^*, u_k}}\sbr{ \eta_k \rbr{ V_r^{\pi_{k+1}}(s') - V_r^{\pi_k}(s') }}
+ (1-\gamma) \eta_k \rbr{ V^{\pi_{k}}_r(s) - V^{\pi^*}_r(s) }  \\
\leq &  \EE_{s' \sim d_s^{\pi^*, u_k}} D^{\pi^*}_{\pi_k}(s')
- \EE_{s' \sim d_s^{\pi^*, u_k}} D^{\pi^*}_{\pi_{k+1}}(s').
\end{align*}
By further taking expectation with respect to $s \sim \rho$, and making use of  $V^{\pi_{k+1}}_r (s) \leq V^{\pi_k}_r(s)$ given \eqref{rpmd_gradient_monotone}, and the definition of $M \coloneqq \sup_{u \in \cU} \norm{d_\rho^{\pi^*, u} / \rho}_\infty < \infty$,
\begin{align*}
& M \sbr{ f_\rho(\pi_{k+1}) - f_\rho(\pi_k) } 
+ (1-\gamma) \sbr{ f_\rho(\pi_k) - f_\rho (x^*) } \leq 
\tfrac{1}{\eta_k} \EE_{s \sim d_\rho^{\pi^*, u_k}} D^{\pi^*}_{\pi_k}(s)
- \tfrac{1}{\eta_k}  \EE_{s \sim d_\rho^{\pi^*, u_k}} D^{\pi^*}_{\pi_{k+1}}(s),
\end{align*}
which after simple arrangement, yields the desired claim. 
\end{proof}

\subsection{Convergence with Increasing Stepsizes}

We proceed to show that by applying exponentially increasing stepsizes, RPMD achieves linear convergence in the optimality gap.
\begin{theorem}\label{thrm_rpmd_linear_convergence}
Suppose the stepsizes $\cbr{\eta_k}$ satisfy 
\begin{align}\label{rpmd_linear_stepsize}
\eta_k \geq \eta_{k-1} \rbr{ 1 - \tfrac{1-\gamma}{M}}^{-1} M', ~~ \forall k \geq 1,
\end{align}
where $M' = \sup_{u, u' \in \cU} \norm{d_{\rho}^{\pi^*, u} / d_\rho^{\pi^*, u'} }_\infty $ is finite whenever $\mathrm{supp}(\rho) = \cS$.
Then for any iteration $k$, RPMD produces policy $\pi_k$ satisfying 
\begin{align*}
f_\rho(\pi_k) - f_\rho(\pi^*) 
 \leq 
 \rbr{ 1 - \tfrac{1-\gamma}{M}}^k \rbr{f_\rho(\pi_0) - f_\rho(\pi^*)} 
 + \rbr{ 1 - \tfrac{1-\gamma}{M}}^{k-1} \tfrac{ D_w}{M \eta_0}.
\end{align*}
\end{theorem}

\begin{proof}
By recursively applying inequality \eqref{rpmd_basic_recursion} from $t = 0$ to $k-1$, we obtain 
\begin{align*}
f_\rho(\pi_k) - f_\rho(\pi^*) 
& \leq 
 \rbr{ 1 - \tfrac{1-\gamma}{M}}^k \rbr{f_\rho(\pi_0) - f_\rho(\pi^*)} 
 + \rbr{ 1 - \tfrac{1-\gamma}{M}}^{k-1} \tfrac{1}{M \eta_0} \EE_{s \sim d_\rho^{\pi^*, u_0}} D^{\pi^*}_{\pi_0}(s) \\
& ~~ + \underbrace{\tfrac{1}{M}\sum_{t=1}^{k-1} \sbr{  \rbr{ 1 - \tfrac{1-\gamma}{M}}^{k-t-1} \tfrac{1}{\eta_t} \EE_{s \sim d_\rho^{\pi^*, u_t}} D^{\pi^*}_{\pi_t}(s)  
-  \rbr{ 1 - \tfrac{1-\gamma}{M}}^{k-t} \tfrac{1}{\eta_{t-1}} \EE_{s \sim d_\rho^{\pi^*, u_{t-1} }} D^{\pi^*}_{\pi_t}(s)
}}_{(C)}.
\end{align*}
Given the definition of stepsizes $\cbr{\eta_k}$ in \eqref{rpmd_linear_stepsize} and $M'$,  it should be clear that term $(C) \leq 0$.
In conclusion, we obtain that whenever \eqref{rpmd_linear_stepsize} holds,  
\begin{align*}
f_\rho(\pi_k) - f_\rho(\pi^*) 
& \leq 
 \rbr{ 1 - \tfrac{1-\gamma}{M}}^k \rbr{f_\rho(\pi_0) - f_\rho(\pi^*)} 
 + \rbr{ 1 - \tfrac{1-\gamma}{M}}^{k-1} \tfrac{1}{M \eta_0} \EE_{s \sim d_\rho^{\pi^*, u_0}} D^{\pi^*}_{\pi_0}(s) \\
 & \leq 
 \rbr{ 1 - \tfrac{1-\gamma}{M}}^k \rbr{f_\rho(\pi_0) - f_\rho(\pi^*)} 
 + \rbr{ 1 - \tfrac{1-\gamma}{M}}^{k-1} \tfrac{ D_w}{M \eta_0}.
\end{align*}
The proof is then completed. 
\end{proof}

It should be noted that whenever $\cU = \cbr{u}$ is a singleton,   solving the  robust MDP $\cM_{\cU}$ is equivalent to solving a non-robust MDP.
Then we have $M' = 1$, and 
$M =  \norm{d_\rho^{\pi^*, u} / \rho}_\infty$ reduces to the mismatch coefficient defined in the analysis of policy gradient methods for solving nominal MDP \cite{xiao2022convergence, agarwal2021theory, kakade2002approximately}. 
In this case, RPMD admits a simple learning rate scaling rule given by $\eta_t \geq \gamma^{-1} \eta_{t-1}$, and the obtained convergence rate for RPMD matches exactly the fastest existing rate of convergence of first-order methods for solving the nominal MDP \cite{xiao2022convergence, li2022homotopic, lan2021policy}.

The obtained linear convergence in Theorem \ref{thrm_rpmd_linear_convergence} seems to be new among the existing literature of first-order policy-based methods applied to solving the robust MDPs. 
In addition,  although the constant $M'$ is in general unknown, one can simply provide a pessimistic upper bound of $M'$.
In particular, for $\rho = \mathrm{Unif}(\cS)$, it suffices to take $M' = \abs{\cS} / (1-\gamma)$.
Finally, the increasing-stepsize schemes are not the only ones that can certify the convergence of RPMD.
In the next subsection, we will establish the convergence of RPMD with a constant-stepsize scheme, which allows one to completely avoid the estimation of $M'$.

\subsection{Convergence with Constant Stepsizes}\label{sec_rpmd_const_stepsize}

In this subsection, we establish the convergence of RPMD with constant stepsize.
In particular, for RPMD with Euclidean Bregman divergence  and a constant stepsize of $\eta$, we establish an $\cO(\max \cbr{1/k, 1/\sqrt{\eta k}})$ rate of convergence for arbitrary rectangular uncertainty set.
For RPMD with a general class of Bregman divergences, we also establish a similar convergence rate for rectangular uncertainty set that satisfies the so-called relative strong convexity, with details deferred to Appendix \ref{sec_rpmd_general_bregman_const_step}.


To proceed, we first establish the following simple fact on the asymptotic stationarity of the RPMD iterates. 

\begin{lemma}\label{lemma_stationary}
For any $k \geq 1$, the iterates in RPMD with constant stepsizes $\eta_k = \eta > 0$ satisfy 
\begin{align}\label{ineq_stationarity}
\tfrac{1}{\eta} \sum_{t=0}^{k-1} \rbr{ D^{\pi_t}_{\pi_{t+1}} (s) + D^{\pi_{t+1} }_{\pi_{t}} (s) }
\leq V^{\pi_0}_r(s) - V^{\pi^*}_r(s).
\end{align}
\end{lemma}

\begin{proof}
Given \eqref{rpmd_gradient_monotone} and \eqref{rpmd_monotone}, we obtain that 
\begin{align*}
V^{\pi_{k+1}}_r(s) - V^{\pi_k}_r(s)
\leq 
-\tfrac{1}{\eta} D^{\pi_k}_{\pi_{k+1}}(s) - \tfrac{1}{\eta}  D^{\pi_{k+1}}_{\pi_k}(s) \leq 0, ~\forall s \in \cS.
\end{align*}
Summing up the prior relation from $t=0$ to $k-1$, we obtain 
\begin{align*}
\tfrac{1}{\eta}  \sum_{t=0}^{k-1} \rbr{ D^{\pi_t}_{\pi_{t+1}} (s) + D^{\pi_{t+1} }_{\pi_{t}} (s) }
\leq V^{\pi_0}_r(s) - V^{\pi_k}_r(s) 
\leq V^{\pi_0}_r(s) - V^{\pi^*}_r(s).
\end{align*}
The proof is then completed.
\end{proof}

Combining Lemma \ref{lemma_stationary} and Lemma \ref{lemma_rpmd_convergence_prop}, we are able to establish the following convergence characterization for RPMD with Euclidean Bregman divergence, when adopting any constant-stepsize scheme.
\begin{theorem}\label{thrm_rpmd_Euclidean_constant_step}
Let $w(\cdot) = \norm{\cdot}_2^2$ be the distance-generating function, and $\eta_t = \eta$ for all $t \geq 0$ and any $\eta > 0$, then at each iteration, RPMD outputs policy $\pi_k$ satisfying 
\begin{align}\label{eq_rpmd_Euclidean_constant_step}
f_\rho(\pi_k) - f_\rho(\pi^*) 
\leq \tfrac{M}{(1-\gamma)k} \rbr{f_\rho(\pi_0) - f_\rho(\pi^*) } 
+  \sqrt{\tfrac{18 \abs{\cS}^2 }{\eta k (1-\gamma)^3}},
\end{align} 
where $M$ is defined as in Lemma \ref{lemma_rpmd_convergence_prop}.
\end{theorem}

\begin{proof}
By summing up inequality \eqref{rpmd_basic_recursion} from $t = 0$ to $k-1$, we obtain 
\begin{align*}
& f_\rho(\pi_k) - f_\rho(\pi^*) 
+ \tfrac{1-\gamma}{M} \sum_{t=1}^{k-1} \rbr{ f_\rho(\pi_t) - f_\rho(\pi^*)}  \\
\leq &  f_\rho(\pi_0) - f_\rho(\pi^*) 
+\sum_{t=0}^{k-1} \rbr{ \tfrac{1}{M \eta} \EE_{s \sim d_\rho^{\pi^*, u_t}} D^{\pi^*}_{\pi_t}(s)
- \tfrac{1}{M \eta}  \EE_{s \sim d_\rho^{\pi^*, u_t}} D^{\pi^*}_{\pi_{t+1}}(s) } 
\end{align*}
Note that $D^{\pi'}_{\pi}(s) = \norm{\pi'(\cdot| s) - \pi(\cdot|s) }_2^2$, and hence 
\begin{align*}
D^{\pi^*}_{\pi_t}(s) -  D^{\pi^*}_{\pi_{t+1}}(s)
\leq \sqrt{ 18 D^{\pi_t}_{\pi_{t+1}} (s)},
\end{align*}
where we use the fact that for any $a, b, c \in \Delta_\cA$, 
\begin{align*}
\sum_{i} (a_i - b_i)^2 - \sum_i (a_i - c_i)^2 & \leq \norm{b + c - 2a}_2 \norm{b -c }_2 
\leq \sqrt{18} \norm{b-c}_2.
\end{align*}
Thus we obtain 
\begin{align*}
 f_\rho(\pi_k) - f_\rho(\pi^*) 
+ \tfrac{1-\gamma}{M} \sum_{t=1}^{k-1} \rbr{ f_\rho(\pi_t) - f_\rho(\pi^*)}  \leq 
f_\rho(\pi_0) - f_\rho(\pi^*) 
 + \tfrac{\sqrt{18} }{M\eta}  \sum_{t=0}^{k-1} \EE_{s \sim d_\rho^{\pi^*, u_t}} \sqrt{ D^{\pi_t}_{\pi_{t+1}}(s) }.
\end{align*}
To proceed, we will make use of Lemma \ref{lemma_stationary}, which gives 
\begin{align*}
 \rbr{ \sum_{t=0}^{k-1} \sqrt{ D^{\pi_t}_{\pi_{t+1}}(s) }}^2 
 \leq k \sum_{t=0}^{k-1} D^{\pi_t}_{\pi_{t+1}}(s)
 \leq k \eta \rbr{V^{\pi_0}_r(s) - V^{\pi^*}_r(s) } \leq \tfrac{k \eta }{1-\gamma}, ~~ \forall s \in \cS,
\end{align*}
which combined with the previous inequality, gives 
\begin{align*}
 f_\rho(\pi_k) - f_\rho(\pi^*) 
+ \tfrac{1-\gamma}{M} \sum_{t=1}^{k-1} \rbr{ f_\rho(\pi_t) - f_\rho(\pi^*)}  \leq 
f_\rho(\pi_0) - f_\rho(\pi^*) 
+  \tfrac{\abs{\cS}}{M } \sqrt{\tfrac{18 k }{\eta (1-\gamma)}}.
\end{align*}
By \eqref{rpmd_gradient_monotone}, we know that $f_\rho(\pi_{t+1}) \leq f_\rho(\pi_t)$ for all $t \geq 0$.
The desired claim then follows immediately after combining this observation with the above inequality.
\end{proof}

It should be noted that the dependence on the size of the state space in the second term of \eqref{eq_rpmd_Euclidean_constant_step} can be simply eliminated by using a large stepsize.
Theorem \ref{thrm_rpmd_Euclidean_constant_step} states that to find an $\epsilon$-optimal policy using constant stepsizes RPMD with Euclidean divergence, 
the iteration complexity is bounded by $\cO \rbr{\max \cbr{1/\epsilon, 1/(\eta \epsilon^2)}}$, which reduces to $\cO(1/\epsilon)$ when choosing $\eta = \Theta(1/\epsilon)$. 
To the best of our knowledge, the fastest existing rate of convergence of Euclidean-divergence based first-order methods for solving robust MDP is $\cO(1/\epsilon^3)$ with $\eta = \Theta(\epsilon)$ \cite{wang2022policy}, which studies a more restrictive subclass of  polyhedral rectangular uncertainty set, and further requires a delicate smoothing technique and  selecting the best policy among historical policy iterates.
In contrast, the RPMD comes with a much stronger convergence guarantee for  general rectangular uncertainty sets, while at the same time enjoying greater algorithmic simplicity.

\section{Stochastic Robust Policy Mirror Descent}\label{sec_srpmd}

In this section, we extend the deterministic RPMD method to the stochastic settings, where the exact information of the robust state-action value function $Q^{\pi}_r$ is not available.
The stochastic robust policy mirror descent (SRPMD) instead uses the stochastic estimator $Q^{\pi,\xi}_r$ to update the policy, where $\xi$ denotes the samples used for the construction of the stochastic estimator.

\begin{algorithm}[t!]
    \caption{The stochastic robust policy mirror descent (SRPMD) method}
    \label{alg_srpmd}
    \begin{algorithmic}
    \STATE{\textbf{Input:} Initial policy $\pi_0$ and stepsizes $\{\eta_k\}_{k\geq 0}$.}
    \FOR{$k=0, 1, \ldots$}
	\STATE{Update policy:
			\vspace{-0.1in}
	\begin{align*}
	\pi_{k+1} (\cdot| s) =
	\argmin_{p(\cdot|s) \in \Delta_{\cA}} \eta_k  \inner{Q_r^{\pi_k, \xi_k}(s, \cdot)}{ p(\cdot| s ) }  + D^{p}_{\pi_k}(s),  ~ \forall s \in \cS.
	\end{align*}}
	\vspace{-0.15in}
    \ENDFOR
    \end{algorithmic}
\end{algorithm}

At iteration $k$, given a stochastic estimator $Q^{\pi_k, \xi_k}_r$, the SRPMD method (Algorithm \ref{alg_srpmd}) updates the policy according to 
\begin{align}\label{update_srpmd}
\pi_{k+1} (\cdot| s) =
	\argmin_{p(\cdot|s) \in \Delta_{\cA}} \eta_k  \inner{Q_r^{\pi_k, \xi_k}(s, \cdot)}{ p(\cdot| s ) }  + D^{p}_{\pi_k}(s),  ~ \forall s \in \cS.
\end{align}
The convergence of SRPMD assumes the following noise condition on the noisy estimate $\{Q^{\pi_k, \xi_k}_r\}$:
\begin{align}\label{stochastic_noise_condition}
\EE_{\xi_k} \norm{ Q^{\pi_k, \xi_k}_r - Q^{\pi_k}_r}_{\infty} \leq e_k.
\end{align}
We will also define  $\delta_k =   Q^{\pi_k, \xi_k}_r - Q^{\pi_k}_r$ for all $k \geq 0$.
Similar to Lemma \ref{lemma_rpmd_convergence_prop}, we first establish the following generic convergence property of the SRPMD method.

\begin{lemma}\label{lemma_srpmd_convergence_prop}
At each iteration of SRPMD, we have 
\begin{align}
 f_\rho(\pi_{k+1}) - f_\rho(\pi^*) 
\leq  \rbr{ 1 - \tfrac{1-\gamma}{M}} \rbr{f_\rho(\pi_k) - f_\rho(\pi^*)} 
+  \tfrac{1}{M \eta_k} \EE_{s \sim d_\rho^{\pi^*, u_k}} D^{\pi^*}_{\pi_k}(s) \nonumber   \\
- \tfrac{1}{M \eta_k}  \EE_{s \sim d_\rho^{\pi^*, u_k}} D^{\pi^*}_{\pi_{k+1}}(s) 
+ \tfrac{4}{1-\gamma}  \norm{\delta_k}_\infty \label{srpmd_basic_recursion}
,
\end{align}
where $ M $ is defined as in Lemma \ref{lemma_rpmd_convergence_prop}.
\end{lemma}

\begin{proof}
First, following the same lines as in the proof of  Lemma \ref{lemma_three_point}, we have 
\begin{align}\label{eq_three_point_stoch}
\eta_k \inner{Q^{\pi_k, \xi_k}_r(s, \cdot)}{\pi_{k+1}(\cdot|s) - p } + D^{\pi_{k+1}}_{\pi_k}(s) 
\leq D^{p}_{\pi_k}(s) -  D^{p}_{\pi_{k+1}}(s).
\end{align}
Thus by letting $p = \pi_k$ in \eqref{eq_three_point_stoch}, we obtain 
\begin{align}\label{stoch_three_point_progress}
\eta_k \inner{Q^{\pi_k, \xi_k}_r(s, \cdot)}{\pi_{k+1}(\cdot|s) - \pi_k(\cdot|s) }  
\leq -  D^{\pi_k}_{\pi_{k+1}}(s) - D^{\pi_{k+1}}_{\pi_k}(s) \leq 0.
\end{align}
On the other hand, by plugging in $p = \pi^*$ in the above relation, 
we obtain 
\begin{align}
& \underbrace{\eta_k \inner{Q^{\pi_k, \xi_k}_r(s, \cdot)}{\pi_{k+1}(\cdot|s) - \pi_k(\cdot|s) }}_{(A)} + 
\underbrace{\eta_k \inner{Q^{\pi_k, \xi_k}_r(s, \cdot)}{\pi_{k}(\cdot|s) - \pi^*(\cdot|s) }}_{(B)} + 
 D^{\pi_{k+1}}_{\pi_k}(s)  \nonumber \\
\leq &  D^{\pi^*}_{\pi_k}(s) -  D^{\pi^*}_{\pi_{k+1}}(s). \label{stoch_decomp}
\end{align}
 Let $u_k = u_{\pi_k}$ denote the worst-case uncertainty of policy $\pi_k$.
To handle term $(A)$, note that 
\begin{align}
V^{\pi_{k+1}}_r(s) - V^{\pi_k}_r(s)
& \overset{(a)}{\leq} \tfrac{1}{1-\gamma} \EE_{s' \sim d^{\pi_{k+1}, u_{k+1}}_{s}} \nonumber
\inner{Q^{\pi_k}_r}{\pi_{k+1} - \pi_k}_{s'} \nonumber\\
&=
\sum_{s' \in \cS}  \tfrac{d^{\pi_{k+1}, u_{k+1}}_{s}(s')}{1-\gamma} \sbr{
\inner{Q_r^{\pi_k, \xi_k}(s',\cdot)}{\pi_{k+1}(\cdot|s') - \pi_{k}(\cdot|s')  }
+ \inner{\delta_k(s',\cdot)}{\pi_{k+1}(\cdot|s') - \pi_{k}(\cdot|s')  }
}
\nonumber
\\
&\leq \tfrac{1}{1-\gamma} 
\sum_{s' \in \cS}  d^{\pi_{k+1}, u_{k+1}}_{s}(s') \sbr{
\inner{Q_r^{\pi_k, \xi_k}(s',\cdot)}{\pi_{k+1}(\cdot|s') - \pi_{k}(\cdot|s')  }
+ 2 \norm{\delta_k}_\infty
}
\nonumber 
\\
& \overset{(b)}{\leq} \tfrac{d^{\pi_{k+1}, u_{k+1}}_{s}(s)}{1-\gamma}  \inner{Q^{\pi_k, \xi_k}_r(s, \cdot)}{\pi_{k+1}(\cdot|s) - \pi_k(\cdot|s)}
+ \tfrac{2}{1-\gamma} \norm{\delta_k}_\infty  \nonumber\\
&  \overset{(c)}{\leq}   \inner{Q^{\pi_k, \xi_k}_r(s, \cdot)}{\pi_{k+1}(\cdot|s) - \pi_k(\cdot|s)} +  \tfrac{2}{1-\gamma} \norm{\delta_k}_\infty \nonumber 
\\  & \leq   \tfrac{2}{1-\gamma} \norm{\delta_k}_\infty, \label{stoch_approx_monotone_value}
\end{align} 
where $(a)$ uses Lemma \ref{lemma_perf_diff}, 
$(b)$ uses \eqref{stoch_three_point_progress}, and $(c)$ uses again \eqref{stoch_three_point_progress} and the fact that $d^{\pi_{k+1}, u_{k+1}}_{s}(s) \geq 1-\gamma$.
Hence we obtain from inequality $(c)$ in the last relation that 
\begin{align}\label{stoch_term_A}
(A) \geq \eta_k \rbr{V^{\pi_{k+1}}_r(s)  - V^{\pi_k}_r(s)} - \tfrac{2\eta_k}{1-\gamma} \norm{\delta_k}_\infty.
\end{align}
For term $(B)$, we have 
\begin{align}
\EE_{s' \sim d_s^{\pi^*, u_{k}}} \sbr{
 \inner{Q^{\pi_k, \xi_k}_r }{\pi_k - \pi^*}_{s'} 
 } 
&  = \EE_{s' \sim d_s^{\pi^*, u_{k}}} \sbr{
 \inner{Q^{\pi_k}_r }{\pi_k - \pi^*}_{s'} 
 + \inner{\delta_k}{\pi_k - \pi^*}_{s'}
 } \nonumber \\
&  \geq 
 (1-\gamma) \rbr{V_r^{\pi_k}(s) - V^{\pi^*}_r(s)  } - 2 \norm{\delta_k}_\infty, \label{stoch_term_B}
\end{align}
where the inequality follows from \eqref{rpmd_vi}. 
Hence by combining \eqref{stoch_decomp}, \eqref{stoch_term_A} and \eqref{stoch_term_B}, we obtain 
\begin{align*}
\EE_{s' \sim d_s^{\pi^*, u_{k}}} \rbr{V^{\pi_{k+1}}_r(s')  - V^{\pi_k}_r(s') - \tfrac{2 \norm{\delta_k}_\infty}{1-\gamma} }
+   (1-\gamma) \rbr{V_r^{\pi}(s) - V^{\pi^*}_r(s)  }  + \tfrac{1}{\eta_k} \EE_{s' \sim d_s^{\pi^*, u_{k}}} D^{\pi_{k+1}}_{\pi_k}(s') \\
\leq 
\tfrac{1}{\eta_k} \EE_{s' \sim d_s^{\pi^*, u_{k}}}  D^{\pi^*}_{\pi_k}(s') -  \tfrac{1}{\eta_k} \EE_{s' \sim d_s^{\pi^*, u_{k}}}  D^{\pi^*}_{\pi_{k+1}}(s')
 + 2 \norm{\delta_k}_\infty
 .
\end{align*}
By using  \eqref{stoch_approx_monotone_value},   the definition of $M$, and further taking $s \sim \rho$ in the above relation, we conclude that   
\begin{align*}
M \EE_{s \sim \rho} \sbr{V^{\pi_{k+1}}_r(s)  - V^{\pi_k}_r(s)}
+   (1-\gamma) \EE_{s\sim \rho} \rbr{V_r^{\pi}(s) - V^{\pi^*}_r(s)  }  + \EE_{s \sim d_\rho^{\pi^*, u_{k}}} D^{\pi_{k+1}}_{\pi_k}(s) \\
\leq 
 \EE_{s \sim d_\rho^{\pi^*, u_{k}}}  D^{\pi^*}_{\pi_k}(s) -  \EE_{s \sim d_\rho^{\pi^*, u_{k}}}  D^{\pi^*}_{\pi_{k+1}}(s)
 +  \tfrac{2(1-\gamma) + 2M }{1-\gamma} \norm{\delta_k}_\infty
 .
\end{align*}
The claim follows immediately after simple rearrangement to the above inequality.
\end{proof}

By specializing Lemma \ref{lemma_srpmd_convergence_prop} with exponentially increasing stepsizes, we obtain the following linear convergence of SRPMD up to a noise level determined by the noise in the stochastic estimation $Q^{\pi_k, \xi_k}$.

\begin{theorem}\label{thrm_srpmd_linear_convergence}
Suppose the stepsizes $\cbr{\eta_k}$ in SRPMD satisfy 
\begin{align}\label{srpmd_linear_stepsize}
\eta_k \geq \eta_{k-1} \rbr{ 1 - \tfrac{1-\gamma}{M}}^{-1} M', ~~ \forall k \geq 1,
\end{align}
where  $M'$ is defined as in Theorem \ref{thrm_rpmd_linear_convergence}.
Then for any iteration $k$, SRPMD produces policy $\pi_k$ satisfying 
\begin{align*}
\EE \sbr{ f_\rho(\pi_k) - f_\rho(\pi^*) }
 \leq 
 \rbr{ 1 - \tfrac{1-\gamma}{M}}^k \rbr{f_\rho(\pi_0) - f_\rho(\pi^*)} 
 + \rbr{ 1 - \tfrac{1-\gamma}{M}}^{k-1} \tfrac{ D_w}{M \eta_0} \\
 + \tfrac{4}{1-\gamma} \sum_{t = 0}^{k-1}  
\rbr{1 - \tfrac{1-\gamma}{M}}^{k - t - 1} e_t
,
\end{align*}
where $M$ is defined as in Lemma \ref{lemma_rpmd_convergence_prop}.
In particular, if we have $\EE_{\xi_k} \norm{ Q^{\pi_k, \xi_k}_r - Q^{\pi_k}}_{\infty} \leq e$ for all $k \geq 0$, then 
\begin{align}\label{srpmd_linear_up_to_precision}
\EE \sbr{ f_\rho(\pi_k) - f_\rho(\pi^*) }
\leq \rbr{ 1 - \tfrac{1-\gamma}{M}}^k \rbr{f_\rho(\pi_0) - f_\rho(\pi^*)} 
 + \rbr{ 1 - \tfrac{1-\gamma}{M}}^{k-1} \tfrac{ D_w}{M \eta_0} 
 + \tfrac{4 M e}{(1-\gamma)^2}.
\end{align}
\end{theorem}

\begin{proof}
The proof follows from similar lines as the proof of Theorem \ref{thrm_rpmd_linear_convergence}, except we will make use of Lemma \ref{lemma_srpmd_convergence_prop} instead of Lemma \ref{lemma_rpmd_convergence_prop}, and taking expectation with respect to $\cbr{\xi_t}$ in the end.
\end{proof}

In view of \eqref{srpmd_linear_up_to_precision} in Theorem \ref{thrm_srpmd_linear_convergence}, the last-iterate of SRPMD converges linearly up to the noise-level of 
$\cO(M e / (1-\gamma)^2)$, where $e$ characterizes the quality of the estimated robust state-action value function.

We then proceed to establish the convergence of SRPMD with a constant stepsize, by focusing on the Euclidean divergence considered in Section \ref{sec_rpmd_const_stepsize}.
Similar to Lemma \ref{lemma_stationary}, we first make the following simple observations regarding the policies generated by SRPMD.

\begin{lemma}\label{lemma_stationary_srpmd}
For any $k \geq 1$, the iterates in SRPMD with constant stepsizes $\eta_k = \eta > 0$ satisfy 
\begin{align}\label{ineq_stationarity_srpmd}
\tfrac{1}{\eta} \sum_{t=0}^{k-1} \rbr{ D^{\pi_t}_{\pi_{t+1}} (s) + D^{\pi_{t+1} }_{\pi_{t}} (s) }
\leq V^{\pi_0}_r(s) - V^{\pi^*}_r(s)
+  \tfrac{2}{1-\gamma} \sum_{t= 0}^{k-1} \norm{\delta_k}_\infty
.
\end{align}
\end{lemma}

\begin{proof}
Given \eqref{stoch_three_point_progress} and  inequality $(c)$ in \eqref{stoch_approx_monotone_value}, we obtain that 
\begin{align*}
 \tfrac{1}{\eta} \rbr{ D^{\pi_k}_{\pi_{k+1}}(s)  + D^{\pi_{k+1}}_{\pi_k}(s)}  \leq  
V^{\pi_k}_r(s) -  V^{\pi_{k+1}}_r(s) + \tfrac{2}{1-\gamma} \norm{\delta_k}_\infty
 , ~\forall s \in \cS.
\end{align*}
Summing up the prior relation from $t=0$ to $k-1$, we obtain the desired result.
\end{proof}

Combining Lemma \ref{lemma_stationary_srpmd} and Lemma \ref{lemma_srpmd_convergence_prop}, we are able to establish the following convergence characterization for SRPMD with Euclidean Bregman divergence, when adopting any constant-stepsize scheme.
\begin{theorem}\label{thrm_srpmd_Euclidean_constant_step}
Let $w(\cdot) = \norm{\cdot}_2^2$ be the distance-generating function, and $\eta_k = \eta$ for some fixed $\eta > 0$ and $k \geq 0$, then at any iteration $k \geq 1$,  SRPMD produces policy $\pi_R$ satisfying 
\begin{align}\label{eq_srpmd_Euclidean_constant_step}
\EE \sbr{ f_\rho(\pi_R) - f_\rho (\pi^*)} 
\leq \tfrac{M}{(1-\gamma) k} \rbr{f_\rho(\pi_0) - f_\rho(\pi^*)}
+ \tfrac{4 M}{(1-\gamma)^2 k} \sum_{t=0}^{k-1} e_t
+ \sqrt{\tfrac{18 \abs{\cS}^2}{k \eta (1-\gamma)^3}} \sqrt{1 + 2 \sum_{t=0}^{k-1} e_t } ,
\end{align} 
where $M$ is defined as in Lemma \ref{lemma_rpmd_convergence_prop}, and  $R$ is a random integer uniformly sampled from $\cbr{1 \ldots k}$.
In particular, if we have $\EE_{\xi_k} \norm{ Q^{\pi_k, \xi_k}_r - Q^{\pi_k}}_{\infty} \leq e$ for all $k \geq 0$, then 
\begin{align}\label{srpmd_Euclidean_converge_up_noise}
\EE \sbr{ f_\rho(\pi_R) - f_\rho (\pi^*)} 
\leq \tfrac{M}{(1-\gamma) k} \rbr{f_\rho(\pi_0) - f_\rho(\pi^*)}
 + \tfrac{4 M e}{(1-\gamma)^2 } 
 + \sqrt{\tfrac{18 \abs{\cS}^2}{\eta k (1-\gamma)^3}} 
 + \sqrt{\tfrac{36 \abs{\cS}^2 e}{\eta (1-\gamma)^3}}.
\end{align}
\end{theorem}

\begin{proof}
By summing up inequality \eqref{srpmd_basic_recursion} from $t = 0$ to $k-1$, we obtain 
\begin{align*}
& f_\rho(\pi_k) - f_\rho(\pi^*) 
+ \tfrac{1-\gamma}{M} \sum_{t=1}^{k-1} \rbr{ f_\rho(\pi_t) - f_\rho(\pi^*)}  \\
\leq &  f_\rho(\pi_0) - f_\rho(\pi^*) 
+\sum_{t=0}^{k-1} \rbr{ \tfrac{1}{M \eta} \EE_{s \sim d_\rho^{\pi^*, u_t}} D^{\pi^*}_{\pi_t}(s)
- \tfrac{1}{M \eta}  \EE_{s \sim d_\rho^{\pi^*, u_t}} D^{\pi^*}_{\pi_{t+1}}(s) } 
+ \tfrac{4}{1-\gamma} \sum_{t=0}^{k-1} \norm{\delta_t}_\infty.
\end{align*}
Following the same lines as in  the proof of Theorem \ref{thrm_rpmd_Euclidean_constant_step}, we can obtain from the above relation that
\begin{align}
 f_\rho(\pi_k) - f_\rho(\pi^*) 
+ \tfrac{1-\gamma}{M} \sum_{t=1}^{k-1} \rbr{ f_\rho(\pi_t) - f_\rho(\pi^*)} &  \leq 
f_\rho(\pi_0) - f_\rho(\pi^*) 
 + \tfrac{4}{1-\gamma}  \sum_{t=0}^{k-1} \norm{\delta_t}_\infty \nonumber \\
& ~~~  + \tfrac{\sqrt{18} }{M\eta}  \sum_{t=0}^{k-1} \EE_{s \sim d_\rho^{\pi^*, u_t}} \sqrt{ D^{\pi_t}_{\pi_{t+1}}(s) } \label{eucilidean_srpmd_recursion_crude}
 .
\end{align}
From Lemma \ref{lemma_stationary_srpmd}, we have 
\begin{align*}
 \rbr{ \sum_{t=0}^{k-1} \sqrt{ D^{\pi_t}_{\pi_{t+1}}(s) }}^2 
 \leq k \sum_{t=0}^{k-1} D^{\pi_t}_{\pi_{t+1}}(s)
&  \leq k \eta  \rbr{V^{\pi_0}_r(s) - V^{\pi^*}_r(s) + \tfrac{2}{1-\gamma} \sum_{t=0}^{k-1} \norm{\delta_t}_\infty } \\
&  \leq \tfrac{k \eta }{1-\gamma} \rbr{1 + 2 \sum_{t=0}^{k-1} \norm{\delta_t}_\infty} , ~~ \forall s \in \cS,
\end{align*}
which combined with \eqref{eucilidean_srpmd_recursion_crude}, gives 
\begin{align*}
 & f_\rho(\pi_k) - f_\rho(\pi^*) 
+ \tfrac{1-\gamma}{M} \sum_{t=1}^{k-1} \rbr{ f_\rho(\pi_t) - f_\rho(\pi^*)} \\  \leq &
f_\rho(\pi_0) - f_\rho(\pi^*) 
 + \tfrac{4}{1-\gamma} \sum_{t=0}^{k-1} \norm{\delta_t}_\infty 
  + \tfrac{\abs{\cS}}{M} \sqrt{\tfrac{18k}{\eta (1-\gamma)}} \sqrt{1 + 2 \sum_{t=0}^{k-1} \norm{\delta_t}_\infty } 
 .
\end{align*}
Finally, given the definition of $R \sim \mathrm{Unif} \cbr{1 \ldots k}$, we take expectation with respect to $\cbr{\xi_t}$ and $R$, and conclude that 
\begin{align*}
\tfrac{(1-\gamma) k }{M} \EE \sbr{ f_\rho(\pi_R) - f_\rho (\pi^*)} 
\leq 
f_\rho(\pi_0) - f_\rho(\pi^*) 
 +\tfrac{4}{1-\gamma} \sum_{t=0}^{k-1} e_t  + \tfrac{\abs{\cS}}{M} \sqrt{\tfrac{18k}{\eta (1-\gamma)}} \sqrt{1 + 2 \sum_{t=0}^{k-1} e_t } 
 ,
\end{align*}
where in the last inequality we also uses that fact that $\sqrt{1+ x}$ is concave and hence 
$\EE \sqrt{1 + 2 \sum_{t=0}^{k-1} \norm{\delta_t}_\infty } \leq \sqrt{1 + 2 \sum_{t=0}^{k-1} e_t }$.
Hence the desired claim \eqref{eq_srpmd_Euclidean_constant_step} follows immediately after simple rearrangement.
In addition, \eqref{srpmd_Euclidean_converge_up_noise} follows from $\sqrt{a+b} \leq \sqrt{a} + \sqrt{b}$.
\end{proof}

Given \eqref{srpmd_Euclidean_converge_up_noise}, Theorem \ref{thrm_srpmd_Euclidean_constant_step} states that for large enough constant stepsize, 
SRPMD with Euclidean divergence converges at the rate of $\cO(1/k)$, until the noise level of $\cO(M e / (1-\gamma)^2)$ is reached.

Our discussions so far have assumed an oracle that outputs the stochastic estimation of the state-action value function $Q^{\pi}_r$ satisfying certain error condition \eqref{stochastic_noise_condition}.
In the next section, we present a stochastic policy evaluation method that can in turn certify this error condition, and consequently determine the total sample complexity of SRPMD.


\section{Sample Complexity of Stochastic Robust Policy Mirror Descent}\label{sec_sample_complexity}

In this section, we discuss an online method of estimating the robust state-action value function $Q^{\pi}_r$ for a given policy $\pi$, by using samples $\xi$ collected during the interaction with the nominal environment $\cM_{\mathrm{N}}$.
By incorporating this online estimation method into the previously discussed SRPMD methods, we are able to learn a robust policy without the need of training policy within its worst-case environment. 
 Consequently, we will also establish the sample complexity of the SRPMD methods with different stepsize schemes discussed in Section \ref{sec_srpmd}.
 
To facilitate our presentation, let us define operator $F: \RR^{\abs{\cS}  \abs{\cA}} \to \RR^{\abs{\cS}  \abs{\cA}} $ by 
\begin{align}\label{fixed_point_op}
F(x) = \diag(\nu^{\pi}) \rbr{\cT^{\pi}(x) - x} + x,
\end{align}
where $\nu^{\pi}$ denotes the stationary state-action pair distribution induced by policy $\pi$ within the nominal environment $\cM_{\mathrm{N}}$ , and operator $\cT^{\pi} : \RR^{\abs{\cS}  \abs{\cA}} \to \RR^{\abs{\cS}  \abs{\cA}} $ is defined by 
\begin{align*}
\sbr{\cT^{\pi}(x)} (s,a)  & = c(s,a) + 
\gamma \max_{u \in \cU} \sum_{s' \in \cS} \sum_{a' \in \cA} \PP_u (s'|s,a) \pi(a'|s') x(s', a') \\
& = c(s,a) + \gamma \max_{u \in \cU} \sum_{s', a'} \PP_u^{\pi} (s', a'|s,a) x(s', a'),
\end{align*}
where in the second equality we denote $\PP_u^{\pi}(s', a'|s,a) =  \PP_u (s'|s,a) \pi(a'|s') $.
We will also write  the previous definition in matrix form as
\begin{align*}
\cT^{\pi}(x) = c + \gamma \max_{u \in \cU} \PP^{\pi}_u x.
\end{align*}
Clearly, given the rectangularity of uncertainty set in \eqref{rectangular_set}, we have the following equivalent definition of $\cT^{\pi}$, 
\begin{align}
\sbr{\cT^{\pi}(x)}(s,a)  & =
c(s,a) + \gamma  \sum_{s' \in \cS} \sum_{a' \in \cA} \PP_{\mathrm{N}} (s'|s,a) \pi(a'|s') x(s', a') + \gamma \max_{u \in \cU} \sum_{s' \in \cS} u(s'|s,a) \sum_{a' \in \cA} \pi(a'|s') x(s', a') \nonumber  \\
& = c(s,a) + \gamma  \sum_{s' \in \cS} \sum_{a' \in \cA} \PP_{\mathrm{N}} (s'|s,a) \pi(a'|s') x(s', a') 
+ \gamma \sigma_{\cU_{s,a}} ( M (\pi, x)) , \label{cT_equiv}
\end{align}
where $M(\pi,x) \in \RR^{\abs{\cS}}$ is defined as $[M(\pi, x)] (s) =  \sum_{a \in \cA} \pi(a|s) x(s, a)$, 
and $\sigma_{X}$ denotes the support function of set $X$.

Given \eqref{eq_robust_bellman_q} in Proposition \ref{prop_robust_bellman_q}, it should be clear that the robust state-action value function $Q^{\pi}_r$ is a fixed point of operator $F$. 
On the other hand, since 
\begin{align}\label{contraction_prop_operator_cT}
\norm{\cT^{\pi}(x) - \cT^{\pi}(y)}_\infty 
= \gamma \norm{\max_{u \in \cU} \PP^{\pi}_u x - \max_{u \in \cU} \PP^{\pi}_u y}_\infty
\leq \gamma \max_{u \in \cU} \norm{ \PP^{\pi}_u x - \PP^{\pi}_u y}_\infty
\leq \gamma \norm{x - y}_\infty,
\end{align}
 $\cT^{\pi}$ is a $\gamma$-contraction in $\norm{\cdot}_\infty$-norm.
Thus whenever $\min (\nu^{\pi}) > 0$, 
$Q^{\pi}_r$ is the unique fixed-point of $F$.

\begin{algorithm}[t!]
    \caption{The robust temporal difference learning  method}
    \label{alg_rtd}
    \begin{algorithmic}
    \STATE{\textbf{Input:} 
    Policy $\pi$ to be evaluated. 
    Initial iterate $\theta_0 \in \RR^{\abs{\cS} \abs{\cA}}$ and stepsizes $\{\alpha_t\}_{t\geq 0}$, initial state $s_0 \in \cS$, initial action $a_0 \sim \pi(\cdot|s_0)$.}
    \FOR{$t=0, 1, \ldots$}
    \STATE{
    Collect $s_{t+1} \sim \PP_{\mathrm{N}}(\cdot|s_t, a_t)$, and make action $a_{t+1} \sim \pi(\cdot|s_{t+1})$. 
    }
    \STATE{
    Let  $\zeta_t = (s_t, a_t, s_{t+1}, a_{t+1})$, and construct 
    $f(\theta_t, \zeta_t)$ according to \eqref{stoch_op_rtd} (with known $\cU$) or \eqref{stoch_op_rtd_epsilon_contamination} (with unknown $\cU$).
    }
	\STATE{
	Update the iterate:
		\vspace{-0.1in}
 \begin{align*}
 \theta_{t+1} = \theta_t + \alpha_t \rbr{f(\theta_t; \zeta_t) - \theta_t} .
 \end{align*}
 }
	\vspace{-0.15in}
    \ENDFOR
    \end{algorithmic}
\end{algorithm}

We propose the robust temporal difference (RTD) method (Algorithm \ref{alg_rtd}) and establish its sample complexity for finding a stochastic estimate of $Q^{\pi}_r$.
For a given policy, an initial state $s_0 \in \cS$, and initial action $a_0 \sim \pi(\cdot|s_0)$,  the
RTD method, at any iteration $t \geq 0$, 
(1) Given $(s_t, a_t)$, collects $s_{t+1} \sim \PP_{\mathrm{N}}(\cdot|s_t, a_t)$, 
and make actions $a_{t+1} \sim \pi(\cdot|s_{t+1})$;
(2) Constructs $\zeta_t = (s_t, a_t, s_{t+1}, a_{t+1})$, and performs the following
 \begin{align*}
 \theta_{t+1} = \theta_t + \alpha_t \rbr{f( \theta_t ; \zeta_t) -  \theta_t}  .
 \end{align*}
 Unlike our previous discussions where generic convergence properties of SRPMD are established in agnostic to $\cU$, the construction of operator $f(\cdot ; \cdot)$ depends on the available information on the uncertainty set $\cU$, which we discuss below.

{\bf Evaluation with Known $\cU$.}
When offline data is available, one can often take the nominal kernel $\PP_{\NN}$ as the empirical kernel estimated from the data and construct the uncertainty set $\cU$ by employing statistical deviation theory \cite{nilim2005robust}.
Note that the henceforth proposed method does not require us to compute and store $\PP_{\NN}$ and $\cU$ beforehand, instead one can form $\PP_{\NN}(\cdot|s,a)$ and $\cU_{s,a}$ whenever these quantities are needed. 
We consider a stochastic operator $f(x; \zeta): \RR^{\abs{\cS} \abs{\cA}} \to \RR^{\abs{\cS}\abs{\cA}}$, where 
$\zeta = (s,a, s', a')$ denotes a random quadruple sampled from a Markov chain defined over $\cZ = (\cS \times \cA)^2$. 
Specifically, the operator takes the form of 
\begin{align}\label{stoch_op_rtd}
f(x; \zeta) = \rbr{c(s, a) + \gamma x(s',a') + \gamma \sigma_{\cU_{s,a}} ( M(\pi, x)) - x(s,a) } e(s,a) + x.
\end{align}

Let $\cbr{\zeta_t}$ be the Markov chain of state-action pairs generated by policy $\pi$ within $\cM_{\NN}$, then given \eqref{fixed_point_op} and \eqref{cT_equiv}, by letting $\nu^{\pi}$ denotes the stationary distribution of $\cbr{\zeta_t}$, we have 
$
\EE_{\zeta \sim \nu^{\pi}} f(x; \zeta) = F(x).
$

{\bf Evaluation with Unknown $\cU$.}
In contrast to training with offline data, another typical application of robust MDP involves hedging against the environment changes when the deployed environment of the learned policy is different from the training environment $\PP_{\NN}$.
In this case both the training environment $\PP_{\NN}$ and the uncertainty set $\cU$ can be unknown.
We consider the notion of $\epsilon$-contamination model \cite{huber1992robust} when modeling the ambiguity set, namely, 
\begin{align}\label{eq_epsilon_uncertainty}
\cU_{s,a} = \cbr{ \epsilon \rbr{ \QQ - \PP_{\NN}(s,a)} : \QQ \in \cQ_{s,a} }, ~ \forall (s,a) \in \cS \times \cA.
\end{align}
The ambiguity set $\cP_{s,a}$ defined in \eqref{eq_ambiguity_set} is therefore the convex combination of the transition kernel of the training environment and another set of pre-specified probability distributions $\cQ_{s,a}$ over $\cS$:
\begin{align*}
\cP_{s,a} = \cbr{ (1- \epsilon) \PP_{\NN}(s,a) + \epsilon \QQ:  \QQ \in \cQ_{s,a}}.
\end{align*}
Clearly, users can adjust the tunable parameter $\epsilon \in [0,1]$ based on their robustness preference. 
The choice of $\cQ_{s,a}$ also allows encoding prior knowledge on the transition kernel of the potential deployed environment into planning.
In particular, setting $\cQ_{s,a} = \Delta_{\cS}$ corresponds to the R-contamination model considered in \cite{wang2022policy}.
For the ambiguity set considered in \eqref{eq_epsilon_uncertainty}, we consider stochastic operator $f(x; \zeta)$ defined as 
\begin{align}\label{stoch_op_rtd_epsilon_contamination}
f(x; \zeta) = \rbr{c(s, a) + \gamma (1-\epsilon) x(s',a') + \gamma \epsilon \sigma_{\cQ_{s,a}} ( M(\pi, x)) - x(s,a) } e(s,a) + x.
\end{align}
Notably, the stochastic operator \eqref{stoch_op_rtd_epsilon_contamination} does not require any information on the unknown uncertainty set $\cU$ \eqref{eq_epsilon_uncertainty}.
One can also immediately verify that $
\EE_{\zeta \sim \nu^{\pi}} f(x; \zeta) = F(x),
$
where $\nu^{\pi}$ denotes the stationary distribution of $\cbr{\zeta_t}$.
It is also possible to define alternative uncertainty sets other than \eqref{eq_epsilon_uncertainty}, 
as long as one can construct $f(x; \zeta)$ as an unbiased estimator of $F(x)$ \cite{liu2022distributionally}.

Through out the rest of our discussions, we make the following assumption on the to-be-evaluated policy and the nominal environment $\cM_{\mathrm{N}}$, which is commonly assumed in the literature of reinforcement learning.

\begin{assumption}\label{assump_exploration}
The policy $\pi$ satisfies $\min_{s \in \cS, a \in \cA} \pi(a|s) > 0$, and the Markov chain $\cbr{s_t}$  induced by $\pi$ within the nominal MDP $\cM_{\mathrm{N}}$ is aperiodic and irreducible.
\end{assumption}

Combining Assumption \ref{assump_exploration} and the finiteness of the state space $\cS$, the Markov chain $\cbr{s_t}$ satisfies geometric-mixing property \cite{levin2017markov}. 
In addition, the stationary distribution of $\cbr{s_t}$, denoted by $\mu^\pi$, satisfies $\mu^{\pi} (s) > 0$ for all $s \in \cS$.
Consequently, we also have $\nu_{\min} = \min_{s\in \cS, a\in \cA} \nu^\pi(s,a) > 0$.

The following lemma characterizes the sample complexity of the RTD method for obtaining a stochastic estimation of $Q^{\pi}_r$, which utilizes the machinery of stochastic approximation applied to contraction operators developed in \cite{chen2021lyapunov}.


\begin{lemma}\label{lemma_rtd_complexity}
Under Assumption \ref{assump_exploration}, for any $\epsilon > 0$,
let $\alpha_t = \alpha$ for some properly chosen $\alpha$,
 then RTD method finds an estimate $\theta_T$ satisfying 
$
 \EE_{\xi} \norm{\theta_T - Q^{\pi}}_\infty \leq \epsilon,
$ in at most 
 \begin{align}\label{rtd_complexity}
 T = \tilde{\cO} \rbr{ \tfrac{\log^2(1/\epsilon)}{(1-\gamma)^5 \nu_{\min}^3 \epsilon^2}  }
 \end{align}
 iterations, 
 where $\xi = \cbr{\zeta_t}_{t=0}^{T}$ denotes the trajectory collected by the RTD method,
 and $\tilde{\cO}(\cdot)$ ignores polylogarithmic terms.
\end{lemma}

\begin{proof}
We begin by establishing several properties of the operator $F$ defined in \eqref{fixed_point_op},  the stochastic operator $f$ defined in \eqref{stoch_op_rtd},
and the Markov chain $\cbr{\zeta_t}$.
For operator $F$, note that
\begin{align}
\norm{F(x) - F(y)}_\infty 
& = \norm{\diag(\nu^{\pi}) \rbr{\cT^\pi(x) - \cT^\pi(y)} + \rbr{I - \diag(\nu^\pi)} (x - y)}_\infty \nonumber \\
& \leq \rbr{ 1 - \min(\nu^\pi)) (1 - \gamma) } \norm{x - y}_\infty, \label{verify_op_prop}
\end{align}
where the inequality uses \eqref{contraction_prop_operator_cT}.
Hence $F$ is a $\rbr{ 1 - \min(\nu^\pi)} $-contraction in $\norm{\cdot}_\infty$ norm.
Consequently, $Q^{\pi}_r$ is the unique fixed point of $F$.
The rest of the proof focuses on stochastic operator defined in \eqref{stoch_op_rtd}, while the argument for the stochastic operator defined in \eqref{stoch_op_rtd_epsilon_contamination} follows similar lines.

For operator $f$ defined in \eqref{fixed_point_op}, we have for any $\zeta$,
\begin{align*}
& \norm{f(x, \zeta) - f(y, \zeta)}_\infty  \\
 =&  \norm{ \sbr{\gamma \rbr{x(s', a') - y(s', a')} - \rbr{x(s,a) - y(s,a)} + \sigma_{\cU_{s,a}}(M(\pi, x)) - \sigma_{\cU_{s,a}}(M(\pi, y))} e(s,a) 
+ x - y}_\infty \\
\leq & 3 \norm{x - y}_\infty + \abs{\sigma_{\cU_{s,a}}(M(\pi, x)) - \sigma_{\cU_{s,a}}(M(\pi, y))}.
\end{align*}
Now by defining  $z_x = M(\pi, x)$ for any $x$, we know  
$
\abs{ z_x(s) - z_y(s) } =\abs{ \tsum_{a \in \cA} \pi(a|s) \rbr{x(s,a) - y(s,a)} } \leq \norm{x - y}_\infty
$ holds for any $s\in \cS$.
Hence we have 
\begin{align*}
 \abs{\sigma_{\cU_{s,a}}(M(\pi, x)) - \sigma_{\cU_{s,a}}(M(\pi, y))}
 & =  \abs{\sigma_{\cU_{s,a}}(z_x) - \sigma_{\cU_{s,a}}(z_y)} \\
 & \leq \max_{u(\cdot|s,a) \in \cU_{s,a}} \abs{ \inner{u(\cdot|s,a)}{z_x - z_y}}  \\
 &\leq \norm{u(\cdot|s,a)}_1 \norm{ z_x - z_y}_\infty \\
 &  \overset{(a)}{\leq}  2 \norm{ z_x - z_y}_\infty \leq 2 \norm{x - y}_\infty ,
\end{align*}
where inequality $(a)$ uses the fact that $\norm{u(\cdot|s,a)}_1 = \norm{\PP_u(\cdot|s,a) - \PP_{\mathrm{N}}(\cdot|s,a)}_1 \leq 2$.
Thus we obtain 
\begin{align}\label{verify_stoch_props_1}
 \norm{f(x, \zeta) - f(y, \zeta)}_\infty \leq 5 \norm{x - y }_\infty.
\end{align}
Additionally, one can readily verify that 
\begin{align}\label{verify_stoch_props_2}
\norm{f(\mathbf{0}, \zeta)}_\infty \leq 1, ~~\forall \zeta \in \cZ.
\end{align}

Lastly, for the Markov chain $\cbr{\zeta_t}$, we proceed to establish its fast-mixing property under Assumption \ref{assump_exploration}.
Note that the stationary distribution of $\cbr{\zeta_t}$, denoted by $\nu^\pi$, is given by 
$\nu^\pi (s,a, s', a') = \mu^{\pi}(s) \pi(a|s) \PP_{\mathrm{N}}(s'|s,a) \pi(a'|s')$.
Let us denote the transition kernel of $\cbr{\zeta_t}$ by $\PP_{{\bf \zeta}}$, 
and accordingly denote the transition kernel of $\cbr{s_t}$ by $\PP_{\mathrm{S}}$.
Then for any $\zeta \in \cZ$,
\begin{align*}
\norm{ \PP^{k+1}_{\zeta}(\zeta, \cdot) - \nu^\pi(\cdot)}_{\mathrm{TV}}
& = \tfrac{1}{2} \tsum_{ \tilde{\zeta}} \abs{ \PP^{k+1} (\zeta, \tilde{\zeta}) - \nu^\pi( \tilde{\zeta})} \\
& \overset{(a)}{=} \tfrac{1}{2} \tsum_{ \tilde{\zeta}} \abs{ \PP^k_{\mathrm{S}}(s', \tilde{s}) \pi(\tilde{a}|\tilde{s}) \PP_{\mathrm{N}} (\tilde{s}'|\tilde{s}, \tilde{a})  \pi(\tilde{a}' | \tilde{s}') - 
\mu^{\pi}(\tilde{s}) \pi(\tilde{a}|\tilde{s}) \PP_{\mathrm{N}}(\tilde{s}'|\tilde{s}, \tilde{a}) \pi(\tilde{a}'|\tilde{s}')
} \\
& \leq \tfrac{1}{2} \tsum_{\tilde{s}} \abs{\PP^k_{\mathrm{S}}(s', \tilde{s}) -  \mu^{\pi}(\tilde{s})} \leq C \kappa^k
\end{align*}
for some $\kappa \in (0,1)$ and $C > 0$, where the last inequality follows from the geometric-mixing property of $\cbr{s_t}$ given Assumption \ref{assump_exploration},
and equality $(a)$ follows from the Markov property.
Thus we obtain that 
\begin{align}\label{verify_stoch_process_prop}
\max_{\zeta} \norm{ \PP^{k+1}_{\zeta}(\zeta, \cdot) - \nu^\pi(\cdot)}_{\mathrm{TV}} \leq C \kappa^k.
\end{align}
Consequently, let us define $T_\alpha \coloneqq \min \cbr{k: k\geq 0, \max_{\zeta \in \cZ} \norm{\PP^k_{\zeta} (\zeta, \cdot) - \nu^\pi (\cdot)}_{\mathrm{TV}} \leq \alpha}$.

From \eqref{verify_op_prop}, \eqref{verify_stoch_props_1}, \eqref{verify_stoch_props_2}, and \eqref{verify_stoch_process_prop}, it should be clear that the operators $F, f$ and the stochastic process $\cbr{\zeta_t}$ satisfy Assumption 2.1 - 2.3 in \cite{chen2021lyapunov}.
Now for any $\epsilon > 0$, consider taking $\alpha_t = \alpha > 0$ in RTD method, where 
$\alpha = \cO \rbr{ \frac{\epsilon^2 \nu_{\mathrm{min}}^2 (1-\gamma)^4}{\log(1/\epsilon) \log (\abs{\cS} \abs{\cA})}}$
and $\alpha T_\alpha = \cO \rbr{\frac{(1-\gamma)^4 \nu_{\mathrm{min}}^2}{\abs{\cA}^2 \log (\abs{\cS} \abs{\cA})}}$, 
 the rest of the proof  then follows the same lines as in Corollary 3.1.1 therein. 
\end{proof}

Given Lemma \ref{lemma_rtd_complexity}, we proceed to establish the sample complexity of the SRPMD method with different stepsize schemes discussed in Section \ref{sec_srpmd}.
We begin with the stepsize scheme \eqref{srpmd_linear_stepsize} considered in Theorem \ref{thrm_srpmd_linear_convergence}, which demonstrates linear convergence up to policy evaluation error.

\begin{proposition}\label{prop_linear_convergence_sample_complexity}
Suppose the stepsizes $\cbr{\eta_k}$ in SRPMD satisfy 
$
\eta_k \geq \eta_{k-1} \rbr{ 1 - \tfrac{1-\gamma}{M}}^{-1} M'
$
for all $k \geq 1$,
 $M'$ is defined as in Theorem \ref{thrm_rpmd_linear_convergence}.
 Furthermore,  for any $\epsilon > 0$, suppose $\EE_{\xi_k} \norm{ Q^{\pi_k, \xi_k}_r - Q^{\pi_k}}_{\infty} \leq e$ with $4Me / (1-\gamma)^2 \leq \epsilon / 2$. 
Then 
SRPMD outputs a policy $\pi_k$
with $\EE \sbr{ f_\rho(\pi_k) - f_\rho(\pi^*) } \leq \epsilon$ in 
\begin{align*}
k = \cO \rbr{ \tfrac{M}{1-\gamma} \sbr{ \log\rbr{\tfrac{\Delta_0}{\epsilon}} + \log \rbr{\tfrac{D_w}{M \eta_0 \epsilon}} }}
\end{align*}
iterations,
where $M$ is defined as in Lemma \ref{lemma_rpmd_convergence_prop}, and $\Delta_0 = f_\rho(\pi_0) - f_\rho(\pi^*)$.
In addition, the total number of samples required by SRPMD can be bounded by 
\begin{align*}
\tilde{\cO} \rbr{
\tfrac{M^3 \log^2 \rbr{4M / (\epsilon (1-\gamma)^2)} }{(1-\gamma)^{10} \nu_{\min}^3 \epsilon^2 }\sbr{ \log\rbr{\tfrac{\Delta_0}{\epsilon}} + \log \rbr{\tfrac{D_w}{M \eta_0 \epsilon}} }
}.
\end{align*}
\end{proposition}

\begin{proof}
The bound on the total number of iterations $k$ can be readily obtained from \eqref{srpmd_linear_up_to_precision} in Theorem \ref{thrm_srpmd_linear_convergence}, if $4M e /(1-\gamma)^2 \leq \epsilon /2 $.
To satisfy this condition, suppose one needs to run RTD for  $T$ iterations when evaluating for each $Q^{\pi_k}_r$,
then  from \eqref{rtd_complexity} in  Lemma \ref{lemma_rtd_complexity}, one can bound $T$ by 
\begin{align*}
T = \tilde{\cO} \rbr{
\tfrac{M^2 \log^2 (4M / (\epsilon (1-\gamma)^2))}{(1-\gamma)^9 \nu_{\min}^3 \epsilon^2  }
}.
\end{align*}
The bound on the total number of samples follows immediately by combining the previous two observations.
\end{proof}

Given Proposition \ref{prop_linear_convergence_sample_complexity}, we remark that the sample complexity of applying SRPMD to solving the robust MDP with $(\mathbf{s}, \mathbf{a})$-rectangular uncertainty sets is comparable to that of solving standard MDPs with linearly converging policy mirror descent methods, in terms of its dependence on the optimality gap \cite{lan2021policy}, and is slightly worse  in terms of its dependence on the effective horizon $(1-\gamma)^{-1}$.
To the best of our knowledge, this is the first sample complexity result for first-order policy-based method that is optimal in terms of the dependence on the optimality gap.
A closer look to the analysis shows that this worse dependence on the effective horizon comes from the current convergence characterization of the RTD method, which exhibits a worse dependence on the effective horizon compared to the CTD method  considered in \cite{lan2021policy} for evaluating policy in standard MDPs. 
See Section \ref{sec_conclusion} for more detailed discussions.

Finally, we establish the sample complexity of SRPMD when using the Euclidean divergence, which allows a constant-stepsize scheme and attains sublinear convergence.

\begin{proposition}\label{sample_complexity_Euclidean_constant_step}
Let $w(\cdot) = \norm{\cdot}_2^2$ be the distance-generating function, and $\eta_k = \eta$ for all $k \geq 0$.
Furthermore,  for any $\epsilon > 0$, suppose $\EE_{\xi_k} \norm{ Q^{\pi_k, \xi_k}_r - Q^{\pi_k}}_{\infty} \leq e$ with $4Me / (1-\gamma)^2 \leq \epsilon / 2$. 
Then  
by taking $\eta =  72 \abs{\cS}^2 / \rbr{(1-\gamma)^2 M \epsilon}$, 
SRPMD outputs a policy $\pi_R$
with $\EE \sbr{ f_\rho(\pi_R) - f_\rho(\pi^*) } \leq \epsilon$, where $R \sim \mathrm{Unif}(\cbr{1 \ldots k})$,  in 
\begin{align*}
k = \cO \rbr{ 
\tfrac{M\Delta_0}{(1-\gamma) \epsilon} 
}
\end{align*}
iterations,
where $M$ is defined as in Lemma \ref{lemma_rpmd_convergence_prop}, and $\Delta_0 = f_\rho(\pi_0) - f_\rho(\pi^*)$.
In addition, the total number of samples required by SRPMD can be bounded by 
\begin{align*}
\tilde{\cO} \rbr{
\tfrac{M^3 \log^2\rbr{4M / (\epsilon (1-\gamma)^2}}{(1-\gamma)^{10} \nu_{\min}^3 \epsilon^3 } 
}.
\end{align*}
\end{proposition}

\begin{proof}
The bound on the total number of iterations $k$ can be readily obtained from \eqref{srpmd_Euclidean_converge_up_noise} in Theorem \ref{thrm_srpmd_Euclidean_constant_step}, provided 
\begin{align*}
  \tfrac{4 M e}{(1-\gamma)^2 }  \leq \tfrac{\epsilon}{4}, ~~
  \sqrt{\tfrac{18 \abs{\cS}^2}{\eta k (1-\gamma)^3}} \leq \tfrac{\epsilon}{4}, ~~
 \sqrt{\tfrac{36 \abs{\cS}^2 e}{\eta (1-\gamma)^3}} \leq \tfrac{\epsilon}{4}.
 \end{align*}
To satisfy the first condition above, suppose one needs to run RTD for  $T$ iterations when evaluating for each $Q^{\pi_k}_r$,
then  from \eqref{rtd_complexity} in  Lemma \ref{lemma_rtd_complexity}, one can bound $T$ by 
\begin{align*}
T = \tilde{\cO} \rbr{
\tfrac{M^2 \log^2 (4M / (\epsilon (1-\gamma)^2))}{(1-\gamma)^9 \nu_{\min}^3 \epsilon^3  } 
}.
\end{align*}
To satisfies the second and third conditions, it suffices to have 
$
\eta \geq \max \cbr{\tfrac{72 \abs{\cS}^2}{k (1-\gamma)^3 \epsilon^2}, \tfrac{36 \abs{\cS}^2 e}{(1-\gamma)^3 \epsilon^2}},
$
which can be readily satisfied by $\eta =  72 \abs{\cS}^2 / \rbr{(1-\gamma)^2 M \epsilon}$ given the bound on $k$ and $e$.
The bound on the total number of samples $Tk$ then follows immediately by combining the previous observations.
\end{proof}

To the best of our knowledge, all the obtained sample complexities in this section appear to be new in the literature of first-order methods applied to the robust MDP problem. 
The best sample complexity for PGM applied to this problem in the existing literature is at the order of $\cO(1/\epsilon^7)$ \cite{wang2022policy}, which focuses on the Euclidean Bregman divergence and a special subclass of uncertainty set defined in \eqref{eq_epsilon_uncertainty}.
In comparison, as shown in Proposition \ref{prop_linear_convergence_sample_complexity} and \ref{sample_complexity_Euclidean_constant_step},  SRPMD with the same divergence improves this sample complexity by orders of magnitude, and applies to a much more general class of uncertainty sets.


\section{Concluding Remarks}\label{sec_conclusion}

In this manuscript, we develop the robust policy mirror descent method and its stochastic variants for controlling Markov decision process with uncertain transition kernels. 
Our established iteration and sample complexity seem to be new in the literature of policy-space first-order methods applied to this problem class.
We highlight a few future directions worthy of continuing explorations from our perspective. 

First, the analysis of constant stepsize RPMD yields an additional dependence on the size of the state space.
Though this dependence can be bypassed with a large stepsize, removing this dependence completely remains not only as a theoretical interest, but can also potentially help improving the sample complexity of the SRPMD methods. 

Second, the current analysis of SRPMD uses only a single characterization on the noise of the stochastic estimate (see \eqref{stochastic_noise_condition}), which contrasts with more delicate approach of separating bias and variance for solving standard MDPs \cite{lan2021policy}.
As a result, it is unclear whether the dependence of obtained sample complexities on the effective horizon is optimal.
The reason for our simplified treatment is due to the fact that the robust TD method in Section \ref{sec_sample_complexity} does not have a separate characterizations for the bias and variance in the obtained stochastic estimate given the nonlinearity of operator $F$ in \eqref{fixed_point_op}.
This hinder the application of techniques for standard MDPs adopted in \cite{lan2021policy}, where the author heavily exploits the fact that bias converges much faster than the variance, given the linearity of the TD operator. 
It is even unclear that whether one can separate bias and variance in estimating the robust state-action value function, which by itself would be an highly interesting question.
Another question related to the robust TD method is to relax Assumption \ref{assump_exploration}, which requires handling the rarely visited state-action pair in evaluating the robust state-action value function.
Techniques for addressing this problem in solving standard MDPs have been recently discussed in \cite{li2023policy}. 

Lastly, it would also be rewarding to develop  RPMD variants for solving robust MDP beyond the $(\mathbf{s}, \mathbf{a})$-rectangular uncertainty sets considered in this manuscript.

\bibliographystyle{plain}
\bibliography{references}

\appendix 
\section{Supplementary Proofs in Section \ref{sec_structural_props}}\label{section_appendix}

\begin{proof}[Proof of Proposition \ref{prop_bellman_v}]
Fix  the policy $\pi \in \Pi$, for any $u \in \cU$, define operator $\cT^{\pi}_u: \RR^{\abs{\cS}} \to \RR^{\abs{\cS}}:  \cT^{\pi}_u(V) = c^{\pi} + \gamma \PP^{\pi}_u V$,
where $c^{\pi}(s) = \sum_{a \in \cA} \pi(a|s) c(s,a)$, and $\PP^{\pi}_u (s, s')  = \sum_{a \in \cA} \PP_u(s'|s,a) \pi(a|s)$.
It is well known that the value  $V^{\pi}_u(s)$ is the entry corresponding to state $s$ in the solution of the following linear program \cite{puterman2014markov}:
\begin{align}\label{value_formulation_lp}
\max_{v \in \RR^{\abs{\cS}}} \mathbf{e}_s^\top v ~~ \mathrm{s.t.}~~ v \leq \cT^{\pi}_u v,
\end{align}
where $\mathbf{e}_s$ denotes the one-hot vector with the entry corresponding to $s$ being non-zero.
Moreover, we have $V^{\pi}_u = \cT^{\pi}_u V^{\pi}_u$.
It is also useful to make note of the following properties of $\cT^{\pi}_u$:
(1) $\cT^{\pi}_u$ is monotone, in the sense that $v \leq v' \Rightarrow \cT^{\pi}_u v \leq \cT^{\pi}_u v'$;
(2) $\cT^{\pi}_u$ is a $\gamma$-contraction in $\norm{\cdot}_\infty$-norm, with the unique fixed-point being $V^{\pi}_u$.
Since both are trivial to verify, we omit their proofs here.

By varying the uncertainty $u \in \cU$, for a fixed state $s \in \cS$, we claim that the robust value   $V^{\pi}_r(s)$ satisfies $V^{\pi}_r(s) = v^*(s)$, where $(v^*, u^*)$ is any  solution of the following program
\begin{align}\label{opt_formulation_robust}
\max_{v \in \RR^{\abs{\cS}}, u \in \cU} \mathbf{e}_s^\top v ~ \mathrm{s.t.}~ v \leq \cT^{\pi}_u v.
\end{align}
To see this, note that (1) $(V^{\pi}_u, u)$ is a feasible solution for \eqref{opt_formulation_robust} for any $u \in \cU$;
(2) any optimal solution $(v^*, u^*)$ must satisfy 
\begin{align*}
v^* \in \Argmax_{v \in \RR^{\abs{\cS}}} ~ \mathbf{e}_s^\top v ~ \mathrm{s.t.} ~ v \leq \cT^{\pi}_{u^*} v,
\end{align*}
which implies $v^*(s) = V^{\pi}_{u^*}(s)$ given \eqref{value_formulation_lp}. 
Combining these two observations, it holds that 
$ V^{\pi}_{u^*}(s) = v^*(s) \geq V^{\pi}_u (s)$ for any $u \in \cU$. 
Consequently, 
$v^*(s)$ is the robust value of $\pi$ at state $s$, and 
$u^*$ is  the corresponding worst-case uncertainty when we start from state $s$.

We proceed to show that formulation \eqref{opt_formulation_robust} is equivalent to the following 
\begin{align}\label{opt_formulation_robust_concise}
\max_{v \in \RR^{\abs{\cS}} } \mathbf{e}_s^\top v ~ \mathrm{s.t.}~ v \leq \sup_{u \in \cU} \cbr{\cT^{\pi}_u v} =  \max_{u \in \cU} \cbr{\cT^{\pi}_u v} ,
\end{align}
where operation $\sup_{u \in \cU} \cbr{\cT^{\pi}_u v}$ is the element-wise supremum, which is  well-defined due to the rectangularity of $\cU$. The  equality holds since $\cU$ is compact and $\cT^{\pi}_u v$ is continuous in $u$.   

To establish equivalence between \eqref{opt_formulation_robust} and  \eqref{opt_formulation_robust_concise}.
Note that for any feasible solution $(v, u)$ to \eqref{opt_formulation_robust},  $v$ must also be feasible to \eqref{opt_formulation_robust_concise}.
Hence we obtain $\mathrm{Opt} \eqref{opt_formulation_robust_concise} \geq \mathrm{Opt} \eqref{opt_formulation_robust} $.
On the other hand, suppose $v^*$ is a solution of \eqref{opt_formulation_robust_concise}, 
given the compactness and the rectangularity of $\cU$, 
we know that there exists $u^* \in \cU$, such that 
$v^* \leq \cT^{\pi}_{u^*} v^* \equiv \max_{u \in \cU} \cbr{\cT^{\pi}_u v^*} $.
Thus $(v^*, u^*)$ is a feasible solution to \eqref{opt_formulation_robust}, and $\mathrm{Opt} \eqref{opt_formulation_robust} \geq \mathrm{Opt} \eqref{opt_formulation_robust_concise} $, which further implies  $\mathrm{Opt} \eqref{opt_formulation_robust_concise} = \mathrm{Opt} \eqref{opt_formulation_robust} $.
Moreover, in this case,  
 $(v^*, u^*)$ is an optimal solution of \eqref{opt_formulation_robust}.
To summarize, we make the following observations:
\begin{itemize}[noitemsep, topsep=0pt]
\item {\it Observation 1.} If $v^*$ is a solution of \eqref{opt_formulation_robust_concise}, then $v^*(s) = V^{\pi}_r(s)$; 
\item {\it Observation 2.} If in addition,  
 $\cT^{\pi}_{u^*} v^* = \max_{u \in \cU} \cbr{\cT^{\pi}_u v^*}$, then  
$u^*$ is the corresponding worst-case uncertainty if we start from state $s$.
\end{itemize}


Now define operator $\cT^{\pi}: \RR^{\abs{\cS}} \to \RR^{\abs{\cS}}$ as $\cT^{\pi} v = \max_{u \in \cU} \cT^{\pi}_u v$, then $\cT^{\pi}$ is monotone since $\cT^{\pi}_u$ is monotone for every $u \in \cU$.
We proceed to show that $\cT^{\pi}$ is also a $\gamma$-contraction in $\norm{\cdot}_\infty$-norm.
\begin{align*}
\norm{\cT^{\pi} v -  \cT^{\pi} v'}_\infty 
= \norm{\sup_{u \in \cU} \cT^{\pi}_u v - \sup_{u \in \cU} \cT^{\pi}_u v'  }_\infty 
\leq 
\sup_{u \in \cU} \norm{\cT^{\pi}_u v -  \cT^{\pi}_u  v'}_\infty 
\leq \gamma \norm{v - v'}_\infty,
\end{align*}
where the first inequality uses  $\norm{\sup_{u \in \cU} f(u, v) - \sup_{u \in \cU} f(u,v')}_\infty \leq \sup_{u \in \cU} \norm{f(u, v) - f(u, v')}_\infty$ for any vector-valued function $f$,
and the second inequality uses the contraction property of operator $\cT^{\pi}_u$ for any $u \in \cU$.

Fixing the state $s \in \cS$, given a solution $v^*$ to \eqref{opt_formulation_robust_concise}, we claim that $v^*(s) = v'(s)$, where $v'$
is the unique fixed point of operator $ \cT^{\pi}$. That is,
$
v' = \cT^{\pi} v' \equiv \sup_{u \in \cU} \cT^{\pi}_u v'
$.
To see this, 
note that
\begin{align}
v^* \overset{(a)}{\leq} \lim_{t \to \infty} \rbr{\cT^{\pi}}^{(t)} v^* \overset{(b)}{=} v',  \label{ineq_v_opt_bound_by_fix_point}
\end{align}
where $(a)$ follows from applying the constraint of \eqref{opt_formulation_robust_concise} repeatedly to $v^*$, together with the monotonicity of $\cT^{\pi}$; 
$(b)$ follows from  $v'$ being the unique fixed point of $\cT^{\pi}$.
In addition, $v'$ is clearly feasible to \eqref{opt_formulation_robust_concise}. 
This in turn implies that if $v^*(s) < v'(s)$, then $v^*$ would not be an optimal solution to \eqref{opt_formulation_robust_concise}.
Combining this with \eqref{ineq_v_opt_bound_by_fix_point}, we must have $v^*(s) = v'(s)$.
Note that from {\it Observation 1}, this in turn implies $V^{\pi}_r(s) = v'(s)$.
Since the state $s$ can be chosen arbitrarily, we obtain $V^{\pi}_r = v'$, and \eqref{bellman_robust_value} follows immediately. 

It remains to show the existence of a single worst-case uncertainty $u'$ regardless of the initial state.
To this end, note that our previous discussion have shown that $v'$ is a solution of \eqref{opt_formulation_robust_concise}, regardless of the initial state $s \in \cS$.
In addition, one can pick $u' \in \cU$ such that $v' = \cT^{\pi} v' = \cT^{\pi}_{u'} v'$, given the fact that $v'$ is the fixed point of $\cT^{\pi}$ and $\cU$ being compact and rectangular. 
Then given {\it Observations 1 and 2}, $(v' , u')$ is an optimal solution of \eqref{opt_formulation_robust}, regardless of the initial state. 
Consequently, $\mu'$ is a worst-case uncertainty regardless of the initial state, and we obtain \eqref{def_worst_case_transition}.
The proof is then completed.
\end{proof}

\begin{proof}[Proof of Proposition \ref{prop_robust_bellman_q}]
Property \eqref{eq_robust_bellman_q} follows from similar lines as the proof of Proposition \ref{prop_bellman_v}.
To show \eqref{def_robust_q}, note that 
\begin{align*}
Q^{\pi}_r(s,a) = \max_{u \in \cU} Q^{\pi}_{u}(s,a) = c(s,a) + \max_{u \in \cU} \sum_{s' \in \cS} \PP_u(s'|s,a) V^{\pi}_u (s'),
\end{align*}
where the last inequality follows from the standard relation between $V^{\pi}_u$ and $Q^{\pi}_u$.
It suffices to note that by taking $u = u_{\pi}$ defined in \eqref{def_worst_case_transition}, we have $V^{\pi}_{u_\pi} = V^{\pi}_r$ and thus 
$ \max_{u \in \cU} \sum_{s' \in \cS} \PP_u(s'|s,a) V^{\pi}_u (s') \geq \sum_{s' \in \cS} \PP_{u_{\pi}} (s'|s,a) V^{\pi}_r (s') 
= \max_{u \in \cU}  \sum_{s' \in \cS} \PP_{u} (s'|s,a) V^{\pi}_r (s') $,
where the last equality follows from \eqref{def_worst_case_transition}.
On the other hand, we have  
\begin{align*}
\max_{u \in \cU} \sum_{s' \in \cS} \PP_u(s'|s,a) V^{\pi}_u (s')
\leq  \max_{u \in \cU} \max_{u' \in \cU}  \sum_{s' \in \cS} \PP_{u} (s'|s,a) V^{\pi}_{u'} (s') 
=  \max_{u \in \cU}  \sum_{s' \in \cS} \PP_{u} (s'|s,a) V^{\pi}_r (s') ,
\end{align*}
hence we obtain 
\begin{align*}
Q^{\pi}_r(s,a)  = c(s,a) + \max_{u \in \cU} \sum_{s' \in \cS} \PP_u(s'|s,a) V^{\pi}_u (s') = 
c(s,a) + \max_{u \in \cU} \sum_{s' \in \cS} \PP_u(s'|s,a) V^{\pi}_r (s').
\end{align*}
Thus \eqref{def_robust_q} is proved.
Moreover, \eqref{robust_value_q_relation} follows from taking expectation with respect to $a \sim \pi(\cdot|s)$ on both sides of \eqref{def_robust_q} and making use of \eqref{bellman_robust_value}.
Finally, from  \eqref{def_robust_q} and \eqref{robust_value_q_relation}, it is also clear that
\begin{align*}
Q^{\pi}_r(s,a)
=  c(s,a) +  \gamma \sum_{s' \in \cS}  \PP_{u_\pi}(s'|s,a) \sum_{a' \in \cA} \pi(a'|s') Q^{\pi}_r(s',a'), ~~ \forall (s,a) \in \cS \times \cA,
\end{align*}
where $u_{\pi}$ is defined as in \eqref{def_worst_case_transition},
 then $Q^{\pi}_r = Q^{\pi}_{u_\pi}$ since $Q^{\pi}_{u_\pi}$ is the unique solution of the previous system.
\end{proof}

\begin{proof}[Proof of Lemma \ref{lemma_almost_differentiability}]

It suffices to consider the differentiability of $f_\rho$ inside $\mathrm{ReInt}(\Pi)$, as its relative boundary is a zero-measure set when taking the  $(\abs{\cA} - 1){\abs{\cS}}$-dimensional Hausdorff measure.
We begin by noting that there exists $\cbr{e_i}_{i = 0}^{\abs{\cA} -1} \subset \RR^{\abs{\cA}}$, such that for any $\pi \in \Pi$, we have uniquely defined $\cbr{a^{\pi}_i(s)}_{i = 1, \ldots, \abs{\cA} - 1, s \in \cS}$ satisfying 
\begin{align}\label{a_pi_definition}
\pi(\cdot|s) = \sum_{i=1}^{\abs{\cA} - 1} a_i^{\pi}(s) \cdot e_i + e_0 
\coloneqq E \cdot a^{\pi}(s) + e_0, ~~ \forall s \in \cS,
\end{align}
where we denote $a^{\pi}(s) = (a_1^{\pi}(s), \ldots, a_{\abs{\cA} - 1}^\pi (s) )$, and $E = [e_1, \ldots, e_{\abs{\cA} - 1}]$ has independent columns.
We also write the above relation in short as $\pi = \cM(a^\pi)$.
It is clear that $\cM$ is a Lipschitz continuous mapping, and we denote its Lipschitz constant by $L_\cM$.
Alternatively, since $E$ has independent columns, we also have 
\begin{align}\label{from_pi_to_a}
a^{\pi}(s) = E^{\dagger} (\pi(\cdot|s) - e_0), ~~ \forall s \in \cS,
\end{align}
where $E^{\dagger}$ denotes the Moore–Penrose inverse of $E$.
We will write the above relation in short as $a^{\pi} = \cM^{-1} (\pi)$.
In addition, we can write the objective \eqref{formulation_single_obj} equivalently as 
\begin{align}\label{def_g}
f_\rho(\pi) = g(a^{\pi}; \cE),  
\end{align}
where $a^{\pi} = \rbr{a^{\pi}(s)}_{s \in \cS} = \cM^{-1}(\pi) \subset \RR^{(\abs{\cA} - 1)\abs{\cS}}$ is defined as in \eqref{from_pi_to_a}, and $\cE = (E; e_0)$.
Now consider the set 
\begin{align*}
\cA = \cbr{ (a(1), \ldots, a(\abs{\cS})): ~ E \cdot a(s)  + e_0 \in \mathrm{ReInt}(\Delta_{\cA}), ~\forall s \in \cS} \subset \RR^{(\abs{\cA} - 1)\abs{\cS}}.
\end{align*}
It is clearly that $\cA$ is an open set in $ \RR^{(\abs{\cA} - 1)\abs{\cS}}$.
In addition, 
for any $a, a' \in \cA$, by letting $\pi = \cM(a)$, $\pi' = \cM(a')$, 
we have 
\begin{align*}
\abs{g(a; \cE) - g(a'; \cE)} 
&= \abs{f_\rho(\pi) - f_\rho(\pi')}  \\
& \overset{(a)}{\leq} \tfrac{1}{1- \gamma}\sup_{s \in \cS}  \norm{\pi(\cdot|s) - \pi'(\cdot|s)}_1 \\
& \overset{(b)}{\leq} \tfrac{\sqrt{\abs{\cA}}}{1-\gamma} \sup_{s \in \cS}  \norm{E}_2 \norm{a(s) - a'(s)}_2 \\
& \leq \tfrac{\sqrt{\abs{\cA}}}{1-\gamma} \norm{E}_2 \norm{a - a'}_2 ,
\end{align*}
where $(a)$ follows from Lemma \ref{lemma_lipschitz_value_wrt_policy}, and $(b)$ follows from the definition \eqref{a_pi_definition}.
From the prior relation, we know that $g(\cdot; \cE): \cA \to \RR$ is a Lipschitz continuous mapping.
Combined with the fact that $\cA$ is open, we conclude from the Rademacher’s theorem \cite{lecture_gmt} that $g(\cdot; \cE)$ is almost everywhere differentiable in $\cA$, 
when the measure is taken to be the $\RR^{(\abs{\cA} - 1)\abs{\cS}}$-dimensional Lebesgue measure.
Let us define $\cA_z \subset \cA \subset \RR^{(\abs{\cA} - 1)\abs{\cS}}$ as the set of non-differentiable points of $g(\cdot;\cE)$.
Accordingly, we define $\Pi_z = \cbr{\pi \in \Pi: \pi = \cM(a), a \in \cA_z} \subset \RR^{\abs{\cS} \abs{\cA}}$. 
We proceed to show that $\Pi_z$ is a zero-measure set when taking the measure to be the $(\abs{\cA} - 1){\abs{\cS}}$-dimensional Hausdorff measure.

Recall that the $m$-dimensional Hausdorff measure of any set $A$ is defined as (see \cite{simon2014introduction})
\begin{align}
\cH^m(A) &= \lim_{\delta \to 0, \delta > 0} \cH^m_\delta(A) = \sup_{\delta > 0} \cH^m_\delta(A) \label{def_hausdorff_equiv}, \\
 \cH^m_\delta(A) &= w_m \inf_{\cbr{C_j}_{j=1}^\infty } \cbr{\sum_{j=1}^\infty \rbr{\tfrac{\mathrm{diam} (C_j) }{2}}^m: ~ \mathrm{diam}(C_j) < \delta, ~A \subset \cup_{j=1}^\infty C_j} , \label{def_hausdorff}
\end{align}
where $w_m = \pi^{m/2} / \Gamma(\tfrac{m}{2} + 1)$.
In addition, by letting $\cL^m$ denote the $m$-dimensional Lebesgure measure in $\RR^m$, we have the following relation \cite{simon2014introduction},
\begin{align}\label{hausdorff_lebesgue_equivalence}
\cL^m(A) = \cH^m(A) = \cH^m_\delta(A), ~ \forall \delta > 0, ~\forall A \subset \RR^m.
\end{align}
Now fix $\delta >0$, for any collection of subset $\cbr{C_j}_{j = 1}^m \subset \RR^{(\abs{\cA} - 1){\abs{\cS}}}$ with $\mathrm{diam}(C_j) < \delta$, and $\cA_z \subset \cup_{j=1}^\infty C_j$, 
we know that $\Pi_z \subset \cup_{j=1}^\infty \cM(C_j)$, 
and $\mathrm{diam}(\cM(C_j)) \leq   L_\cM \mathrm{diam}(C_j) 
 $. Thus, 

\begin{align*}
\cH^{(\abs{\cA} - 1){\abs{\cS}}}_{(L_\cM \delta)} (\Pi_z)
 \leq \sum_{j=1}^\infty \rbr{\tfrac{\mathrm{diam} (\cM(C_j)) }{2}}^{(\abs{\cA} - 1)\abs{\cS}}  \leq \sum_{j=1}^\infty \rbr{\tfrac{\mathrm{diam} (C_j) L_\cM }{2}}^{(\abs{\cA} - 1)\abs{\cS}}.
\end{align*}
Now by taking infimum over $\cbr{C_j}_{j = 1}^m \subset \RR^{(\abs{\cA} - 1){\abs{\cS}}}$ of the right hand side, we obtain 
\begin{align*}
\cH^{(\abs{\cA} - 1){\abs{\cS}}}_{(L_\cM \delta)} (\Pi_z)
& \leq L_{\cM}^{(\abs{\cA} - 1)\abs{\cS}} \inf_{\cbr{C_j}_{j=1}^\infty} \sum_{j=1}^\infty \rbr{\tfrac{\mathrm{diam} (C_j)  }{2}}^{(\abs{\cA} - 1)\abs{\cS}} \\
& \overset{(a)}{=} L_{\cM}^{(\abs{\cA} - 1)\abs{\cS}} \cdot \cH^{(\abs{\cA} - 1)\abs{\cS}}_\delta (\cA_z) \\
& \overset{(b)}{=} L_{\cM}^{(\abs{\cA} - 1)\abs{\cS}} \cdot \cL^{(\abs{\cA} - 1)\abs{\cS}}(\cA_z)  \\
&\overset{(c)}{=} 0,
\end{align*}
where $(a)$ follows from the definition in \eqref{def_hausdorff},
$(b)$ follows from equivalence of $\cL^{(\abs{\cA} - 1)\abs{\cS}}$ and $\cH^{(\abs{\cA} - 1)\abs{\cS}}_\delta$ for any $\delta > 0$ given \eqref{hausdorff_lebesgue_equivalence}, and the fact that 
$\cA_z \subset \RR^{(\abs{\cA} - 1)\abs{\cS}}$.
Finally, $(c)$ follows from the fact that $\cL^{(\abs{\cA} - 1)\abs{\cS}}(\cA_z) = 0$.
Thus, by letting $\delta \to 0$ on the left hand side, and making use of  the definition of  Hausdorff measure \eqref{def_hausdorff_equiv}, we obtain 
$\cH^{(\abs{\cA} - 1)\abs{\cS}} (\Pi_z) = 0$.

We then proceed to show that $f_\rho$ is differentiable within $\Pi_d = \mathrm{ReInt}(\Pi) \setminus \Pi_z = \cM(\cA \setminus \cA_z)$,
where the differentiability is defined  in the sense of Definition \ref{def_policy_grad}.
To see this, we first note that from the differentiability of $g(\cdot;\cE)$, 
 for any $a' \in  \cA \setminus \cA_z$, 
\begin{align}\label{diff_g_def}
g(a; \cE) - g(a'; \cE) - \inner{\nabla g(a'; \cE)}{a - a'} = \smallO(\norm{a - a'}), ~\forall a.
\end{align}

Let us  denote $\nabla g(a;\cE)[s]$ as the partial derivative of $g(a; \cE)$ with respect to $a(s)$. 
Now given any $ \pi' \in \Pi_d$, consider any policy $\pi$, 
we know that there exists $a = \cM^{-1}(\pi)$, $a' = \cM^{-1}(\pi')$, with $a' \in \cA \setminus \cA_z$.
Hence we obtain from the differentiability of $g(\cdot; \cE)$ at $a'$ that
\begin{align*}
& f_\rho (\pi) - f_\rho(\pi') - \sum_{s \in \cS} \inner{\nabla g (a'; \cE) [s]}{E^\dagger (\pi(\cdot|s) - \pi'(\cdot|s)} \\
\overset{(a)}{=} & 
g(a; \cE) - g(a'; \cE) 
-  \sum_{s \in \cS} \inner{\nabla g (a'; \cE) [s]}{ a(s) - a'(s)} \\
\overset{(b)}{=} &  \smallO(\norm{a - a'}) \\
\overset{(c)}{=} & \smallO(\norm{\pi - \pi'}),
\end{align*}
where  $(a)$ follows from \eqref{from_pi_to_a} and  \eqref{def_g}, 
$(b)$ follows from \eqref{diff_g_def} and the differentiability of $g(\cdot; \cE)$ at $a'$,
and $(c)$ follows again from \eqref{from_pi_to_a}.
Thus from the above relation and Definition \ref{def_policy_grad}, we know that $f_\rho$ is differentiable at any $\pi \in \Pi^d$, whose $(s,a)$-entry is given by
\begin{align*}
\nabla f_\rho(\pi) [s ,a]  = \sbr{ (E^\dagger)^\top \nabla g (\cM^{-1}(\pi); \cE) [s] }[a], ~~ \forall (s,a) \in \cS \times \cA.
\end{align*}
\end{proof}

\begin{proof}[Proof of  Lemma \ref{lemma_grad_rmdp_equiv_subgrad}]
Let us define $e = \mathbf{1}/\abs{\cA} \in \RR^{\abs{\cA}}$, 
and let $\mathbf{e} = e \otimes \mathbf{1} \in \RR^{\abs{\cS} \abs{\cA}}$. 
For any policy $\pi \in \Pi$, we can accordingly define $\pi_e$ where $\pi_e(\cdot|s) = \pi (\cdot|s) -e$. 
Conversely, let $\Pi_e= \cbr{\pi_e: \pi \in \Pi}$,
note that the mapping $\cM: \Pi \to \Pi_e$ from $\pi$ to $\pi_e$ is one-to-one and onto.
 Let us define function $g$ with domain $\Pi_e$ as
\begin{align*}
g(\pi_e ) = f_\rho(\pi), ~~ \forall \pi_e \in \Pi_e.
\end{align*}
It should be clear that $\Pi_e \subset \cU$ where $\cU$ is a $(\abs{\cA} - 1)\abs{\cS}$-dimensional subspace in $\RR^{\abs{\cS} \abs{\cA}}$.
Let $\norm{\cdot}$ denote the euclidean-norm on $\RR^{\abs{\cS} \abs{\cA}}$, then it is immediate that 
$(\cU, \norm{\cdot})$ is a Banach space, and $g: \Pi_e \subset \cU \to \RR$.

For any $\pi \in \mathrm{ReInt}(\Pi)$ where $f_\rho$ is differentiable in the sense of Definition \ref{def_policy_grad}, we have
$
 f_\rho(\pi') - f_\rho(\pi) - \inner{\nabla f_\rho(\pi)}{\pi' - \pi} = \smallO(\norm{\pi - \pi'}_2)
$.
Thus, by letting $\mathrm{Int}_{\cU}(X)$ denote the interior of set $X$  inside the Banach space $(\cU, \norm{\cdot})$, we obtain for  $\pi_e = \cM^{-1}(\pi) \in \mathrm{Int}_{\cU}(\Pi_e)$  that
\begin{align*}
g(\pi_e') - g(\pi_e)  - \inner{\nabla f(\pi)}{\pi_e' - \pi_e} =  \smallO(\norm{\pi_e - \pi_e'}), ~ \forall \pi'_e \in \Pi_e.
\end{align*}
That is, $g$ is Fr\'echet differentiable at $\pi_e$ with gradient $\nabla g(\pi_e) = \nabla f(\pi)$.
Given Lemma \ref{lemma_almost_differentiability}, it should be clear that the set of $\pi_e$ in $\Pi_e$ where $g$ is not Fr\'echet differentiable has a measure of zero in the $(\abs{\cA} -1)\abs{\cS}$-dimensional Hausdorff measure.  

We next show that \eqref{frechet_subgrad_of_objective} defines a Fr\'echet subgradient of $g$.
This is due to the fact that for any $\pi_e \in \mathrm{Int}_{\cU}(\Pi_e)$, 
\begin{align*}
 & \liminf_{\pi'_e \to \pi_e, \pi_e' \neq \pi_e, \pi_e' \in \Pi_e} \rbr{ g(\pi'_e) - g(\pi_e)  -  \tfrac{1}{1-\gamma}  \sum_{s \in \cS}  \sum_{a \in \cA}
d_\rho^{\pi, u_\pi}(s) Q^{\pi}_{u_\pi}(s, a) \rbr{\pi'_e(a|s) - \pi_e(a|s)} } \big/\norm{\pi_e' - \pi_e} 
\\
=&  \liminf_{\pi' \to \pi, \pi' \neq \pi, \pi' \in \Pi} \rbr{ f_\rho(\pi') - f_\rho(\pi)  -  \tfrac{1}{1-\gamma}  \sum_{s \in \cS}  \sum_{a \in \cA}
d_\rho^{\pi, u_\pi}(s) Q^{\pi}_{u_\pi}(s, a) \rbr{\pi'(a|s) - \pi(a|s)} } \big/\norm{\pi' - \pi}_2 \geq 0,
 \end{align*}
 where the equality follows from the construction of $g$ and $\pi_e$, and the inequality follows from Lemma \ref{lemma_frechet_subgrad}.
 
By combining observations from the previous two paragraphs, it is clear that the set of non-Fr\'echet differentiable points of $g$ form a zero-measure set in $\cU$, when taking the $(\abs{\cA} -1)\abs{\cS}$-dimensional Hausdorff measure.
The Fr\'echet subgradient at any point $\pi_e$ is given by 
\eqref{frechet_subgrad_of_objective}.
Moreover, at any Fr\'echet differentiable point $\pi_e$ of $g$, we conclude that its Fr\'echet gradient  $\nabla g(\pi)$ is the only element in its Fr\'echet subdifferential (Proposition 1.1, \cite{kruger2003frechet}), which is given by \eqref{frechet_subgrad_of_objective}.

It remains to show that for any Fr\'echet differentiable point $\pi_e \in \mathrm{Int}_{\cU}(\Pi_e)$ of $g$, 
with its derivative denoted by $\nabla g(\pi)$,
the corresponding $\pi \in \mathrm{ReInt}(\Pi)$ is also a differentiable point of $f_\rho$ in the sense of Definition \ref{def_policy_grad}.
To see this, note that  
\begin{align*}
& f_\rho(\pi') - f_\rho(\pi) - \inner{\nabla g(\pi_e)}{\pi' - \pi} \\
= & g(\pi'_e) - g(\pi_e)  - \inner{\nabla g(\pi_e)}{\pi'_e - \pi_e}
= \smallO(\norm{\pi'_e - \pi_e}) 
= \smallO(\norm{\pi' - \pi}_2), ~~ \forall \pi' \in \mathrm{ReInt}(\Pi).
\end{align*}
Thus we conclude that $f_\rho$ is also differentiable at $\pi$, and the gradient $\nabla f_\rho(\pi)$ defined in Definition \ref{def_policy_grad} coincides with the Fr\'echet gradient of $g$ at $\pi_e$, given by \eqref{frechet_subgrad_of_objective}. The proof is then completed.
\end{proof}

\begin{proof}[Proof of Lemma \ref{lemma_everywhere_differentiability_with_unique_assump}]
The essential arguments  can be understood as an application of Danskin's Theorem \cite{Danskin1967TheTO}, but with additional care to handle the low-dimensional nature of the domain $\Pi$. This is due to the reason that  Danskin's Theorem requires the domain to be an open set in the euclidean space, which is full-dimensional, 
and hence can not be directly applied in our setup, see Theorem II of Chapter 3 in \cite{Danskin1967TheTO}. 

We will inherit the same notations and definitions as in the proof of Lemma \ref{lemma_almost_differentiability}.
Let $\Pi_{\cM} = \cbr{ a^{\pi}: \pi \in \Pi}$. 
Note that the mapping $\cM: \Pi_{\cM} \to \Pi$ is one-to-one and onto, 
and $\mathrm{Int}(\Pi_{\cM})$ is an open set in $\RR^{(\abs{\cA} -1)\abs{\cS}}$.
To proceed, we first show that $g(a; \cE)$ is differentiable inside $\mathrm{Int}(\Pi_{\cM})$.
To this end, note that for any $\pi_a = \cM(a)$, 
\begin{align*}
g(a; \cE) = f_\rho(\pi_a) = \sum_{s\in \cS} \max_{u \in \cU} V^{\pi_a}_u (s)  \rho(s).
\end{align*}
To apply Danskin's Theorem, it suffices to show that 
\begin{itemize}
\item[(O)] $V^{\pi}_u(s)$ is continuous in $(\pi, u)$.
\item[(A)] For any $a \in \mathrm{Int}(\Pi_{\cM})$ and $u \in \cU$, $V^{\pi_a}_u (s)$ is differentiable in $a$, and the partial gradient is continuous in $(a,u)$.
\item[(B)] For any $a \in \mathrm{Int}(\Pi_{\cM})$, the worst-case uncertainty $\cbr{u \in \cU: V^{\pi_a}_u(s) = \max_{u \in \cU} V^{\pi_a}_u(s)}$ is a singleton, denoted by $u_{\pi_\alpha}$.
\end{itemize}
Note that 
condition (O) is trivial to verify, and 
condition (B) is readily implied by the precondition of the lemma. 
We then turn to show (A).
For any $a', a \in \mathrm{Int}(\Pi_{\cM})$, by letting $\delta = \pi_a' - \pi_a$, 
\begin{align}
V^{\pi_{a'}}_u (s) - V^{\pi_a}_u (s) 
 & \overset{(a)}{=} \tfrac{1}{1-\gamma} \sum_{s' \in \cS}  \sum_{\tilde{a} \in \cA}
d_s^{\pi_a, u}(s') Q^{\pi_a}_u(s', \tilde{a}) \delta(\tilde{a}|s') 
+ 
\smallO(\norm{\delta}_2) ,  \nonumber \\
& \overset{(b)}{=} \cL^{\pi_a}_u(\delta) + \smallO(\norm{a' - a}_2) \label{linear_operator_grad_implicit} \\
& \overset{(c)}{=} \cL^{\pi_a}_u(\cM(a' - a)) + \smallO(\norm{a' - a}_2) , \nonumber 
\end{align}
where $(a)$ is due to Lemma \ref{lemma_pg_standard}, and $(b)$ and $(c)$ follow from the definition in \eqref{a_pi_definition}, and $\cL^{\pi}_u$ denotes a linear operator of $\delta$ mapping to $\RR$ and implicitly defined via the first term in equality $(a)$.
Hence given $(c)$, since $\cL_u^a \circ M$ is again a linear operator, we know that $V^{\pi_a}_u(s)$ is differentiable at any point $a \in \mathrm{Int}(\Pi_{\cM})$.
To show the continuity of the gradient, it suffices to show that the operator $\cL_u^{\pi}$ is continuous, which simply follows from the fact that $Q^{\pi}_u$ and $d^{\pi, u}_s$ is continuous in $(\pi, u)$ (following from similar arguments of \eqref{state_visitation_continuity_of_pi_fix_u} and \ref{eq_lipschitz_value_wrt_policy}).   
Thus term (A) is proved.

Applying the Danskin's Theorem, we obtain that the robust value $V^{\pi_a}_r(s) = \max_{u \in \cU} V^{\pi_a}_u(s)$ is also differentiable in $a$, for any $a \in \mathrm{Int}(\Pi_\cM)$.
Specifically, we have 
\begin{align}\label{danskin_diff_wrt_a}
V^{\pi_{a'}}_r (s) - V^{\pi_a}_r (s)
=  \cL^{\pi_a}_{u_{\pi_\alpha}}(\cM(a' - a)) + \smallO(\norm{a' - a}_2),
\end{align}
for any $a, a' \in \mathrm{Int}(\Pi_{\cM})$.
Now given any $\pi, \pi' \in \mathrm{ReInt}(\Pi)$, it is clear that $a = a^\pi, a' = a^{\pi'} $ both belong to  $\mathrm{Int}(\Pi_{\cM})$, hence 
\begin{align*}
V^{\pi'}_r(s) - V^{\pi}_r(s)  = 
V^{\pi_{a'}}_r (s) - V^{\pi_a}_r (s)
& \overset{(a)}{=}   \cL^{\pi_a}_{u_{\pi_\alpha}}(\cM(a' - a)) + \smallO(\norm{a' - a}_2) \\
& \overset{(b)}{=}  \cL^{\pi}_{u_{\pi}}(\pi' - \pi) + \smallO(\norm{\pi' - \pi}_2),
\end{align*}
where $(a)$ follows from \eqref{danskin_diff_wrt_a},
and $(b)$ follows from the definition of $a, a'$ and \eqref{a_pi_definition}, \eqref{from_pi_to_a}.
Since $\cL^{\pi}_{u_\pi}$ is a linear operator, we obtain that $V^{\pi}_r(s)$ is differentiable at $\pi \in \mathrm{ReInt}(\Pi)$ in the sense of Definition \ref{def_policy_grad}.
The concrete form of gradient can be simplify read from the definition of operator $\cL^{\pi}_u$ in \eqref{linear_operator_grad_implicit}, the proof is then completed.

\end{proof}

\section{RPMD with a General Class of Bregman Divergences}\label{sec_rpmd_general_bregman_const_step}

In this section, we show that for a general class of Bregman divergences, whenever the set of uncertain transitions $\cP_{s,a}$ form a relatively strongly convex set in $\RR^{\abs{\cS}}$, then RPMD converges with any constant stepsize for a general class of divergences. 
To proceed, we first recall the following definition of strongly convex sets.

\begin{definition}[Strongly Convex Set]\label{def_sc_set}
A set $\cC$ in $\RR^d$ is called strongly convex with respect to $R > 0$, if 
$\cC$ is bounded, and for any $x,y \in \cC$, any $\lambda \in [0,1]$, we have 
$\cB(\lambda x + (1-\lambda) y) , \delta) \subset \cC$, 
where $\delta = \tfrac{\lambda (1-\lambda) }{2R} \norm{x-y}_2^2$. 
Here $\cB(z, r) = \cbr{z' \in \RR^d: \norm{z' - z}_2 \leq r}$ denotes the ball centered around $z$ with radius $r$.
\end{definition}

It is well known that any level set of a strongly convex function is a strongly convex set \cite{vial1982strong}. 
In particular, any full-dimensional ellipsoid is a strongly convex set.
A strongly convex convex set also satisfies the following useful property.
\begin{lemma}[Theorem 1, \cite{vial1982strong}]\label{sc_set_property}
For any set  $\cC \subset \RR^d$ that is strongly convex with respect to $R > 0$,  
let $\mathrm{Bd}(C)$ denote its boundary, and $\cN_{\cC}(x)$ denote its normal cone at $x \in X$.
Then for any pair of vectors $(x_i, p_i)$, $i = 1,2$, where $x_i \in \mathrm{Bd}(\cC)$ and $p_i \in \cN_{\cC}(x_i)$ with $\norm{p_i}_2 = 1$, we have 
\begin{align}\label{eq_sc_set_property}
\norm{x_1 - x_2}_2 \leq R \norm{p_1 - p_2}.
\end{align}
\end{lemma}

Note that by definition, a strongly convex set must have full dimension, which can not be satisfied by $\cP_{s,a}$ as $\cP_{s,a} \subset \Delta_{\cA}$ lies in a lower dimensional hyperplane.
Given this observation, we define the following strong convexity of a set $\cS$ restricted to its affine span.

\begin{definition}[Relatively Strongly Convex Set]\label{def_sc_set_affine}
A set $\cC$ in $\RR^d$ is called relatively strongly convex with respect to $R > 0$, if 
$\cC$ is bounded, and for any $x,y \in \cC$, any $\lambda \in [0,1]$, we have 
$\cB(\lambda x + (1-\lambda) y) , \delta) \cap \cH_\cC \subset \cC$, 
where $\delta = \tfrac{\lambda (1-\lambda) }{2R} \norm{x-y}_2^2$. 
Here $\cB(z, r) = \cbr{z' \in \RR^d: \norm{z' - z}_2 \leq r}$ denotes the ball centered around $z$ with radius $r$,
and $\cH_\cC = \mathrm{Aff}(\cC)$ denotes the affine subspace spanned by $\cC$.
\end{definition}

Similar to Lemma \ref{sc_set_property}, the relatively strongly convex set also satisfies an analogy to \eqref{eq_sc_set_property}.

\begin{lemma}\label{relative_sc_set_property}
For any set  $\cC \subset \RR^d$ that is relatively strongly convex with respect to $R > 0$,  
let $\mathrm{ReBd}(C)$ denote its relative boundary, and $\cN^r_{\cC}(x)$ denote its normal cone at $x \in X$ restricted to $\cH_\cC$, i.e.,
\begin{align*}
\cN^r_{\cC}(x) = \cbr{z \in \mathrm{Lin}(\cC): (y - x)^\top z \leq 0, \forall y \in \cC},
\end{align*}
 where $\mathrm{Lin}(\cC)$ denotes the unique linear subspace defining $\cH_\cC$, which satisfies $\cH_\cC = x + \mathrm{Lin}(C)$ for any $x \in \cC$.
Then for any pair of vectors $(x_i, p_i)$, $i = 1,2$, where $x_i \in \mathrm{ReBd}(\cC)$ and $p_i \in \cN^r_{\cC}(x_i)$ with $\norm{p_i}_2 = 1$, we have 
\begin{align}\label{eq_relative_sc_set_property}
\norm{x_1 - x_2}_2 \leq R \norm{p_1 - p_2}.
\end{align}
\end{lemma}
\begin{proof}
Note that $\cN^r_{\cC}(x)$ is invariant to translation of set $C$. 
Moreover, the both sides of statement \eqref{eq_relative_sc_set_property} are also invariant to translation of set $C$, and hence we can without loss of generality assume $\mathbf{0} \in \cC$, and hence $\cH_{\cC} = \mathrm{Lin}(\cC)$.
The rest of the claim follows directly from a change of coordinates and working in the lower-dimensional space $\mathrm{Lin}(\cC)$,  in which we can apply Lemma \ref{sc_set_property}. 
\end{proof}


We then show that if set of uncertain transitions $\cP_{s,a}$ is a relatively strongly convex set in $\RR^{\abs{\cS}}$ with dimension $\abs{\cS} - 1$, then the worst-case transition kernel $\PP_{u_{\pi}}$ is 
uniquely defined, and is
also Lipschitz continuous with respect to policy $\pi$.
En route, we will make use of the following simple fact.
\begin{lemma}\label{fact_normalized_norm_comp}
For $ x,y \in \RR^d $ and $x, y \neq \mathbf{0}$,  we have 
\begin{align}\label{eq_normalized_norm_comp}
\norm{\tfrac{x}{\norm{x}_2} - \tfrac{y}{\norm{y}_2}}_2 \leq   \norm{x - y}_2 / \min \cbr{\norm{x}_2, \norm{y}_2}.
\end{align}
\end{lemma}
\begin{proof}
Since both side of claim \eqref{eq_normalized_norm_comp} are symmetric with respect to $(x,y)$, we can without loss of generality assume $\norm{y}_2 \geq \norm{x}_2$.
Note that $\norm{\tfrac{x}{\norm{x}_2} - \tfrac{y}{\norm{y}_2}}_2 = \tfrac{1}{\norm{x}_2} \norm{x - \tfrac{y}{\norm{y}_2} \norm{x}_2}_2$.
We make the following observations.

Denote $\mathrm{Proj}_{y}: \RR^d \to \RR^d$ as the projection operator onto $\mathrm{span}(\cbr{y})$.
We know that 
$\mathrm{Proj}_{y}(x) = \delta_x y$ for  $\delta_x = \inner{x}{y}/ \norm{y}_2^2 $ and $\mathrm{Proj}_{y}(x) \leq \norm{x}_2$.
On the other hand, we have $x' \coloneqq \tfrac{\norm{x}_2}{\norm{y}_2} y = \delta_x' y$ with $\norm{x'}_2 = \norm{x}_2$,
hence $\delta_x' \geq \abs{\delta_x}$.
Since $x' \in \mathrm{span}(\cbr{y})$, we have
\begin{align*}
\norm{x - x'}_2^2 
= \norm{x - \mathrm{Proj}_y(x)}_2^2 + \norm{\mathrm{Proj}_y(x) - x'}_2^2 
= \norm{ x - \mathrm{Proj}_y(x)}_2^2 + (\delta_x' - \delta_x)^2 \norm{y}_2^2.
\end{align*}
On the other hand, since $\delta_x' \leq 1$, we also have 
\begin{align*}
\norm{ x - \mathrm{Proj}_y(x)}_2^2 + (\delta_x' - \delta_x)^2 \norm{y}_2^2
\leq \norm{ x - \mathrm{Proj}_y(x)}_2^2 + (1 - \delta_x)^2 \norm{y}_2^2
= \norm{x - y}_2^2.
\end{align*}
Hence, we obtain $\norm{x - x'}_2^2 \leq \norm{x - y}_2^2$.
In conclusion, since we have assumed $\norm{x}_2 \leq \norm{y}_2$, we obtain 
\begin{align*}
\norm{\tfrac{x}{\norm{x}_2} - \tfrac{y}{\norm{y}_2}}_2 \leq   \norm{x - y}_2 / \min \cbr{\norm{x}_2, \norm{y}_2}.
\end{align*}
The proof is then completed.
\end{proof}

The next lemma provides a sufficient condition the uniqueness of the worst-case transition kernel. 

\begin{lemma}\label{lemma_lipschitz_transition_wrt_policy}
Suppose $\cP_{s,a}$ is relatively strongly convex with respect to $R_{s,a} > 0$ for all $(s,a) \in \cS \times \cA$, and $\mathrm{dim}(\cP_{s,a}) = \abs{\cS} - 1$.
Let $r_* = \min_{\pi \in \Pi} \norm{\mathrm{Proj}_{\cH} ( V^{\pi}_r)}_2$, where $\cH = \mathrm{Lin}(\cP_{s,a}) =  \cbr{x \in \RR^{\abs{\cS}}: \mathbf{1}^\top x = 0} $.
Suppose $r_* > 0$, then for any policy $\pi \in \Pi$, the worst-case environment $\PP_{u_\pi}$ defined in \eqref{def_worst_case_transition}  is unique, and 
\begin{align}\label{eq_lipschitz_transition_wrt_policy}
 \norm{\PP_{u_{\pi}}(\cdot|s,a) - \PP_{u_{\pi'}}(\cdot|s,a)}_2 \leq 
\tfrac{R_{s,a}}{r_*} \norm{V^{\pi}_r - V_r^{\pi'}}_2 , ~~\forall (s, a) \in \cS \times \cA.
\end{align}
\end{lemma} 

The result of Lemma \ref{lemma_lipschitz_transition_wrt_policy} clearly hinges upon the quantity $r_* > 0$. 
We will provide a verifiable necessary and sufficient condition that certifies this requirement. 

\begin{proof}[Proof of Lemma \ref{lemma_lipschitz_transition_wrt_policy}]
We first recall from \eqref{def_worst_case_transition} that for any policy $\pi \in \Pi$, the worse case transition $\PP_{u_{\pi}}$ is given by 
\begin{align}\label{worst_case_transit_prob_opt_char}
\PP_{u_{\pi}}(\cdot|s,a) \in \argmax_{u(\cdot|s,a) \in \cU_{s,a}} \sum_{s' \in \cS} \PP_u(s'|s,a) V^{\pi}_r(s')
=  \argmax_{p \in  \cP_{s,a}}  p^\top V^{\pi}_r
, ~~ \forall (s,a) \in \cS \times \cA.
\end{align}

It is worth mentioning that the solution to \eqref{worst_case_transit_prob_opt_char} remains unchanged  when we shift $V^{\pi}_r$ by $\delta \mathbf{1}$ for any $\delta \in \RR$, where $\mathbf{1} \in \RR^{\abs{\cS}}$ denotes the all-one vector.
To see this, note that 
\begin{align}\label{opt_translation_invariant}
p^\top ( V^{\pi}_r + \delta \mathbf{1}) = p^\top V^{\pi}_r + \delta
\end{align}
due to $p \in \cP_{s,a}$,
and hence shifting $V^{\pi}_r$ only changes the objective by a constant $\delta$ for every feasible $p$.
Let $\mathrm{Proj}_{\cH}: \RR^{\abs{\cS}} \to \RR^{\abs{\cS}}$ denote the projection operator onto $\cH$, then 
given observation \eqref{opt_translation_invariant}, we must have 
\begin{align}\label{worst_case_transit_prob_opt_char_proj}
\PP_{u_{\pi}}(\cdot|s,a) \in  \argmax_{p \in  \cP_{s,a}}  p^\top \mathrm{Proj}_{\cH} ( V^{\pi}_r)
, ~~ \forall (s,a) \in \cS \times \cA, ~\forall \pi \in \Pi.
\end{align}

Let $\mathrm{ReBd}(\cP_{s,a})$ denote the relative boundary of $\cP_{s,a}$.
We claim that for any solution $p^*$ of \eqref{worst_case_transit_prob_opt_char} (or equivalently \eqref{worst_case_transit_prob_opt_char_proj}), we must have $p^* \in \mathrm{ReBd}(\cP_{s,a})$. 
Suppose not, and $p^*$ is in the relative interior of $\cP_{s,a}$, then there exists $\delta > 0$ such that 
$\cB(p^*, \delta) \cap \cH \subset \cP_{s,a}$.
Now it is immediate to see that $p^* + \delta \mathrm{Proj}_{\cH} (V^{\pi}_r) / \norm{\mathrm{Proj}_{\cH} (V^{\pi}_r)}_2$ is a strictly better solution than $p^*$ unless 
$\mathrm{Proj}_{\cH} (V^{\pi}_r) =\mathbf{0}$, which can not happen since $r_* > 0$.

We proceed to show that $p^*$ is indeed unique. 
Suppose not, and $p_1 \neq p_2 \in \cP_{s,a}$ are two solutions of \eqref{worst_case_transit_prob_opt_char}.
From the relative strong convexity of set $\cP_{s,a}$, we know that for any $\lambda \in (0,1)$,
we have $p_\lambda = \lambda p_1 + (1-\lambda) p_2 $ in the relative interior of $\cP_{s,a}$, 
which contradicts with the previously established fact that any solution of  \eqref{worst_case_transit_prob_opt_char} is in  $\mathrm{ReBd}(\cP_{s,a})$.

By reading the optimality condition of the \eqref{worst_case_transit_prob_opt_char_proj}, we have 
\begin{align}
( p - \PP_{u_{\pi}}(\cdot|s,a))^\top \mathrm{Proj}_{\cH} ( V^{\pi}_r) \leq 0, ~~ \forall p \in \cP_{s,a}.
\end{align}
Hence we immediately see that $\PP_{u_{\pi}}(\cdot|s,a) \in \mathrm{ReBd}(\cP_{s,a})$, and $\mathrm{Proj}_{\cH} ( V^{\pi}_r) \in \cN_{\cP_{s,a}} (\PP_{u_{\pi}}(\cdot|s,a))$.
Then we can apply Lemma \ref{sc_set_property} and obtain 
\begin{align*}
 \norm{\PP_{u_{\pi}}(\cdot|s,a) - \PP_{u_{\pi'}}(\cdot|s,a)}_2 \leq 
R_{s,a} \norm{\overline{V}^{\pi}_r - \overline{V}_r^{\pi'}}_2 , ~~\forall (s, a) \in \cS \times \cA,
\end{align*}
where $\overline{V}^{\pi}_r = \mathrm{Proj}_{\cH} ( V^{\pi}_r)/ \norm{ \mathrm{Proj}_{\cH} ( V^{\pi}_r)}_2, ~ \overline{V}^{\pi'}_r = \mathrm{Proj}_{\cH} ( V^{\pi'}_r) / \norm{ \mathrm{Proj}_{\cH} ( V^{\pi'}_r) }_2$.
Now applying Lemma \ref{fact_normalized_norm_comp}, we obtain 
\begin{align*}
 \norm{\PP_{u_{\pi}}(\cdot|s,a) - \PP_{u_{\pi'}}(\cdot|s,a)}_2 & \leq 
 R_{s,a}
  \norm{\mathrm{Proj}_{\cH} ( V^{\pi}_r) - \mathrm{Proj}_{\cH} ( V^{\pi'}_r) }_2 
 / r_* \\
 & \leq  \tfrac{R_{s,a}}{r_*} \norm{V^{\pi}_r - V^{\pi'}_r}_2,
\end{align*}
where $r_* = \min_{\pi \in \Pi} \norm{\mathrm{Proj}_{\cH} ( V^{\pi}_r)}_2$,
and the last inequality uses the non-expansiveness of the projection operator.
\end{proof}

It should be noted that the conditions on the relative strong convexity of $\cP_{s,a}$ and $\mathrm{dim}(\cP_{s,a}) = \abs{\cS} - 1$ are readily satisfied by many common uncertainty sets.
For instance, one can take $\cP_{s,a}$ as the intersection of a full-dimensional ellipsoid with $\Delta_{\cS}$ (i.e., ellipsoidal uncertainty \cite{roy2017reinforcement}).
On the other hand, Lemma \ref{lemma_lipschitz_transition_wrt_policy} relies on the key condition that $r_* = \min_{\pi \in \Pi} \norm{\mathrm{Proj}_{\cH} ( V^{\pi}_r)}_2 > 0$.
We next provide a simple necessary and sufficient condition that certifies $r_* > 0$.

\begin{lemma}\label{lemma_condition_of_r}
Suppose the cost function $c$ satisfies $\cap_{s \in \cS} \sbr{ \min_{a \in \cA} c(s,a) , \max_{a \in \cA} c(s,a) } = \emptyset$, then we must have 
$r_* = \min_{\pi \in \Pi} \norm{\mathrm{Proj}_{\cH} ( V^{\pi}_r)}_2 > 0$, where $\cH =  \cbr{x \in \RR^{\abs{\cS}}: \mathbf{1}^\top x = 0} $.
The converse of the statement is also true. 
\end{lemma}

\begin{proof}
``$\Rightarrow$'':
Suppose the cost function satisfies $c$ satisfies $\cap_{s \in \cS} \sbr{ \min_{a \in \cA} c(s,a) , \max_{a \in \cA} c(s,a) } = \emptyset$, but $r_* = 0$.
Then we have $V^{\pi}_r = \lambda \mathbf{1}$ for some $\lambda > 0$ (since $V^{\pi_r} > 0$ must hold)  and  $\pi \in \Pi$.
Let us define $\PP^{\pi}_{u_\pi} : \cS \times \cS \to [0,1]$ by $\PP^{\pi}_{u_\pi} (s, s') = \sum_{a \in \cA} \PP_{u_\pi}(s'|s,a) \pi(a|s)$,
and $c^{\pi}: \cS \to \RR$ by $c^{\pi}(s) = \sum_{a \in \cA} \pi(a|s) c(a|s)$, 
 then given observation \eqref{robust_and_standard_value_relation} and the standard Bellman condition of $V^{\pi}_{u_\pi}$, we know that 
$
V^{\pi}_r = \rbr{I - \gamma \PP^{\pi}_{u_\pi}}^{-1} c^{\pi},
$
which gives 
\begin{align}\label{condition_reward_pi_all_one}
c^{\pi} =  \rbr{I - \gamma \PP^{\pi}_{u_\pi}} V^{\pi}_r = V^{\pi}_r - \gamma  \PP^{\pi}_{u_\pi} V^{\pi}_r \overset{(a)}{=} \lambda \mathbf{1} - \gamma \lambda \mathbf{1} = (1-\gamma) \lambda \mathbf{1},
\end{align}
where in term $(a)$ we uses the fact that $\PP^{\pi}_{u_\pi} \mathbf{1} = \mathbf{1}$.
However, \eqref{condition_reward_pi_all_one} must contradict with the condition that 
$\cap_{s \in \cS} \sbr{ \min_{a \in \cA} c(s,a) , \max_{a \in \cA} c(s,a) } = \emptyset$,
since for each $s \in \cS$, we have $c^{\pi}(s) \in  \sbr{ \min_{a \in \cA} c(s,a) , \max_{a \in \cA} c(s,a) }$.

``$\Leftarrow$'':
On the other hand, 
suppose $r_* > 0$ but  $\cap_{s \in \cS} \sbr{ \min_{a \in \cA} c(s,a) , \max_{a \in \cA} c(s,a) } \neq \emptyset$, then one can readily find a policy $\pi$ such that 
$c^{\pi} = \lambda \mathbf{1}$ for some $\lambda \in \RR$. 
Thus we have 
\begin{align*}
V^{\pi}_r = \rbr{I - \gamma \PP^{\pi}_{u_\pi}}^{-1} c^{\pi}
\overset{(b)}{=} \sum_{t = 0}^\infty \gamma^t \rbr{ \PP^{\pi}_{u_\pi}}^t c^{\pi}
= \lambda \sum_{t = 0}^\infty \gamma^t  \mathbf{1} 
= \tfrac{\lambda}{1-\gamma} \mathbf{1},
\end{align*}
where in term $(b)$ we use the fact that $\rho(\gamma P^{\pi}_{u_\pi}) \leq \gamma $, where $\rho(A)$ denotes the spectral radius of matrix $A$. 
Note that the previous observation then implies $r_* \leq \norm{\mathrm{Proj}_{\cH} ( V^{\pi}_r)}_2 = 0$.
\end{proof}
%

We proceed to show that $V^{\pi}_r(s)$ is Lipschitz continuous in $\pi$.

\begin{lemma}\label{lemma_lipschitz_value_wrt_policy}
For any policies $\pi, \pi' \in \Pi$, we have 
\begin{align}\label{eq_lipschitz_value_wrt_policy}
\norm{V^{\pi}_r  - V^{\pi'}_r}_\infty 
\leq \tfrac{1}{1-\gamma} \norm{\pi - \pi'}_\infty,
\end{align}
where we define $\norm{\pi - \pi'}_\infty = \sup_{s \in \cS}  \norm{\pi(\cdot|s) - \pi'(\cdot|s)}_1$,
i.e., the matrix $\ell_\infty$-norm when viewing $\pi$ as a matrix in $\RR^{\abs{\cS} \times \abs{\cA}}$.
\end{lemma}

\begin{proof}
Note that for any $u \in \cU$, from the performance difference lemma of standard MDPs \cite{kakade2002approximately,lan2021policy}, we have 
\begin{align*}
\abs{ V^{\pi}_u (s) - V^{\pi}_u(s) } =
\abs{ \EE_{s' \sim d_{s}^{\pi, u}} \sbr{\inner{Q^{\pi}_u}{\pi - \pi'}_{s'} } } 
\leq \tfrac{1}{1-\gamma} \norm{\pi - \pi'}_\infty,
\end{align*}
where in the last inequality we use the Cauchy-Schwarz inequality, and the fact that $\norm{Q^{\pi}_u}_\infty \leq \tfrac{1}{1-\gamma}$.
Hence we obtain 
\begin{align*}
\abs{V^{\pi}_r(s) - V^{\pi'}_r(s) } 
= \abs{ \sup_{u \in \cU} V^{\pi}_u (s) - \sup_{u \in \cU} V^{\pi}_u(s)  }
\leq \sup_{u \in \cU} \abs{ V^{\pi}_u (s) - V^{\pi}_u(s) } 
\leq \tfrac{1}{1-\gamma} \norm{\pi - \pi'}_\infty.
\end{align*}
\end{proof}

By combining Lemma \ref{lemma_lipschitz_value_wrt_policy} and Lemma \ref{lemma_lipschitz_transition_wrt_policy}, we are ready to establish the Lipschitz continuity of $d^{\pi^*, u_{\pi}}$ with respect to policy $\pi$.

\begin{lemma}\label{lemma_smoothness_visitation_measure_wrt_policy}
Suppose $\cP_{s,a}$ is relatively strongly convex with respect to $R_{s,a} > 0$ for all $(s,a) \in \cS \times \cA$ with $R = \max_{s\in \cS, a\in \cA} R_{s,a}$, and $\mathrm{dim}(\cP_{s,a}) = \abs{\cS} - 1$.
Let $r_* = \min_{\pi \in \Pi} \norm{\mathrm{Proj}_{\cH} ( V^{\pi}_r)}_2$, where $\cH  =  \cbr{x \in \RR^{\abs{\cS}}: \mathbf{1}^\top x = 0} $.
If $r_* > 0$,
then we have 
\begin{align}\label{eq_smoothness_visitation_measure_wrt_policy}
\norm{d^{\pi^*, u_{\pi}}_\rho - d^{\pi^*, u_{\pi'}}_\rho}_1
\leq
\tfrac{ \gamma \abs{\cA}\abs{\cS}  R}{(1-\gamma)^2 r_*} \norm{\pi - \pi'}_\infty.
\end{align}
\end{lemma}

\begin{proof}
Let us define $\PP^{\pi}_u : \cS \times \cS \to [0,1]$ by $\PP^{\pi}_u (s', s) = \sum_{a \in \cA} \PP_u(s'|s,a) \pi(a|s)$, then for any $\pi \in \Pi$, and $\rho \in \Delta_{\cS}$, we obtain
\begin{align}
d_{\rho}^{\pi, u}
= (1-\gamma) \sum_{t=0}^{\infty} \gamma^t (\PP^{\pi}_u)^t \rho = (1 -\gamma) \rbr{ 1- \gamma \PP^{\pi}_u}^{-1} \rho,
\end{align}
where in the last equality we use the fact that $\rho( \gamma \PP^{\pi}_u) \leq \gamma < 1$.
Hence we have 
\begin{align*}
d^{\pi^*, u_{\pi}}_\rho - d^{\pi^*, u_{\pi'}}_\rho 
& = (1-\gamma) \rbr{
 \rbr{ I - \gamma \PP^{\pi^*}_{u_\pi}}^{-1} -  \rbr{  I - \gamma \PP^{\pi^*}_{u_\pi'}}^{-1}
} \rho \\
& = 
(1-\gamma) \gamma
\rbr{ I - \gamma \PP^{\pi^*}_{u_\pi} }^{-1}
\underbrace{\rbr{\PP^{\pi^*}_{u_\pi}  - \PP^{\pi^*}_{u_\pi'} }}_{\Delta^{\pi}_{\pi'}}
\rbr{ I - \gamma \PP^{\pi^*}_{u_\pi'} }^{-1}
 \rho,
\end{align*}
where the last equality uses the matrix identity
$A^{-1} - B^{-1} = A^{-1} (B-A) B^{-1}$ for any invertible matrix $A, B$.
Note that $\norm{\rbr{ I - \gamma \PP^{\pi^*}_{u_\pi} }^{-1}}_1 \leq (1-\gamma)^{-1}$ as we have shown in \eqref{bound_matrix_ell1_norm}, it suffices to control the $\ell_1$-norm of $\Delta^{\pi}_{\pi'} = \PP^{\pi^*}_{u_\pi}  - \PP^{\pi^*}_{u_\pi'}$.
We now make the following observations
\begin{align*}
\sum_{s' \in \cS} \abs{\PP^{\pi^*}_{u_\pi} (s', s) -  \PP^{\pi^*}_{u_\pi'} (s', s)}
& = 
\sum_{s' \in \cS} \abs{ \sum_{a \in \cA} \rbr{\PP_{u_\pi}(s'|s,a) - \PP_{u_{\pi'}}(s'|s,a)   } \pi^*(a|s) } \\
& \leq 
\sum_{a \in \cA} \sum_{s' \in \cS}
\abs{ \PP_{u_\pi}(s'|s,a) - \PP_{u_{\pi'}}(s'|s,a)   } \pi^*(a|s) \\
& \leq 
\sup_{a \in \cA} \norm{\PP_{u_\pi}(\cdot |s,a) - \PP_{u_{\pi'}}(\cdot |s,a) }_1 \\
& \overset{(a)}{\leq} 
\tfrac{R \sqrt{\abs{\cS}}}{r_*} \norm{V^{\pi}_r - V_r^{\pi'}}_2 \\
&\overset{(b)}{\leq} \tfrac{R \abs{\cS}}{r_* (1-\gamma)} \norm{\pi - \pi'}_\infty
\end{align*}
where $(a)$ uses Lemma \ref{lemma_lipschitz_transition_wrt_policy} and the definition of $R$,
and $(b)$ uses Lemma \ref{lemma_lipschitz_value_wrt_policy},
From the prior inequality, we obtain $\norm{\Delta^{\pi}_{\pi'}}_1 \leq \tfrac{R \abs{\cS} }{(1-\gamma)r_*} \norm{\pi - \pi'}_\infty$.
Putting everything together, we conclude that 
\begin{align*}
\norm{d^{\pi^*, u_{\pi}}_\rho - d^{\pi^*, u_{\pi'}}_\rho}_1
& \leq 
(1-\gamma) \gamma
\norm{\rbr{ I - \gamma \PP^{\pi^*}_{u_\pi} }^{-1}}_1
\norm{\Delta^{\pi}_{\pi'}}_1
\norm{\rbr{ I - \gamma \PP^{\pi^*}_{u_\pi'} }^{-1}}_1
 \norm{\rho}_1  \\
&  \leq
\tfrac{ \gamma \abs{\cS}  R}{(1-\gamma)^2 r_*} \norm{\pi - \pi'}_\infty.
\end{align*}
The proof is then completed.
\end{proof}

Combining elements established above, we are now ready to establish the main result of this subsection. 

\begin{theorem}\label{thrm_rpmd_constant_step_sc_set}
Suppose $\cP_{s,a}$ is relatively strongly convex with respect to $R_{s,a} > 0$ for all $(s,a) \in \cS \times \cA$ with $R = \max_{s\in \cS, a\in \cA} R_{s,a}$, and $\mathrm{dim}(\cP_{s,a}) = \abs{\cS} - 1$.
Let $r_* = \min_{\pi \in \Pi} \norm{\mathrm{Proj}_{\cH} ( V^{\pi}_r)}_2$, where $\cH = \cbr{x \in \RR^{\abs{\cS}}: \mathbf{1}^\top x = 0} $, and suppose $r_* > 0$.
For any $\eta > 0$, let $\eta_k = \eta > 0$ for all $k \geq 0$. 
If
\begin{enumerate}
\item[(1)]  Distance-generating function $w$ is $\mu$-strongly convex with respect to $\norm{\cdot}_1$-norm;
\item[(2)] $D_w^*=  \sup_{\pi \in \Pi} D^{\pi^*}_{\pi}(s) < \infty$.
\end{enumerate}
 Then RPMD outputs policy $\pi_k$ satisfying 
\begin{align}\label{eq_rpmd_constant_step_sc_set}
f_\rho(\pi_k) - f_\rho(\pi^*) 
\leq \tfrac{M}{(1-\gamma)k} \rbr{f_\rho(\pi_0) - f_\rho(\pi^*) }   
+ \tfrac{D_w}{(1-\gamma) \eta k} 
+\tfrac{ 2\gamma \abs{\cS}^2  R D_w^* }{(1-\gamma)^{7/2} r_* \sqrt{\eta \mu k}}  ,
\end{align}
where $M$ is defined as in Lemma \ref{lemma_rpmd_convergence_prop}.
\end{theorem}

\begin{proof}

By summing up inequality \eqref{rpmd_basic_recursion} from $t = 0$ to $k-1$, we obtain 
\begin{align}
& f_\rho(\pi_k) - f_\rho(\pi^*) 
+ \tfrac{1-\gamma}{M} \sum_{t=1}^{k-1} \rbr{ f_\rho(\pi_t) - f_\rho(\pi^*)}   \nonumber \\
\leq &  f_\rho(\pi_0) - f_\rho(\pi^*) + 
\tfrac{1}{M \eta_0} \EE_{s \sim d_\rho^{\pi^*, u_0}} D^{\pi^*}_{\pi_0}(s)
+ \underbrace{ \sum_{t=1}^{k-1} \rbr{\tfrac{1}{M \eta_t} \EE_{s \sim d_\rho^{\pi^*, u_t}} D^{\pi^*}_{\pi_t }(s) - \tfrac{1}{M \eta_{t-1}} 
\EE_{s \sim d_\rho^{\pi^*, u_{t-1} }} D^{\pi^*}_{\pi_{t}}(s) }}_{(A)} .  \label{rpmd_constant_step_sc_set_recursion_1}
\end{align}
%
We will make use of Lemma \ref{lemma_smoothness_visitation_measure_wrt_policy} to bound term $(A)$.
Specifically, we have 
\begin{align}
\abs{ \EE_{s \sim d_\rho^{\pi^*, u_t}} D^{\pi^*}_{\pi_t }(s) - 
\EE_{s \sim d_\rho^{\pi^*, u_{t-1} }} D^{\pi^*}_{\pi_{t}}(s) }
& \overset{(a)}{\leq } 
\norm{d_\rho^{\pi^*, u_t} - d_\rho^{\pi^*, u_{t-1} } }_1
D_w^* \nonumber \\
& \overset{(b)}{\leq} 
\tfrac{ \gamma \abs{\cS}  R D_w^*}{(1-\gamma)^2 r_*} \norm{\pi_t - \pi_{t-1}}_\infty,  \label{rpmd_constant_step_sc_set_divergence_lb}
\end{align}
where $(a)$ uses Holder's inequality, and $(b)$ uses Lemma \ref{lemma_smoothness_visitation_measure_wrt_policy}, and the definition that $u_t = u_{\pi_t}$ for all $t \geq 0$.
Since $w(\cdot)$ is $\mu$-strongly convex with respect to $\norm{\cdot}_1$-norm, we have 
\begin{align*}
\norm{\pi_t - \pi_{t-1}}_\infty
= \sup_{s \in \cS} \norm{ \pi_t(\cdot|s) - \pi_{t-1}(\cdot|s) }_1
\leq \sup_{s \in \cS} \sqrt{ 2 D^{\pi_t}_{\pi_{t-1}}(s) / \mu},
\end{align*}
which combined with  \eqref{rpmd_constant_step_sc_set_recursion_1} and \eqref{rpmd_constant_step_sc_set_divergence_lb}, gives 
\begin{align*}
& f_\rho(\pi_k) - f_\rho(\pi^*) 
+ \tfrac{1-\gamma}{M} \sum_{t=1}^{k-1} \rbr{ f_\rho(\pi_t) - f_\rho(\pi^*)}  \nonumber \\
\leq &  f_\rho(\pi_0) - f_\rho(\pi^*) + 
\tfrac{1}{M \eta} \EE_{s \sim d_\rho^{\pi^*, u_0}} D^{\pi^*}_{\pi_0}(s)
+\tfrac{ \gamma \abs{\cS}  R D_w^*}{(1-\gamma)^2 r_* M \eta}
 \sum_{t = 1}^{k-1}  \sum_{s \in \cS} \sqrt{ \tfrac{2 D^{\pi_t}_{\pi_{t-1}} (s)}{\mu}} .
\end{align*}
It follows from Lemma \ref{lemma_stationary} that
\begin{align*}
 \rbr{ \sum_{t=1}^{k-1} \sqrt{ D^{\pi_t}_{\pi_{t-1}}(s) }}^2 
 \leq k \sum_{t=1}^{k-1} D^{\pi_t}_{\pi_{t-1}}(s)
 \leq k \eta \rbr{V^{\pi_0}_r(s) - V^{\pi^*}_r(s) } \leq \tfrac{k \eta }{1-\gamma}, ~~ \forall s \in \cS.
\end{align*}
By combining two prior relations, we obtain 
\begin{align*}
& f_\rho(\pi_k) - f_\rho(\pi^*) 
+ \tfrac{1-\gamma}{M} \sum_{t=1}^{k-1} \rbr{ f_\rho(\pi_t) - f_\rho(\pi^*)}  \nonumber \\
\leq &  f_\rho(\pi_0) - f_\rho(\pi^*) + 
\tfrac{1}{M \eta} \EE_{s \sim d_\rho^{\pi^*, u_0}} D^{\pi^*}_{\pi_0}(s)
+ 
\tfrac{ 2\gamma \abs{\cS}^2  R D_w^* \sqrt{k}}{(1-\gamma)^{5/2} r_* M \sqrt{\eta \mu}}
.
\end{align*}
By \eqref{rpmd_gradient_monotone}, we know that $f_\rho(\pi_{t+1}) \leq f_\rho(\pi_t)$ for all $t \geq 0$.
The desired claim then follows immediately after combining this observation with the above inequality.
\end{proof}

In view of Theorem \ref{thrm_rpmd_constant_step_sc_set}, the RPMD method with constant stepsize $\eta$ outputs an $\epsilon$-optimal policy within 
$\cO(\max\cbr{1/\epsilon, 1/(\eta \epsilon^2)})$ iterations, for a general class of Bregman divergences.
Moreover,  conditions  (1) and (2) in Theorem \ref{thrm_rpmd_constant_step_sc_set}  are both satisfied by
some common distance-generating functions, including the squared $\ell_2$-norm 
 $w(p) = \norm{p}_2^2$, 
 and the negative Tsallis entropy $w(p) = \tfrac{k}{q-1} \rbr{\sum_{i} p_i^q - 1}$ for any entropic-index $q \in (1,2]$.

In addition, the second term of \eqref{eq_rpmd_constant_step_sc_set} also establishes a link on geometry of the set of uncertain probabilities $\cP_{s,a}$ to the convergence rate.
In addition, the dependence of convergence on the size of the state space is largely due to Lipschitz constant of $d_\rho^{\pi^*, u_{\pi}}$ with respect to policy $\pi$. 
The current characterization of such Lipschitz constant  seems improvable with additional efforts. 
Nevertheless, it is worth mentioning that   one can simply get rid of such a dependence by using a large stepsize.

\end{document}